\documentclass{article}

\usepackage[final]{neurips_2021}
\usepackage[utf8]{inputenc} % allow utf-8 input
\usepackage[T1]{fontenc}    % use 8-bit T1 fonts
\usepackage{hyperref}       % hyperlinks
\usepackage{url}            % simple URL typesetting
\usepackage{booktabs}       % professional-quality tables
\usepackage{amsfonts,amssymb}       % blackboard math symbols
\usepackage{nicefrac}       % compact symbols for 1/2, etc.
\usepackage{microtype}      % microtypography
\usepackage{xcolor}         % colors

\usepackage{graphicx}

\usepackage{amsmath,amsfonts,bm}

\def\eqref#1{equation~\ref{#1}}
\def\Eqref#1{Equation~\ref{#1}}
\def\1{\bm{1}}

\def\eps{{\epsilon}}

\def\ro{{\textnormal{o}}}

\def\vf{{\bm{f}}}
\def\vg{{\bm{g}}}
\def\vh{{\bm{h}}}

\def\vx{{\bm{x}}}
\def\vy{{\bm{y}}}

\def\mA{{\bm{A}}}
\def\mB{{\bm{B}}}

\def\mD{{\bm{D}}}
\def\mE{{\bm{E}}}

\def\mL{{\bm{L}}}

\def\mW{{\bm{W}}}

\DeclareMathAlphabet{\mathsfit}{\encodingdefault}{\sfdefault}{m}{sl}
\SetMathAlphabet{\mathsfit}{bold}{\encodingdefault}{\sfdefault}{bx}{n}

\def\gF{{\mathcal{F}}}

\def\gS{{\mathcal{S}}}

\newcommand{\R}{\mathbb{R}}

\DeclareMathOperator*{\argmin}{arg\,min}

\DeclareMathOperator{\Tr}{Tr}

\usepackage{wrapfig}
\usepackage{multirow}
\usepackage{enumitem}

\usepackage{comment}

\usepackage{amsthm}
\newtheorem{theorem}{Theorem}
\newtheorem{lemma}{Lemma}
\newtheorem{corollary}{Corollary}
\newtheorem{proposition}{Proposition}
\renewenvironment{proof}{\textit{Proof: }}{\hfill$\square$}

\usepackage{theoremref}
\newcommand{\pa}[1]{\left(#1\right)}
\newcommand{\br}[1]{\left[#1\right]}

\renewcommand{\ro}{\partial}
\newcommand{\ba}[1]{\overline{#1}}
\newcommand{\del}{\nabla}
\newcommand{\Z}{\mathbb{Z}}
\renewcommand{\Tr}[1]{\mathrm{Tr}\left[#1\right]}
\newcommand{\mat}[1]{\begin{pmatrix}#1 \end{pmatrix} }
\newcommand{\norm}[1]{\left\|#1\right\|}

\newcommand{\C}{\mathbb{C}}
\renewcommand{\L}{\mathcal{L}}
\newcommand{\out}[1]{}
\newcommand{\outNim}[1]{}

\newcommand{\ND}[1]{{\color{orange}[ND: #1]}}

\newcommand{\ry}[1]{{\color{magenta}[RY: #1]}}

\newcommand{\edits}[1]{{\color{blue}#1}}

\newcommand{\nd}{\ND}

\title{Automatic Symmetry Discovery with Lie Algebra Convolutional Network}

\author{%
  Nima Dehmamy\\
  Northwestern University\\%, Evanston, IL, USA\\
  \texttt{nimadt@bu.edu} \\
   \And
  Robin Walters \\
  Northeastern University \\%, Boston, MA, USA \\
  \texttt{rwalters@northeastern.edu}
   \AND
  Yanchen Liu \\
  Northeastern University \\%, Boston, MA, USA \\
  \texttt{liu.yanc@northeastern.edu}
   \And
  Dashun Wang \\
  Northwestern University \\ %,
  \texttt{dashun.wang@kellogg.northwestern.edu}
   \And
  Rose Yu \\
  University of California San Diego \\
  \texttt{roseyu@ucsd.edu} \\
}

\begin{document}

\maketitle

\begin{abstract}
    %\ry{identify the issue first}\nd{done}
    Existing equivariant neural networks  require  prior knowledge of the symmetry group and % for continuous groups
    discretization for continuous groups. %  or group representations
    % All these approaches require detailed knowledge of the group parametrization and cannot learn entirely new symmetries.
    % We show that by working
    We propose to work
    with  Lie algebras (infinitesimal generators) instead of  Lie groups.
    Our model, the Lie algebra convolutional network (L-conv) %, is based on infinitesimal generators of continuous groups and
    can automatically discover symmetries and
    does not require discretization of the group.
    We show that
    L-conv can serve as a building block to construct \textit{any} group equivariant feedforward architecture.
    % It can potentially learn new symmetries as well.
    % We show that L-conv can approximate any group convolutional layer by composition of layers.
    Both CNNs and Graph Convolutional Networks
    % and fully-connected networks
    can be expressed as L-conv with appropriate groups.
    We  discover direct connections between L-conv and physics: (1) group invariant loss generalizes field theory (2)  Euler-Lagrange equation measures the robustness, and (3) equivariance leads to conservation laws and Noether current.
    These connections open up new avenues for designing more general equivariant networks and applying them to important problems in physical sciences.
    \footnote{Code: \href{https://github.com/nimadehmamy/L-conv-code}{github.com/nimadehmamy/L-conv-code}
    }
    \out{
    L-conv also reveals striking connections between learning and physics:
    the loss for a single L-conv layer and find a deep relation with Lagrangians used in physics, with some of the physics aiding in defining generalization and symmetries in the loss landscape.
    % We find that generalization and symmetries n the loss landscape
    Conversely, L-conv could be used to propose more general equivariant ans\"atze for scientific machine learning.}
    %%%%
    % We propose to learn the symmetries during the training of the group equivariant architectures.
    % Additionally, by allowing the infinitesimal generators to be learnable, L-conv.
    % We also show how the symmetries are related to the statistics of the dataset in linear settings.
    % We find an analytical relationship between the symmetry group and a subgroup of an orthogonal group preserving the covariance of the input.
    % Our experiments show that L-conv with trainable generators performs well on problems with hidden symmetries.
    % Due to parameter sharing, L-conv also uses far fewer parameters than fully-connected layers.
\end{abstract}

\out{
\nd{
Todo:
\begin{enumerate}
    % \item known symmetries
    % \item clean up theory
    \item multilayer experiment
    \item closed under commutation (options: 1) doesn't matter experimentally; 2) take learned $L_i$ and express in terms of actual generators; 3) regularizers for small vectors)
    \item learn Hamiltonian dynamics for simple harmonic oscillator
\end{enumerate}

If the architecture could learn conserved quantities it would be very interesting.
}
}%%%

 \section{Introduction}
% Many learning tasks involve data which is equivariant under certain transformations, usually constituting a symmetry group. 
% Equivariant transformations keep the underlying distribution of data invariant \cite{bloem2019probabilistic}. 
% We will use the terms symmetry and equivariant transformations interchangeably.
% In cases such as natural images, most humans are capable of processing and identifying symmetries such as shifts, rotations, and scaling, e.g. 2D Euclidean symmetries $E(2)$. \ry{very verbose, what is the main point?}
% \ND{
% \begin{enumerate}
%     \item Parameters: FC with similar  
%     \item Multi layer
%     \item check equivariance 
%     \item L visualization 
%     \item Experiments 
%     \item prop 2 experiment
% \end{enumerate}
% }

% \short{
% Many machine learning (ML) tasks involve data from unfamiliar domains, which may or may not have hidden symmetries. 
% }%%%% \citep{cohen2016group,bogatskiy2020lorentz,maron2020learning,bronstein2021geometric}

% \nd{Todo: 1) compare with group discretization; 2) $W^0\to$MLP; 3) real-world use; 4) Image test fig }
Incorporating symmetries into a deep learning  architecture  can reduce sample complexity, improve generalization, while significantly decreasing the number of model parameters 
%via parameter sharing 
\citep{cohen2019general, cohen2016steerable,ravanbakhsh2017equivariance,ravanbakhsh2020universal,wang2020incorporating}.
% There is great interest in ML in designing new architectures that follow the principle of equivariance to symmetry transformations 
For instance, Convolutional Neural Networks (CNN) \citep{lecun1989backpropagation,lecun1998gradient} implement translation symmetry through weight sharing. 
General principles for constructing symmetry-aware group equivariant neural networks were introduced in \citet{cohen2016steerable}, \citet{kondor2018generalization}, and \citet{cohen2019general}.

% For continuous groups, each with a different parametrization, 

However, most work on equivariant  networks requires knowing the symmetry group \textit{a priori}. A different equivariant model needs to be re-designed for  each symmetry group.  In practice, we may  not have a good inductive bias and such knowledge of the symmetries may not be available. Constructing and selecting the equivariant network with the appropriate symmetry group becomes  quite tedious. 
Furthermore,
many existing works are limited to \textit{finite groups} such as permutations \cite{hartford2018deep,ravanbakhsh2017equivariance,zaheer2017deep}, $90$ degree rotations
\cite{cohen2018spherical} or dihedral groups $\mathrm{D}_N$ $E(2)$  \cite{weiler2019general}.

For a continuous group, 
existing approaches  either discretize the group \cite{weiler20183d,weiler2018learning,cohen2016group},
or use a truncated sum over irreducible representations (irreps) \cite{weiler2019general,weiler20183d} via spherical harmonics in \citet{worrall2017harmonic} or more general Clebsch-Gordon coefficients \citet{kondor2018clebsch,bogatskiy2020lorentz}.
These approaches are prone to approximation error. 
Recently, \citet{finzi2020generalizing} propose to approximates the integral over the Lie group by Monte Carlo sampling.
This approach  % It requires ``lifting'' inputs from a vector to an element in the group and 
requires implementing the matrix exponential  and  obtaining a local neighborhood for each point. Both parametrizing Lie groups for sampling and finding irreps are computationally expensive. 
% The logarithm needs to be implemented for each group. 
% While for low-dimensional groups such as $SO(2)$ or $SO(3)$ there are closed forms, for higher dimensions numerical methods need to be used, further complicating the architecture.
\citet{finzi2021practical} provide a general algorithm for constructing equivariant multi-layer perceptrons (MLP),
% Still, their method 
but require explicit knowledge of the group to encode its irreps, and solving a set of constraints.

% \textbf{Symmetry Discovery Literature}
% In addition to simplifying the construction of equivariant architectures, our method can also learn the symmetry generators from data. 

% \textbf{Contributions}
% Existing approaches require manual design of the  group representations or elements. 
% For continuous groups, 
% Each Lie group has a different parametrization and this task is quite tedious. 
% this often means the model needs to be re-designed for each group.
% Our method bypasses the integral over the group manifold.  that can automatically discover symmetries from data
We provide a novel framework for designing equivariant neural networks.
We leverage the fact that Lie groups can be constructed from a set of infinitesimal generators, called Lie algebras.
A Lie algebra has a finite basis, assuming the group is finite-dimensional.
% The Lie algebra usually has a finite basis 
% (a notable exception being Kac-Moody Lie algebras for 2D Conformal Field Theories \citep{belavin1984infinite} in physics).
Working with the Lie algebra basis allows us to encode an infinite group without discretizing or summing over irreps. 
Additionally, all Lie algebras have the same general structure and hence can be implemented the same way. 
% Additionally, for most symmetries we can expect the number of Lie algebra basis elements to be small and we could make them learnable parameters.  
% Hence, our architecture
% , which generalizes a group convolutional layer, 
% is potentially capable of learning symmetries in data without imposing inductive biases. 
%
We propose {Lie Algebra Convolutional Network} (\textbf{L-conv}), a novel architecture that can automatically discover symmetries from data.
% \nd{Rao 1999 showed simple versions of L-conv can learn the $L_i$. }
% Our work uses the Lie algebra (the linearization of the group near its identity) of continuous groups. 
Our main contributions can be summarized as follows:
\begin{itemize}
    \item We propose the Lie algebra convolutional network (\textbf{L-conv}), a building block for constructing group equivariant neural networks.
    \item We prove that multi-layer L-conv can approximate group convolutional layers, including CNNs,
    % on connected Lie groups.
    % can be approximated by multi-layer L-conv, 
    % and that
    % Fully-connected, 
    % \item 
    and find graph convolutional networks to be a special case of L-conv.
    % \item L-conv outperforms CNN on rotated and scrambled images in models with single hidden layer, suggesting it performs well in domains with unknown symmetries.
    \item We can learn the Lie algebra basis in L-conv, enabling automatic symmetry discovery. %, meaning it can potentially
    % to automatically discover symmetries.
    \item L-conv also reveals interesting connections between physics and learning: equivariant loss generalizes important Lagrangians in field theory;  robustness and equivariance  can be expressed as Euler-Lagrange equations and Noether currents. 
    % \nd{show experiments}
    % , and L-conv outperforms CNN on domains with hidden symmetries, such rotated and scrambled images. 
    % \item \nd{build Lie algebra neural net? } 
    % \item For linear regression, we derive analytical relations between symmetries in orthogonal groups preserving covariance of data. % and 
    % \item We 
    % devise a methodology for automatic extraction of these symmetries. % of generators of the symmetry group.
\end{itemize}

Learning symmetries from data has been studied in limited settings for commutative Lie groups as in \citet{cohen2014learning}, 2D rotations and translations in 
\citet{rao1999learning}, \citet{sohl2010unsupervised} or permutations \citep{anselmi2019symmetry}. 
% Perhaps the closest to our work, in spirit, is 
\citep{zhou2020meta} popose a general method for symmetry discovery. 
% For direct symmetry discovery, 
% \citep{zhou2020meta} uses meta-learning.  
Yet, their weight-sharing scheme and the symmetry generators are very different from ours.
% In our approach the learned generators are directly interpretable as the Lie algebra basis.
Our approach  use much fewer parameters and has a direct interpretation using Lie algebras (SI \ref{ap:symm-disc-lit}).
\citet{benton2020learning} propose Augerino to learn a distribution over data augmentations. 
It also involves Lie algebras, but is restricted to
% Augerino uses data augmentation to transform the input data, which means it is restricting the group to be  
a subgroup of 2D affine transformations and requires matrix logarithm and sampling (SI \ref{ap:symm-disc-lit}). In contrast, our approach is simpler and more general.

% 
% In fact, we show that L-conv can be implemented as a set of graph convolutional networks, with the Lie algebra generators replacing the graph adjacency matrix. 
% Learning symmetries inherently requires more parameters than a pre-defined symmetry. 
% Thus, in problems where the desired symmetry group is known, this knowledge should be used as  inductive bias to directly encode the symmetry (into L-conv or other architectures). 
% Thus symmetry discovery capabilities of L-conv are particularly useful in cases where the symmetries are \textit{unknown} and where we do not have good inductive bias. 

% \section{Related Work}
% \input{secs2/related}

\section{Background}

We review the core concepts %our method 
L-conv
builds upon: equivariance, group convolution and Lie algebras.

\textbf{Notations.}
% For brevity, we will use the Einstein summation convention, where a repeated upper and lower index is summed over, meaning $A^a B_a =\sum_a A^a B_a = A\cdot B $. 
% We will occasionally keep explicit summation $\sum_i$ for clarity. 
% Following Einstein notation, 
Unless explicitly stated, $a$ in $A^a$ is an index, not an exponent.
% For brevity, we will sometimes 
We use the Einstein summation $A^a B_{ab} =\sum_a A^a B_{ab} = [A B]_b $, where a repeated upper and lower index are summed. %: $A^a B_{ab} =\sum_a A^a B_{ab} = [A B]_b $. 
% Matrix product simplifies to $A\cdot B \equiv  A^a B_a$.
\out{
For a linear transformation $A:\R^{d_1} \to \R^{d_2}$ acting on the spatial index or the channel index, we will use one upper and one lower index as in $A^\mu_\nu h_\mu = [A\cdot h]_\nu $. 
We will use $(a,b,c)$ for channels, and $(\mu,\nu,\rho)$ for spatial, and $(i,j,k)$ for Lie algebra basis indices. }
% Indices for the bilinear operators will be similar, like $(i,j,k)$, or $(a,b,c)$ or $(\mu,\nu,\rho)$, as in $A_{i,\mu}^\nu= [A_i]^\nu_\mu$. 

\textbf{Equivariance.}
Let $\mathcal{S}$ be a topological space on which a Lie group $G$ (continuous group) acts from the left, meaning for all $\vx \in \mathcal{S}$ and $g\in G$, $g\vx \in \mathcal{S}$. 
We refer to $\mathcal{S}$ as the base space. 
Let $\mathcal{F}$, the ``feature space'', 
%We will assume the  $\mathcal{F}$ is a 
be the vector space $\mathcal{F} = \R^m$. % over a field $K$ (such as $\R$ or $\C$).
% , which is a vector space on which $G$ acts through a representation $\pi: G \to \mathcal{F}\times \mathcal{F}$, meaning for $y \in \mathcal{F}$ and $g\in G$, $\rho[g]y\in \mathcal{F}$. 
Each data point is a feature map $f:\mathcal{S}\to \mathcal{F}$.
% , or $f\in \mathcal{F}[\mathcal{S}]$, denoting the set of function from $\mathcal{S}\to \mathcal{F}$ by $\mathcal{F}[\mathcal{S}]$. 
% \ry{include the first half of the fig 1 in here with wrapfig and explain the two concepts (1) lift (2) group conv}
% For instance, in images $f(\vx)\in \mathcal{F}$ are the colors at pixel $\vx\in \mathcal{S}$.
The action of $G$ on the input of $f$ induces an action on feature maps.
For ``scalar'' features, for $u \in G$, the transformed features $u\cdot f$ are given by
\begin{align}
    u\cdot f(\vx) &= %\rho[u]
    f(u^{-1}\vx). % & u \in G
    % \label{eq:equivariance-steerable}
    \label{eq:f-G-action}
\end{align}
Denote the space of all functions from $\mathcal{S}$ to $\mathcal{F}$ by $\mathcal{F}^\mathcal{S}$, so that $f\in\mathcal{F}^\mathcal{S}$. 
Let $F$ be a mapping to a new feature space $\mathcal{F}' = \R^{m'}$, meaning  $F:\mathcal{F}^\mathcal{S} \to {\mathcal{F}'}^\mathcal{S}$. 
% mapping taking features $f:\mathcal{S} \to \mathcal{F}$ as input and returning transformed features $F(f):\mathcal{S} \to \mathcal{F}'$.  
We say $F$ is \textit{equivariant} under $G$ if $G$ acts on $\mathcal{F}'$ and for $u\in G$, we have
\begin{align}
    u\cdot (F(f)) &= F(u\cdot f).
    \label{eq:equivairance-general}
\end{align}
\textbf{Group Convolution.}
\citet{kondor2018generalization} showed that $F$ is a linear equivariant map if and only if it performs a group convolution (G-conv).
To define G-conv, we first lift $\vx$ to elements in $G$ \citep{kondor2018generalization}.
% This requires that the map $\mathrm{Lift}:\mathcal{S}\to G$ be injective. 
% This puts the condition on $G$ that $\forall \vx,\vy \in \mathcal{S}$, $\exists g\in G$ such that $\vy = g\vx$ (e.g. $G$ can be 2D translation on images). 
% For example, when $\mathcal{S} = \R$, the group of 1D translations (i.e. addition by a real number $G=(\R,+)$) satisfies this condition.
Specifically, we pick an origin $\vx_0 \in \mathcal{S}$ and replace each point $\vx = g\vx_0$ by $g$. 
% With a slight abuse of notation, 
We will often drop $\vx_0$ for brevity and write
% The feature maps then become 
$f(g)\equiv f(g\vx_0)$. 
% Lift $x\to g\in G$ and let $dv\equiv d\mu(v)$. 
% From \cite{cohen2018intertwiners} Eq (41) 
Let $\kappa: G\to %\mathrm{Hom}( \mathcal{F}', \mathcal{F}) \sim 
\R^{m'} \otimes \R^{m}$ %, meaning $\kappa(g)$ is 
be a linear transformation from $\mathcal{F}$ to $\mathcal{F}'$.
G-conv is defined as 
% As shown in \cite{kondor2018generalization}, the following G-conv is equivariant  
\begin{align}
    [\kappa\star f](g) = \int_G \kappa(g^{-1}v)f(v) dv = \int_G \kappa(v)f(gv) dv,
    \label{eq:G-conv}
\end{align}
We denote the Haar measure on $G$ as $dv\equiv d\mu(v)$ for brevity.

\textbf{Equivariance of G-conv.}
G-conv  in \eqref{eq:G-conv} is equivariant \citep{kondor2018generalization}. 
By definition, for $w\in G$ we have
\begin{align}
    [\kappa\star w\cdot f](g) 
    &= \int_G \kappa(v)w\cdot f(gv) dv 
    %\cr & 
    = \int_G \kappa(v)f(w^{-1}gv) dv\cr
    & = [\kappa\star f](w^{-1}g)
    = w\cdot [\kappa\star f](g) 
    % = \int \kappa(v)f(gv) dv 
    \label{eq:G-conv-equiv}
\end{align}
\out{
In particular, note that when $\kappa(v) = \delta(v,v_0)$, meaning it has support only around $v_0$, $[\kappa \star f](g) = f(gv_0)$ is an equivariant map because $w\cdot f(gv_0) = f(w^{-1}gv_0)$. 

The equivariance literature which discretize $G$ use $v_0$ in a discrete subgroup of $G$. 
% Our approach will be slightly different. 
Instead, we show that choosing $v_0$ close to the identity allows one to cover large parts of $G$ by composing the $v_0$.
}%
Existing works on equivariance networks implement $\int_G$ by discretizing the group or summing over irreps. 
We take a different approach and use the infinitesimal generators of the group.
% is to instead approximate $\kappa(v)$ utilizing the Lie algebra. %, as we elaborate next.
% A major challenge in using existing methods for G-conv is that to implement the integral $\int_G$ one has to encode an approximate version of $G$ into the architecture, as discussed above.  
% For continuous symmetries this step is quite tedious, involving either discretization or truncated sum over irreps \citep{kondor2018clebsch}.
% Another major drawback of these methods is that the symmetry group needs to be known \textit{a priori}. \nd{discuss recent symmetry discovery}
While a Lie group $G$ is infinite, usually it can be generated using a small number of infinitesimal generator, comprising its ``Lie algebra''. %, defined below. 
We use the Lie algebra to introduce a building block to approximate G-conv. 
% The core of our architecture is a G-conv whose kernel has a small support. 
% Next, we review core concepts about Lie algebras. 
Figure \ref{fig:Lie-group-S} visualizes a Lie group, Lie algebra and the concept we discuss below. 

\out{
In many domains, such as physical systems, 
exhibit the data often have a low dimensional representation. 
The symmetries hidden in the data should be related to the dimensionality of the underlying space and have a small number of infinitesimal generators. % much lower than the dimensionality of the data itself. 
In order to design an architecture for all \textit{continuous symmetries} and can automatically \textit{discover symmetries} from data, we propose to design the architecture in terms of the Lie algebra basis.
}%%%%
\begin{figure}
    \centering
% \begin{wrapfigure}{R}{.3\textwidth}
    % \includegraphics[width = 1\linewidth, trim=0 0 0 20pt]{figs/L-conv-sketch-2020-10-01.pdf}
    % \includegraphics[width = 1\linewidth]{figs2/Lie-group.pdf}
    % \vspace{-20pt}
    % \includegraphics[width = .8\linewidth]{figs2/L-conv-GS-spaces-2.pdf}
    \includegraphics[width = .8\linewidth]{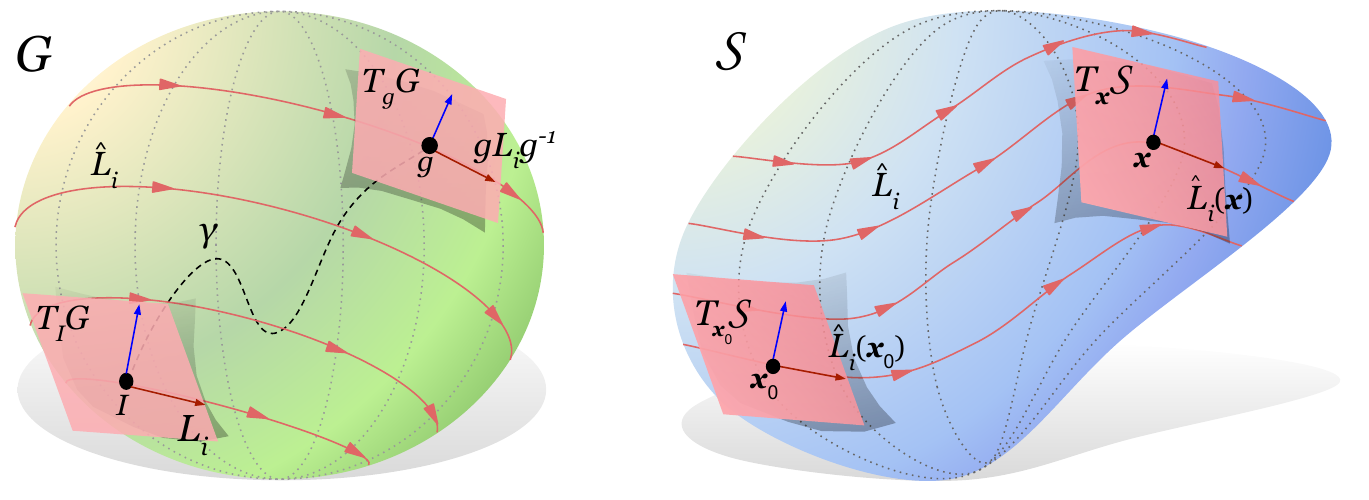}
    \caption{
    \textbf{Lie group and Lie algebra:}
    Illustration of the group manifold of a Lie group $G$ (left).  
    The Lie algebra $\mathfrak{g}=T_IG$ is the tangent space at the identity $I$. 
    $L_i$ are a basis for $T_IG$.
    If $G$ is connected, $\forall g\in G$ there exist %an infinite number of 
    paths like $\gamma$ from $I$ to $g$ and $g$ can be written as a path-ordered integral $g=P\exp[\int_\gamma dt^i L_i]$. 
    % The tangent space $T_gG$ has a basis $gL_ig^{-1}$, the pushforward by $g$.  
    \textbf{
    Base space}
    Right is a schematic of the base space 
    $\mathcal{S}$ as a manifold. 
    The lift $\vx=g\vx_0$ takes $\vx\in \mathcal{S}$ to $g\in G$, and maps the tangent spaces $T_\vx \mathcal{S}\to T_gG$. 
    Each Lie algebra basis $L_{i}\in \mathfrak{g}=T_IG$ generates a vector field $\hat{L}_i$ on the tangent bundle $TG$ via the pushforward $\hat{L}_i(g) = gL_{i}g^{-1}$.
    Via the lift, $L_i$ also generates a vector field $\hat{L}_i=\hat{L}_i^\alpha(\vx)\ro_\alpha = [gL_i\vx_0]^\alpha \ro_\alpha$. %, with $\vx=g\vx_0$. 
    }
    \label{fig:Lie-group-S}
    % \vspace{-120pt}
% \end{wrapfigure}
\end{figure}

\textbf{Lie algebra.}
Let $G$ be a Lie group, which includes common continuous groups.
Group elements $u\in G$ infinitesimally close to the identity element $I$ can be written as $u \approx I +  \eps^i L_i$ (note Einstein summation), where  
% We can find a basis 
$L_i\in\mathfrak{g}$ with the Lie algebra $\mathfrak{g}= T_IG$ is the tangent space of $G$ at the identity element.  
The Lie algebra has the property that it is closed under a
Lie bracket $[\cdot,\cdot]: \mathfrak{g}\times \mathfrak{g} \to \mathfrak{g}$
\begin{align}
    % u & = \exp(t\cdot L) \cr
    % u & \approx I + \eps\cdot L\cr
    [L_i,L_j] &= %\sum_k
    {c_{ij}}^k L_k,
    \label{eq:Lie-commutator}
\end{align}
which is skew-symmetric and satisfies the Jacobi identity.
Here the coefficients ${c_{ij}}^k\in \R$ or $\C$ are called the structure constants of the Lie algebra.
For matrix representations of $\mathfrak{g}$, $[L_i,L_j] = L_iL_j-L_jL_i$ is the commutator.
The $L_i$ are called the infinitesimal generators of the Lie group. 
% They define vector fields spanning the tangent space of the manifold of $G$ near its identity element $I$.  
% with a connected group manifold (e.g. for 2D rotations, $SO(2)$, the group manifold is the circle $S^1$). 

\textbf{Exponential map.}  
If the manifold of $G$ is connected
\footnote{
When $G$ has multiple connected components, these results hold for the component containing $I$, and generalize easily for mutli-component groups such as $\Z_k \otimes G $ \citep{finzi2021practical}.
% \nd{discuss in SI?}
},
% \ry{refer to Rao 99 for structure}
an exponential map $\exp : \mathfrak{g}\to G$ can be defined such that $g = \exp[t^i L_i] 
% =\lim_{N\to \infty}(I+t^i L_i/N)^N 
\in G$.
For matrix groups, if $G$ is connected and compact, the matrix exponential % defined through a Taylor expansion 
is such a map and it is surjective. 
For most other groups (except $\mathrm{GL}_d(\mathbb{C})$ and nilpotent groups) it is not surjective.
Nevertheless, for any connected group every $g\in G$ can be written as a product $g=\prod_a \exp[t_a^i L_i]$ %using the matrix exponential 
\citep{hall2015lie}.
Making $t^i_a$ infinitesimal steps $dt^i(s)$ tangent to a path $\gamma$ from $I$ to $g$ on $G$ % allows us to write any $g\in G$ as 
yields %what's called a 
the surjective path-ordered exponential in physics, denoted as $g=P\exp[\int_\gamma dt^i L_i]$ %, which is surjective 
(SI \ref{ap:theory-extended}, and see Time-ordering in \citet[p143]{weinberg1995quantum}). 
% In particular, for any $u\in G$ we can write $u = u_\eps v$  with $u_\eps = (I+\eps^iL_i)$. 
% This allows us to Taylor expand functions on $G$. 
% Figure \ref{fig:Lie-group-S} shows an illustration of a Lie group and Lie algebra. 
% We will use this fact to modify existing results about group equivariant architectures to introduce L-conv. 

\textbf{Pushforward.}
$L_i\in T_IG$ can be pushed forward to $\hat{L}_i(g)=gL_ig^{-1}\in T_gG $ to form a basis for $T_gG$, satisfying the same Lie algebra 
% $[L_i^{(g)},L_j^{(g)}]={c_{ij}}^k L_k^{(g)}$.
$[\hat{L}_i(g),\hat{L}_j(g)]={c_{ij}}^k \hat{L}_k(g)$.
The manifold of $G$ together with the set of all $T_gG$ attached to each $g$ forms the tangent bundle $TG$, a type of fiber bundle \citep{lee2009manifolds}. 
$\hat{L}_i$ is a vector field on $TG$.
The lift maps $\hat{L}_i$ to an equivalent vector field on $T\mathcal{S}$, which we will also denote by $\hat{L}_i$.
Figure \ref{fig:Lie-group-S} illustrates the flow of these vector fields on $TG$ and $T\mathcal{S}$.
% \nd{Add tangent bundle and discretized $\mathcal{S}$ to fig.}

% \input{secs/background}
% \input{secs/equivariance}

% \input{secs/G-conv}
% \section{Equivariant Architecture}

% \ry{Show a picture of generation in a set, made with mathematical \href{https://keenan.is/illustrating/2017/03/03/illustrating-lie-groups/}{Illustration}}
% \subsection{Continuous Symmetry, Lie Groups and Lie Algebras}
% \nd{Shorten}
% We will focus on Lie groups that are subsets $G\subseteq \mathrm{GL}_d(\R)$ of the general linear group in $d$ dimensions.
% Our results generalize trivially to other number fields such as $\C$. 
% For any Lie group $G\subset \mathrm{GL}_d(\R)$,

\section{Lie Algebra Convolutional Network}
We can use 
the Lie algebra basis $L_i\in \mathfrak{g}$ to construct the Lie group $G$ with the exponential map. 
Similarly, we show that  Lie algebras can also serve as building blocks to construct G-conv layers.
We propose the Lie algebra convolutional network (L-conv).  The key idea is to approximate the kernel $\kappa(u)$ using localized kernels which can be constructed using the Lie algebra (Fig. \ref{fig:G-conv2L-conv}).
This is possible because the exponential map is a generalization of a Taylor expansion. 
% This process is closely mimics the universal approximation theorem of neural networks. 
We   show that a G-conv whose kernel is concentrated near the identity can be expanded in the Lie algebra.  
% Concretely, we will use the Lie algebra to expand G-conv \eqref{eq:G-conv} near the identity $I$. 
% This yields basic building blocks from which an arbitrary G-conv can be constructed.
% is to instead approximate $\kappa(v)$ utilizing the Lie algebra.
% \subsection{Expanding in G-conv in the Lie algebra}
% We consider a G-conv \eqref{eq:G-conv} with a kernel localized near identity $I$.

Let $\delta_\eta (u)\in \R$ denote a normalized \textit{localized kernel}, meaning $\int_G \delta_\eta(g)dg=1 $, and with support on a small neighborhood of size $\eta$ near the identity $I$ (i.e., $\delta_\eta(I+\eps^iL_i) \to 0 $ if $\|\eps\|^2>\eta^2$).
Let $\kappa_0 (u) = W^0 \delta_\eta (u)$, where $W^0\in \R^{m'} \otimes \R^{m}$ are constants. 
% and $\delta_\eta(u) \in \R$ is a normalized kernel, $\int_G \delta_\eta(g)dg=1 $,
% with support on a small neighborhood of size $\eta$ near $I$, meaning $\delta_\eta(I+\eps^iL_i) \to 0 $ if $\|\eps\|^2>\eta^2$.
The localized kernels $\kappa_0$ can be used to approximate G-conv. 
We first derive the expression for a G-conv whose kernel is $\kappa_0$. 

\out{
Note that with $f(g)\in \R^m$, each $\eps^i \in\R^{m} \otimes \R^m $ is a matrix. 
With indices, $f(gv_\eps) $ is given by 
\begin{align}
    [f(gv_\eps)]^a &= \sum_b f^b(g(\delta^a_b + [\eps^i]^a_b L_i)) 
\end{align}
}%%%%

\textbf{Linear expansion of G-conv with localized kernel.}
% \ry{simplify the notations} 
We can expand a G-conv whose kernel is $\kappa_0(u) = W^0 \delta_\eta(u)$ in the Lie algebra of $G$ to linear order. 
With $v_\eps = I+\eps^i L_i $, we have
%to obtain the L-conv architecture.
(see SI \ref{ap:theory-extended})%, \eqref{eq:L-conv-basic})
\begin{align}
    Q[f](g)&=[\kappa_0 \star f](g) 
    = \int_G dv \kappa_0 (v) f(gv )
    =\int_{\|\eps\| <\eta } dv_\eps  \kappa_0 (v_\eps) f(gv_\eps )\cr
    % &= \int d\eps \delta_\eta (I+\eps^i L_i) f(g(I+\eps^i L_i))\cr
    % &= \int d\eps \delta_\eta (I+\eps^i L_i) f(g+\eps^i gL_i)\cr
    &
    = W^0\int d\eps \delta_\eta(v_\eps) %(I+\eps^i L_i) 
    \br{f(g)+ \eps^i g L_i \cdot {d\over dg}  f(g) + O(\eps^2) } %\bigg|_{u\to g}
    \cr
    % &\approx \int d\eps \kappa_0 (I+\eps^i L_i) \br{I+ \eps^i g L_i \cdot {d\over dg} } f(g)
    % \cr
    % & \approx W^0\br{I + \ba{\eps}^i g L_i\cdot {d\over dg} } f(g)
    % \label{eq:L-conv-basic} 
    &= W^0 \br{I + \ba{\eps}^i g L_i\cdot {d\over dg} }f(g) +O(\eta^2) %\cr 
    % \mbox{\nd{out}}
    % &= [W^0]_b f^a\pa{g\pa{\delta^a_b+ [\ba{\eps}^i]_a^b L_i}} +O(\eta^2)
    \label{eq:L-conv}
\end{align}
% where $\delta^a_b$ is the Kronecker delta and, 
where $W^0\in \R^{m'} \otimes \R^m$, and
using  %$\kappa_0(v)=c \delta_\eta(v)$ with $c\in \R^{m'}\otimes \R^m$ 
$\int_G \delta_\eta(g)dg= \int d\eps \delta_\eta(v_\eps)=1 $ and $\eps^i\in \R^{m}\otimes \R^m$, we defined
\begin{align}
    % W^0 &= %\int_G dv \kappa_0(v) = 
    % \int d\eps\kappa_0 (I+\eps^iL_i) \in \R^{m'} \otimes \R^m, &
    \ba{\eps}^i &= \int d\eps \delta_\eta (v_\eps) %(I+\eps^iL_i) 
    \eps^i \in \R^{m} \otimes \R^m. 
\end{align}
Note that because $\|\eps\|<\eta$ we also have $\|\ba{\eps}\|<\eta $ (SI \ref{ap:theory-extended}, \eqref{eq:O-eta-p}). 
Here $d\eps $ is the integration measure on the Lie algebra $\mathfrak{g}=T_IG$ induced by the Haar measure $dv_\eps $ on $G$.
% where we used the fact that the Haar measure $dv$ induces an integration measure on the Lie algebra, which we denoted by $d\eps$.

\textbf{Interpreting the derivatives.}
In a matrix representation of $G$, we have $g L_i\cdot {df\over dg} = [g L_i]_\alpha^\beta {df\over dg_\alpha^\beta} = \Tr{[g L_i]^T {df\over dg} }$.
% $g L_i\cdot {df\over dg} = \sum_{a,b}[g L_i]_{ab} {df\over dg_{ab}} = \Tr{[g L_i]^T {df\over dg} }$.
% Note that in $g(I+\eps^iL_i)\vx_0$, the $gL_i\vx_0 = \hat{L}_i(g)\vx$ come from the pushforward $L_i^{(g)}=gL_ig^{-1} \in T_gG$.
%of $L_i\in T_IG$ from the tangent space at $I$ to the tangent space at $g$, meaning $gL_ig^{-1} \in T_gG$. 
% This is seen from the fact that $gv_\eps$ is moving from $g$ by $\eps L_i$, hence in $T_gG$. 
% We will show this explicitly in following examples. 
% \Eqref{eq:L-conv-basic} is the core of the architecture we are proposing, the Lie algebra convolution or \textbf{L-conv}.
% The $gL_i\cdot df/dg$ in \eqref{eq:L-conv} 
This can be written in terms of partial derivatives $\ro_\alpha f(\vx) = \ro f/\ro \vx^\alpha$ as follows.
% In general, 
Using $\vx^\rho = g^\rho_\sigma \vx_0^\sigma$, we have ${df(g\vx_0)\over dg^\alpha_\beta} =\vx_0^\beta \ro_\alpha f(\vx) $, and so
% \ry{add a cartoon}
\begin{align}
    % \vx^\mu &= g^\rho_\sigma \vx_0^\sigma &
    % {df(g\vx_0)\over dg^\alpha_\beta}
    % &= {d(g^\rho_\sigma \vx_0^\sigma) \over dg^\alpha_\beta}\ro_\rho f(\vx) = \vx_0^\beta \ro_\alpha f(\vx)
    % \label{eq:dgx0-dg}
    % \\
    \hat{L}_if(\vx)\equiv gL_i\cdot {df\over dg} &= [gL_i]^\alpha_\beta \vx_0^\beta \ro_\alpha f(\vx) = [gL_i\vx_0]\cdot \del f(\vx) %= \hat{L}_if(\vx)
    \label{eq:dfdg-general}
\end{align}
Hence, for each $L_i$, the pushforward $gL_ig^{-1}$ generates a flow on $\mathcal{S}$ through the vector field $\hat{L}_i\equiv gL_i\cdot d/dg = [gL_ig^{-1} \vx]^\alpha \ro_\alpha$ (Fig. \ref{fig:Lie-group-S}). 
% $\hat{L}_i$ is related to the Maurer-Cartan form $\omega = g^{-1} \ro_\alpha g d\vx^\alpha$, which encodes the pushforward.
% Write $f(gv_\eps) = f(g+ \eta(\eps)^\alpha \ro_\alpha g)$. 
% We have $ \eta^\alpha \ro_\alpha g = \eps^i g L_i$, so $ \eps^i L_i \eta^\alpha g^{-1} \ro_\alpha g = \eta \cdot \omega $.
% The Maurer-Cartan form defines a way to parallel transport (pushforward) and define a basis frame on  
% Vector fields are sections of the tangent bundle $T\mathcal{S}$.
% We will encounter $gL_i\vx_0$ again on a discretized $T\mathcal{S}$ again below. 
\out{
Being a vector field $\hat{L}_i\in T\mathcal{S}$ (i.e. 1-tensor), $\hat{L}_i$ is basis independent, meaning 
for $v \in G$, $\hat{L}_i(v\vx) = \hat{L}_i$. 
Its components transform as $[\hat{L}_i(v \vx)]^\alpha =[vgL_i\vx_0]^\alpha = v^\alpha_\beta \hat{L}_i(\vx)^\beta $, while the partial transforms as $\ro / \ro[v\vx]^\alpha = [v^{-1}]^{\gamma}_\alpha \ro_\gamma$. 
}%%%
% \ry{formally define a L-conv layer, put Equation 10 here}

\textbf{Lie algebra convolutional (L-conv) layer.}
% \out{
% \textbf{L-conv Layer} 
\Eqref{eq:L-conv} states that for a kernel localized near the identity, the effect of the kernel can be summarized in $W^0$ and $\ba{\eps}^i\hat{L}_i$. 
Note that we do not need to perform the integral over $G$ explicitly anymore. 
Instead of working with a kernel $\kappa_0$, we only need to specify $W^0$ and $\ba{\eps}^i$. 
% The construction in \eqref{eq:L-conv} turns out to be the basic building block of G-conv, as we show below. 
Hence, in general, we define the Lie algebra convolution (L-conv) as  
\begin{align}
    Q[f](\vx)%&= \br{W^0 + W^i g L_i\cdot {d\over dg} } f(g) 
    % &= W^0 \br{I + \ba{\eps}^i g L_i\cdot {d\over dg} }f(g\vx_0) 
    &= W^0\br{I+\ba{\eps}^i\hat{L}_i} f(\vx)\cr 
    &= W^0 \br{I + \ba{\eps}^i  [gL_i\vx_0]^\alpha \ro_\alpha }f(\vx) %\cr 
    % \mbox{\nd{out}}
    % &= [W^0]_b f^a\pa{g\pa{\delta^b_a+ [\ba{\eps}^i]_a^b L_i}\vx_0} +O(\ba{\eps}^2)
    \label{eq:L-conv-def0}
\end{align}
% }%%%%
% For $f(g)\in \R^m $, 
% For brevity, we define $W^i\equiv W^0 \ba{\eps}^i$.
% , and found using the Moore-Penrose inverse $\tilde{W}_0^{-1} = (W^{0T}W^0)^{-1}W^0$. 

Being an expansion of G-conv, L-conv inherits the equivariance of G-conv, as we show next.

\begin{proposition}[Equivariance of L-conv]\label{prop:L-conv-equiv}
    With assumptions above,
    L-conv is equivariant under $G$. % to  $O(\eta^2)$. 
\end{proposition}

% \nd{redo proof with expansion}
\begin{proof}
\out{
    From \eqref{eq:L-conv} it follows that L-conv is equivariant up to $O(\ba{\eps}^2) \sim O(\eta^2)$, as for $w\in G$
\begin{align}
    w\cdot Q[f](g)&= Q[f](w^{-1}g) 
    = W^0 f\pa{w^{-1} g\pa{I+ \ba{\eps}^i L_i}} +O(\eta^2)
    %= W^0 w \cdot f\pa{ g v_\eps} 
    = Q[w\cdot f](g).
    \label{eq:L-conv-equiv}
\end{align}
\nd{2nd way}
}%%%%%
First, note that the components of $\hat{L}_i$ transform as $[\hat{L}_i(v \vx)]^\alpha =[vgL_i\vx_0]^\alpha = v^\alpha_\beta \hat{L}_i(\vx)^\beta $, while the partial transforms as $\ro / \ro[v\vx]^\alpha = [v^{-1}]^{\gamma}_\alpha \ro_\gamma$.
As a result in $\hat{L}_i = [gL_i\vx_0]^\alpha \ro_\alpha $ all factors of $v$ cancel, meaning 
for $v \in G$, $\hat{L}_i(v\vx) = \hat{L}_i(\vx)$. 
% To prove the equivariance of the expanded version of L-conv, 
This is because of the fact that $\hat{L}_i\in T\mathcal{S}$ is a vector field  (i.e. 1-tensor) and, thus, invariant under change of basis. 
% Its components transform as $[\hat{L}_i(v \vx)]^\alpha =[vgL_i\vx_0]^\alpha = v^\alpha_\beta \hat{L}_i(\vx)^\beta $, while the partial transforms as $\ro / \ro[v\vx]^\alpha = [v^{-1}]^{\gamma}_\alpha \ro_\gamma$ resulting in all factors of $v$ cancelling out. 
Plugging into \eqref{eq:L-conv-def0}, for $w\in G$
\begin{align}
    w\cdot Q[f](\vx)&= Q[f](w^{-1}\vx) = W^0\br{I+\ba{\eps}^i \hat{L}_i(w^{-1}\vx)} f(w^{-1}\vx) \cr
    &=  W^0\br{I+\ba{\eps}^i \hat{L}_i(g)} f(w^{-1}\vx) 
    =  W^0\br{I+\ba{\eps}^i \hat{L}_i(g)} w\cdot f(\vx)%\cr &
    = Q[w\cdot f](\vx)
    \label{eq:L-conv-equiv}
\end{align}
which proves L-conv is equivariant. 
\end{proof}

% Next, we provide some explicit example of the form of L-conv for familiar continuous symmetries. 

% \ry{make this a prop}
% \textbf{Equivariance of L-conv}

\out{
\ry{explain the Taylor expansion in (25)}
% $\delta_\eta \star f$ may be expanded  near identity as
\begin{align}
    Q[f](g)&\equiv c[\delta_\eta \star f](g) 
    = %\approx
    c\int d\eps \delta_\eta (I+\eps^i L_i) \br{I+ \eps^i g L_i \cdot {d\over dg} } f(g) %\bigg|_{u\to g}
    +O(\eta^2)
    \cr
    &= W^0 \br{I + \ba{\eps}^i g L_i\cdot {d\over dg} }f(g) +O(\eta^2)\cr 
    &= [W^0]_b f^a\pa{g\pa{I+ [\ba{\eps}^i]_a^b L_i}} +O(\eta^2)
    \label{eq:L-conv}
    % = \br{W^0 + W^i g L_i\cdot {d\over dg} } f(g)
    % \label{eq:L-conv-basic} 
\end{align}
}%%%% 
% \subsection{Interpretation and Examples \label{sec:interpret-continuous} }
% \paragraph{General case} 

\textbf{Examples.}
Using \eqref{eq:dfdg-general} we can calculate L-conv for specific groups (details in SI \ref{ap:examples-continuous}). 
For {translations} $G=T_n = (\R^n,+)$,  we find the generators become simple partial derivatives $\hat{L}_i = \ro_i$ (SI \ref{ap:example-Tn}), yielding $f(\vx) + \eps^\alpha \ro_\alpha f(\vx)$. 
For {2D rotations} (SI \ref{ap:example-so2}) 
the generator $\hat{L} \equiv \pa{x \ro_y - y \ro_x } = \ro_\theta$, which is the angular momentum operator about the z-axis in quantum mechanics and field theories. %, which generates rotations around the $z$ axis. 
For rotations with scaling, $G= SO(2)\times \R^+$, 
we have two $L_i$, one $\hat{L}_\theta=\ro_\theta $ from $so(2)$ and a scaling with $L_r = I$, yielding
$\hat{L}_r = x\ro_x+y\ro_y= r\ro_r $. % (SI \ref{ap:examples-continuous}) 
Next, we discuss the form of L-conv on discrete data.

\subsection{Approximating G-conv using L-conv \label{sec:G-conv2L-conv} }
\begin{figure}
    \centering
    \includegraphics[width=.8\linewidth]{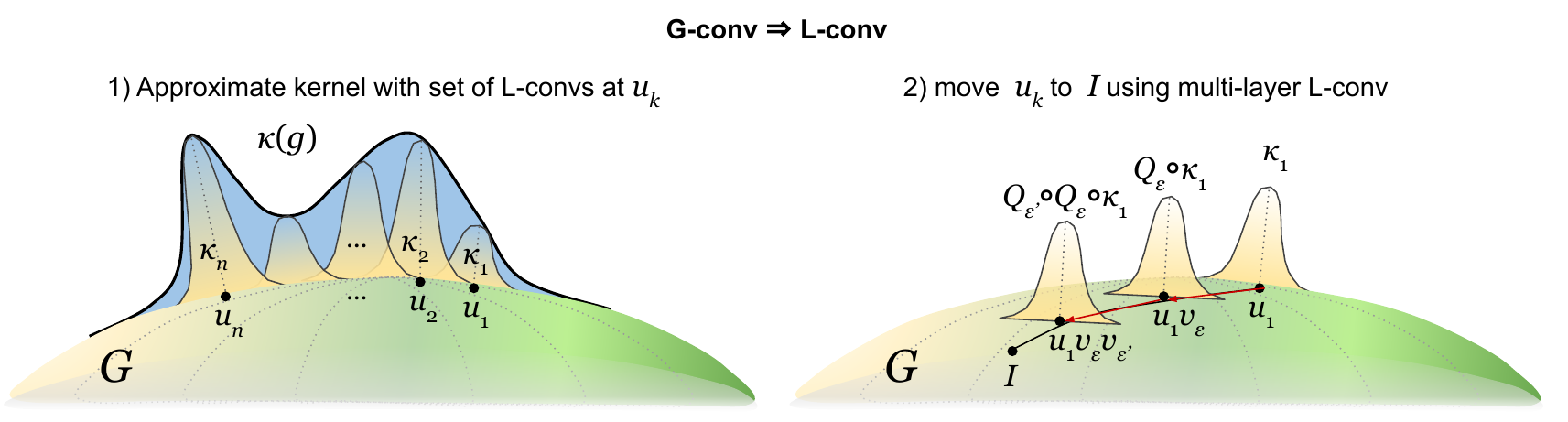}
    \caption{
    Sketch of the procedure for approximating G-conv using L-conv. 
    First, the kernel is written as the sum of a number of localized kernels $\kappa_k$ with support around $u_k$ (left). 
    Each of the $\kappa_k$ is then moved toward identity by composing multiple L-conv layers $Q_{\eps'}\circ Q_\eps\dots \kappa_k $ (right). 
    % The more layers used, the more accurate the 
    }
    \label{fig:G-conv2L-conv}
\end{figure}

L-conv can be used as a basic building block to construct G-conv with more general kernels. 
Figure \ref{fig:G-conv2L-conv} sketches the argument described here (see also SI \ref{ap:G-conv2L-conv}).
\begin{theorem}[G-conv from L-convs]
    \thlabel{thm:Gconv2Lconv}
    G-conv \eqref{eq:G-conv} can be approximated using L-conv layers.
\end{theorem}
\begin{proof}
    The procedure involves two steps, as illustrated in Fig. \ref{fig:G-conv2L-conv}: 1) approximate the kernel using localized kernels as the $\delta_\eta$ in L-conv; 2) move the kernels towards identity using multiple L-conv layers. 
    The following lemma outline the details.
\end{proof}

% To do so, we first note that any function $\kappa(v)$ can be approximated to arbitrary accuracy using a set of local kernels $\delta_\eta $ (using the universal approximation theorem of neural networks \citep{hornik1989multilayer,cybenko1989approximation}).

\begin{lemma}[Approximating the kernel] 
    \label{lem:kernel-approx}
    % \nd{Robin, please check:}
    Let the kernel $\kappa:G\to \mathcal{F}'\otimes \mathcal{F}$ with $\int_G \|\kappa(g)\|^2 dg < \infty$ be continuously differentiable with $\|d\kappa(g)/dg\|^2<\xi^2 $, and with compact support over $G_0\subset G$.  
    Let $\kappa_k(g)=c_k\delta_\eta(u_k^{-1}g) $ be a set of $N$ kernels with support on an $\eta$ neighborhood of $u_k\in G$.  
    Then there exist $c_k \in  \mathcal{F}'\otimes \mathcal{F}$ and $u_k \in G$ such that $\tilde{\kappa}= \sum_{k=1}^N \kappa_k$ approximates $\kappa$, meaning
    $\int_G \|\kappa(g)-\tilde{\kappa}(g)\|^2dg < \zeta^2 $ for arbitrary small $\zeta\in \R_+$.
\end{lemma}
% \textbf{Approximating the kernel}
\begin{proof}
See SI \ref{ap:G-conv2L-conv} for details. 
The intuition is similar to the universal approximation  theorem for neural networks \citep{hornik1989multilayer,cybenko1989approximation}, only generalized to  a group manifold instead of $\R$. 
Let $B_0 $ be the set of $v_\eps = I+\eps^iL_i\in \mathfrak{g}$, with $\|\eps\|^2<\eta^2$.
Choose a set of $u_k\in G$ such that the neighborhoods $B_k = u_k B_0 \subset G$ cover the support $G_0$ of $\kappa$. 
The bound $\|d\kappa(g)/dg\|^2<\xi^2 $ means that on small enough neighborhoods $B_k\subset G$, for any two $u,v \in B_k$ we have $\|\kappa(u)-\kappa(v)\|^2 \leq \eta^2 \xi^2 $,
where $|G_0|$ is the volume of the support of $\kappa$. 
%, with $\eta$ being the longest geodesic path on $B_k$. 
Hence, for $g\in B_k$, $\kappa(g)$ can be approximated with $\kappa_k(g) =\kappa(u_k) \delta_\eta(u_k^{-1}g)$, with normalized localized kernels   $\delta_\eta(g)$, and any element $u_k \in B_k$.
We show that the approximation error of using $\tilde{\kappa}= \sum_k \kappa_k$ to approximate $\kappa$ is bounded by $ \int_G dg \|\kappa(g)-\tilde{\kappa}(g)\|^2 < |G_0|\eta^2 \xi^2$. 
Any desired error bound $\zeta$ can then be attained by choosing small enough $\eta$ for neighborhood sizes.
\end{proof}

Thus, we can approximate a large class of kernels as $\kappa(g) \approx \sum_k \kappa_k(g)$ where the local kernels $\kappa_k (g) = c_k \delta_\eta (  u_k^{-1}g)$ have
support only on an $\eta$ neighborhood of $u_k \in G $. 
Here $c_k\in \R^{m'} \otimes \R^{m}$ are constants and $\delta_\eta(u)$ is as in \eqref{eq:L-conv}.  
Using this, G-conv \eqref{eq:G-conv} becomes
\begin{align}
    [\kappa \star f](g)
    % & = \sum_k [\kappa_k \star f](g) 
    &= \sum_k c_k \int dv \delta_\eta(u_k^{-1} v) f(gv) 
    % &= \sum_k c_k \int dv \delta_\eta(v ) f(gu_k v) 
    = \sum_k c_k [\delta_\eta \star f] (gu_k).
    \label{eq:c-k-delta}
\end{align}
% As we showed in \eqref{eq:L-conv-basic}, $[\delta_\eta \star f](g)$ is the definition of L-conv. 
% \ry{put this as a theorem, show as an approximation error}
The kernels $\kappa_k$ are localized around $u_k$, whereas in L-conv the kernel is around identity. 
We can compose L-conv layers to move $\kappa_k$ from $u_k$ to identity. 

\out{
Can we write the movement of a $\delta$ on the group explicitly as the action of the exp map on a $\delta$ near identity? 
For a group element it is clear, but for a function over the group? 
Should we use the fact that we have convolution, i.e. integral over the group? 
With $v = P\exp[\int_\gamma dt^i L_i]$ we have 
\begin{align}
    f(g) &= \int_G du f(u)\delta(u^{-1}g) \cr 
    f(vg) &= \int_G du f(u)\delta(u^{-1}vg) = \int_G du' f(v^{-1}u')\delta({u'}^{-1}g) 
\end{align}
Can we use the integral 
}
\begin{lemma}[Moving kernels to identity]
    % The local kernel 
    $\kappa_k$ can be moved near identity using a multilayer L-conv.  
\end{lemma}
\begin{proof}
In \eqref{eq:c-k-delta}, write $u_k = v_\eps u'_k$, with $v_\eps = I+ \eps^i L_i \in \mathfrak{g}$. 
Using the definition \eqref{eq:L-conv-def0} an L-conv layer $Q_{\eps} = I-\eps^i \hat{L}_i$ performs a first order Taylor expansion (SI \ref{ap:G-conv2L-conv}) and so 
$Q_\eps[\delta_\eta]({u'}_k^{-1} v)  = \delta_\eta(u_k^{-1} v) +O(\eps^2)$. 
\out{
\begin{align}
    \delta_\eta(u_k^{-1} v) &= \left.\br{I+ \eps^i g L_i \cdot {d\over dg} }\delta_\eta(g)\right|_{g\to {u'}_k^{-1} v} + O(\eps^2) \approx Q_\eps[\delta_\eta]({u'}_k^{-1} v)
\end{align}
where $Q_{\eps} = I-\eps^i \hat{L}_i$ is an L-conv layer with $W^0=I$. 
}%%%
Thus, 
% This means that a 
applying one L-conv layer % with the parameters above 
moves the localized kernel along $v_\eps $ on $G$.   
% Iterating this further, 
Writing $u_k$ as the product of a set of small group elements $u_k = \prod_{a=1}^p v_a$, with $v_a = I+ \eps_a^i L_i \in \mathfrak{g}$. 
Defining L-conv layers $Q_a = I-\eps_a^i \hat{L}_i $,
we can write 
\begin{align}
    \kappa_k (g) &\approx c_k Q_p \circ \cdots \circ Q_1 \circ \delta_\eta (g)  
\end{align}
meaning $\kappa_k$ localized around $u_k$ can be written as a $p$ layer L-conv acting on a kernel  $\delta_\eta(g)$, localized around the identity of the group. 
\out{
Thus, by applying multiple L-conv layers we can find $Q_p \circ \dots Q_1\circ \kappa_k\approx c_k\delta_\eta $. 
This can always be done using a set of $v_{\eps_a} = (I+\ba{\eps}_a^iL_i)$ such that $u_k \approx \prod_{a=1}^p v_{\eps_a}$. 
}%%%%
With $\|\eps_a\|<\eta$, the error in $u_k$ is $O(\eta^{p+1})$.
\end{proof}

Thus, we conclude that any G-conv \eqref{eq:G-conv} can be approximated by multilayer L-conv. 
% We can take this result even further, 
Furthermore, for compact $G$, using the theorem in \citet{kondor2018generalization}, we can show that any equivariant feedforward neural network can be approximated using multilayer L-conv with nonlinearities. 

\textbf{Equivariance of nonlinearity.}
Pointwise nonlinearities give equivariant maps between scalar feature maps. %\nd{ref instead?} 
To see this, 
let $\sigma:\R \to \R $.  We extend $\sigma: \gF \to \gF$ by applying $\sigma$ component-wise. Let $f:\mathcal{S}\to \mathcal{F}$ be a scalar feature map (i.e., $g\cdot f(\vx) = f(g^{-1}\vx)$).  Then 
\[
g \cdot (\sigma \circ (f))(\vx) = 
\sigma \circ (f)( g^{-1} \vx) = 
\sigma \circ (g \cdot f)(\vx).
\]
Since the composition of equivariant maps is equivariant, given equivariant linear mapping $Q:\mathcal{F}^\gS\to \mathcal{F}^{\prime\gS}$ (i.e. $g\cdot Q[f]= Q[g\cdot f]$), the layer $f \mapsto \sigma \circ Q[f]$ is equivariant.
% Then, $\sigma(Q[f](\vx))$ is still a scalar and for $g\in G$ we have 
% \begin{align}
%     % g\cdot \sigma\circ f^a(\vx)& = \sigma\circ f^a(g^{-1}\vx) = \sigma\circ g\cdot f^a(\vx) \cr
%     g\cdot \sigma\circ Q[f](\vx)& = \sigma\circ Q[f](g^{-1}\vx) = \sigma\circ g\cdot Q[f](\vx) 
% \end{align}
% meaning, point-wise nonlinearity $\sigma$ is equivariant for scalar $Q[f]$. 
Hence we have the corollary:

\begin{corollary}%[Feedforward NN]
    %Multilayer $G$-conv with pointwise non-linearity can be approximated by multilayer $L$-conv with pointwise non-linearity.
    Assume $G$ is compact and acts on $\gS$ transitively.  Then any equivariant feedforward neural network (FNN) can be approximated using multilayer L-conv with point-wise nonlinearities. 
\end{corollary}
\begin{proof}
%This follows from Theorem 1 in  \citet{kondor2018generalization} and the our Theorem \ref{thm:Gconv2Lconv}. 
A FNN is defined as $\sigma_p \circ F_p[\cdots [ \sigma_1\circ  F_1  [f]](\vx)$ where $F_k$ are linear and $\sigma_k$ are point-wise nonlinearities.
By Theorem 1 of \citet{kondor2018generalization}, any linear layer in the equivariant FNN is a G-conv, which by Theorem \ref{thm:Gconv2Lconv} can be approximated by multilayer L-conv.
Therefore, multilayer L-conv with nonlinearity can approximate any equivariant FNN.
\end{proof}

% As a final note, we should point out 
Finally,  to our knowledge it is not known whether \textit{every equivariant function} can be approximated by equivariant FNN for a Lie group $G$. 
Hence, the corollary above is \textit{not} a universal approximation theorem for equivariant scalar functions in terms of L-conv.  However, it does show that multilayer $L$-conv is equally expressive as other equivariant networks. 
%
% \nd{Mention/review Rao's result and how it demonstrates this theorem for 1D translation.}
\out{
We conducted small controlled experiments to verify how multilayer L-conv approximates G-conv (SI \ref{ap:experiments}). 
We briefly discuss them here. 

\paragraph{Learning symmetries using L-conv}
\citet{rao1999learning} introduced a basic version of L-conv and showed that it can learn 1D translation and 2D rotation. 
% They showed that t
The learned 
% infinitesimal generator 
$L_i$ for 1D translation reproduced finite translation well using $(I+\eps L)^N$, which is $N$ recursive L-conv layers. 
Hence, their results prove that L-conv can be used to approximate CNN, as well as $SO(2)$ G-conv. 
We discuss the details of L-conv approximating CNN below in sec. \ref{sec:approx}.
We also conducted more complex experiments using recursive L-conv to learn large rotation angle between two images 
% \nd{revise if experiments moved to SI}
(SI \ref{ap:experiments}).
% These experiments required L-conv to approximate rotations by large angles, which it did successfully. 
% The learned $L$ also has the right structure (Fig. \ref{fig:L-so2-combined}).
\out{
% We can use L-conv for learning the Lie algebra as well. 
% In fact, the architecture used in \citet{rao1999learning} is a basic version of L-conv with $W^0 = 1$. 
They show that with a small fixed $\ba{\eps^i}$ they could learn learn the single $L_i$ for continuous 1D translations and for 2D rotations.
Indeed, the architecture used in \citet{rao1999learning} is a special case of L-conv with $W^0 = 1$ and $\ba{\eps}^i\in \R$. 
We conducted a more advanced experiment with L-conv learning rotation angles in pairs of random images (SI \ref{ap:exp-L-multi})
}%
% experiments for with $G=SO(2)$.
% In the first test, we used fixed small rotation angle $\pi/10$ and used 
Figure \ref{fig:L-so2-combined} shows the learned $L$. 
Left shows $L\in so(2)$ learned using L-conv in $3$ recursive layers to learn rotation angles between a pair of $7\times 7$ random images $\vf$ and $R(\theta) \vf$ with $\theta \in [0,\pi/3)$. 
Middle and right of Fig. \ref{fig:L-so2-combined} are experiments with fixed small rotation angle $\theta = \pi/10$ (SI \ref{ap:exp-L-small}).
Middle is the $L$ learned using L-conv and right is using the exact solution $R = (YX^T)(X^TX)^{-1}$. 
While the middle $L$ is less noisy, it does not capture weights beyond first neighbors of each pixel. 
% Right shows the generator calculated using the exact solution to the linear regression problem with $\theta = \pi/10$. 
% The SGD solutions using L-conv are less noisy and capture more details.
(also see SI \ref{ap:experiments} for a discussion on symmetry discovery literature.)

L-conv can potentially replace other equivariant layers in an architectures.
We conducted limited experiments for this on small image datasets (SI \ref{ap:exp-image}). 
L-conv allows one to look for potential symmetries in data which may have been scrambled or harbors hidden symmetries. 
% Next, we will provide some examples of explicit forms of L-conv.
Since many datasets such as images deal with discretized spaces, we first need to derive how L-conv acts on such data, discussed next. 
% Next, we discuss the form of L-conv on discrete data. 

\begin{figure}
    \centering
    \includegraphics[width=.17\linewidth]{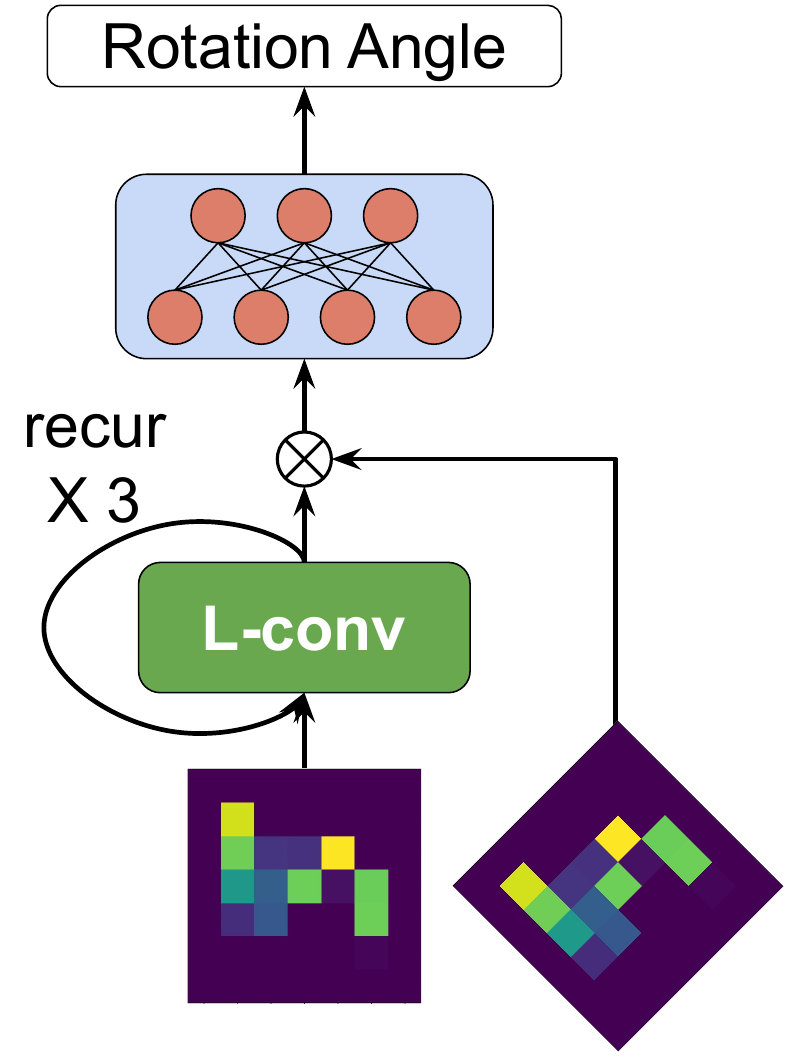}
    \includegraphics[width=.82\linewidth]{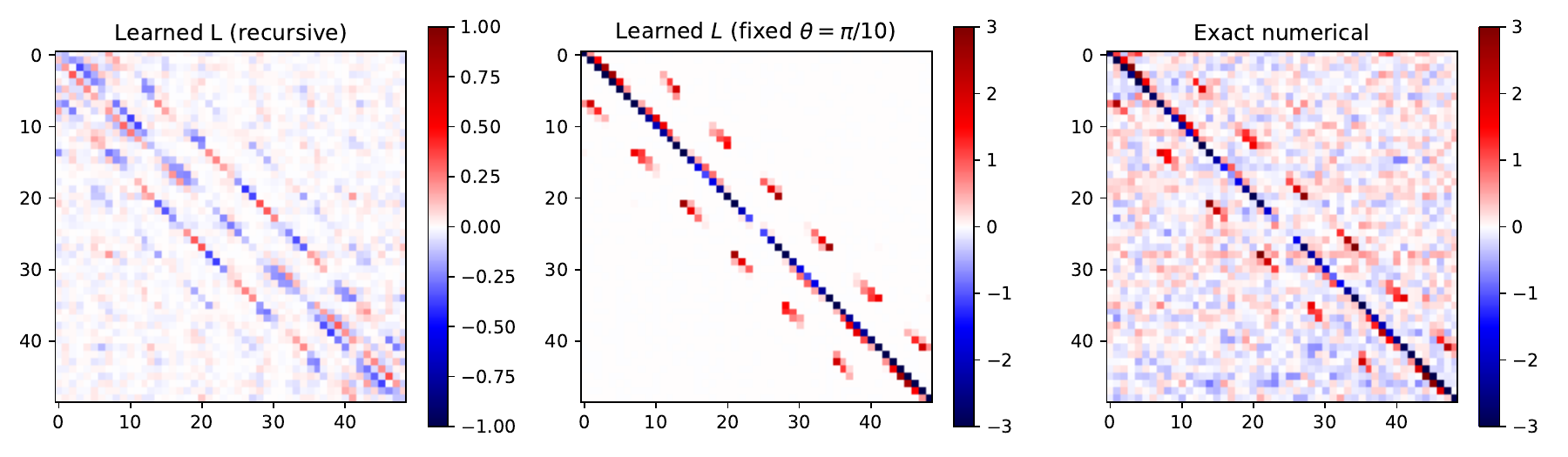}
    \caption{
    \textbf{Learning the infinitesimal generator of $SO(2)$} 
    Left shows the architecture for learning rotation angles between pairs of images (SI \ref{ap:exp-L-multi}). 
    Next to it is the $L$ learned using recursive L-conv in this experiment. %to learn rotation angle between pairs of images.  
    %in $3$ recursive layers to learn rotation angles between a pair of $7\times 7$ random images $\vf$ and $R(\theta) \vf$ with $\theta \in [0,\pi/3)$. 
    Middle $L$ is learned using a fixed small rotation angle $\theta = \pi/10$, and
    % While this $L$ is less noisy, it does not capture weights beyond first neighbors of each pixel. 
    right shows $L$ found using the numeric solution from the data. 
    % The SGD solutions using L-conv are less noisy and capture more details.
    }
    \label{fig:L-so2-combined}
\end{figure}
}%%%%
Next, we discuss implementation details. 
% Since many datasets such as images deal with discretized spaces, we first need to derive how L-conv acts on such data, discussed next. 

\section{Discretized space and implementation: the tensor notation \label{sec:tensor} }
% If the dataset being analyzed is in the form of $f(\vx)$ for some sample of points $\vx$, together with derivatives $\del f(\vx) $, we can use the L-conv formulation above. 
% However, i
\out{
\begin{wrapfigure}{r}{.45\textwidth}
    \vspace{-40pt}
    \includegraphics[width = 1\linewidth, trim=0 0 0 20pt]{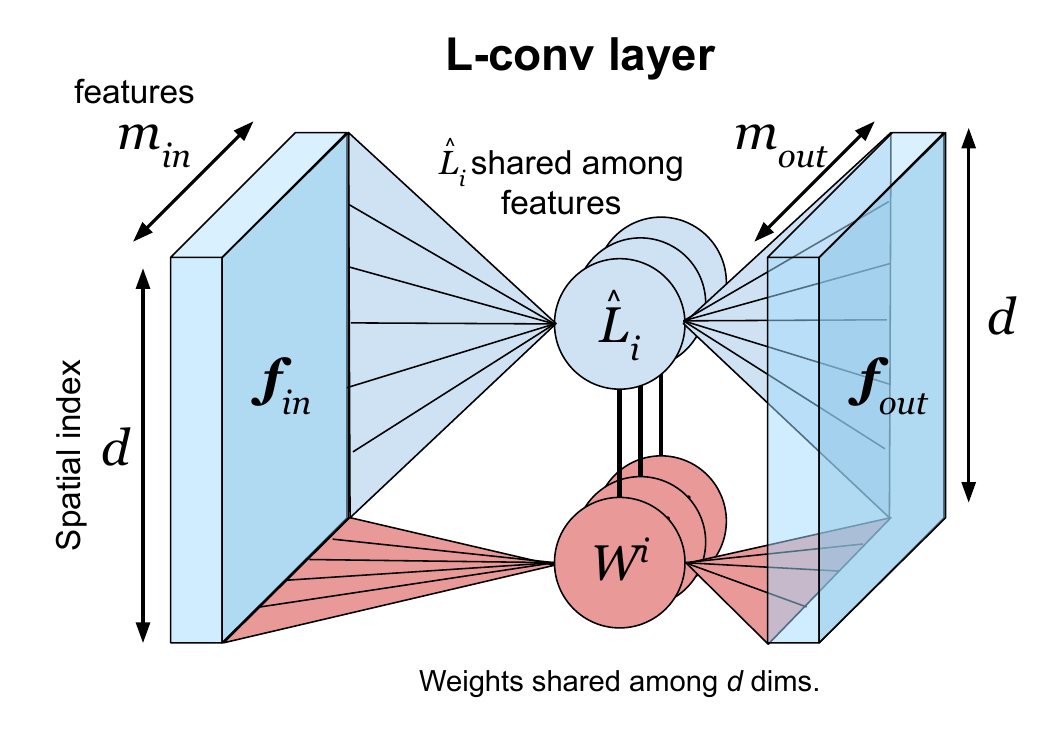}
    % {figs/L-conv-sketch-2021-02-04.pdf}
    %{figs/L-conv-sketch-2020-11-19.pdf}
    \caption{
    {
    L-conv layer architecture.
    $L_i$ only act on the $d$ flattened spatial dimensions, and $W^i$ only act on the $m_{in}$ input features and returns $m_{out}$ output features.
    For each $i$, L-conv is analogous to a Graph Convolutional Network with $d$ nodes and $m_{in}$ features. 
    }
    }
    \label{fig:L-conv-layer}
    \vspace{-20pt}
\end{wrapfigure}
}%%%%%%

In many datasets, such as images, $f(\vx)$ is not given as continuous function, but rather as a discrete array, with %$\vx$ taking values over a grid. 
% In this case all the discussions above still hold, as discuss now. 
% In particular, e
% Let 
$\mathcal{S}= \{\vx_0,\dots \vx_{d-1}\}$ containing $d$ points. 
Each $\vx_\mu$ represents a coordinate in higher dimensional space,
% For instance, 
e.g. on a $10\times 10$ image, $\vx_0$ is $(x,y)=(0,0)$ point and $ \vx_{99}$ is $(x,y) =(9,9)$.

\textbf{Feature maps and group action}
In the tensor notation, we encode $\vx_\mu \in \mathcal{S}$ as the canonical basis (one-hot) vectors in $\vx_\mu \in \R^d$ with 
% To define features $f(\vx_\mu)\in \R^m$ for $\vx_\mu \in \mathcal{S}$, we embed $\vx_\mu \in \R^d$ and 
% encode them as the canonical basis (one-hot) vectors with components 
$[\vx_\mu]_\nu = \delta_{\mu\nu}$ (Kronecker delta), e.g. $\vx_0 = (1,0,\dots, 0)$.
The features become $\vf \in \mathcal{F}= \R^d \otimes \R^m $, meaning 
% feature maps $\vf \in \mathcal{F}$ are 
$d\times m$ tensors, with $f(\vx_\mu) = \vx_\mu^T \vf = \vf_\mu $.
%
% \paragraph{Group action}
% \nd{Change to $\vx_\mu^T g^T$ for less confusion.}
Although $\mathcal{S}$ is discrete, the group acting on $\mathcal{F}$ can be continuous (e.g. image rotations). 
Any $G\subseteq \mathrm{GL}_d(\R)$ of the general linear group (invertible $d\times d$ matrices) acts on $\vx_\mu \in \R^d$ and $\vf\in \mathcal{F}$. 
% Since $\vx_\mu \in \R^d$, $g \in G$ also naturally act on $\vx_\mu$. 
% The resulting 
% Note $\vy = g \vx_\mu \notin  \mathcal{S}$ is a linear combination $\vy = c^\nu \vx_\nu$ of elements in $\mathcal{S}$, not a single element.
% The action of $G$ on $\vf$ and $\vx$, can be defined in multiple equivalent ways.
% The most elegant one is if w
We define $f(g \cdot \vx_\mu) = \vx_\mu^T g^T\vf, \forall g\in G$, so that
% consistent with the lift. 
% This makes the group action on $\vf$ the natural action of $G$, meaning 
for $w\in G$ we have 
\begin{align}
    w\cdot f(\vx_\mu)&= f(w^{-1}\cdot \vx_\mu) = \vx_\mu^T w^{-1T}\vf = %\vx_\mu^T [w^T]^{-1} \vf = 
    [w^{-1}\vx_\mu]^T \vf
    \label{eq:G-action-tensor-T}
\end{align}
Dropping the position $\vx_\mu$, the transformed features are matrix product $w\cdot \vf= w^{-1T}\vf$. 
We can write G-conv in this notation (SI \ref{ap:tensor-long}).
\out{
\paragraph{G-conv and L-conv in tensor notation} 
Writing G-conv \eqref{eq:G-conv} in the tensor notation we have 
\begin{align}
    [\kappa\star f](g\vx_0) = \int_G \kappa(v)f(gv\vx_0) dv = \vx_0^T \int_G v^{T} g^{T} \vf \kappa^T(v) dv 
    \equiv \vx_0^T [\vf \star \kappa](g) 
    \label{eq:G-conv-tensor}
\end{align}
where we moved $\kappa^T(v) \in \R^m \otimes \R^{m'}$ to the right of $\vf$ because it acts as a matrix on the output index of $\vf$.
The equivariance of \eqref{eq:G-conv-tensor} is readily checked with $w\in G$
\begin{align}
    w\cdot [\vf \star \kappa](g) & = [\vf \star \kappa](w^{-1}g) 
    % = \int_G v^{T} [w^{-1}g]^{T} \vf \kappa^T(v) dv %\cr &
    = \vx_0^T \int_G v^{T} g^T w^{-1T} \vf \kappa^T(v) dv %\cr &
    %= \int_G v^{-1} g^{-1} w\vf \kappa^T(v) dv 
    % = [w^{-1T} \vf \star \kappa](g)
    = [(w\cdot\vf) \star \kappa](g)
\end{align}
where we used $[w^{-1}g]^{T}\vf = g^{T} w^{-1T}\vf $.
}%%%%%
Similarly, we can rewrite L-conv \eqref{eq:L-conv} in the tensor notation. 
Defining $v_\eps =I+ \ba{\eps}^i L_i $ 
\begin{align}
    Q[\vf](g) 
    &= W^0 f\pa{g\pa{I+ \ba{\eps}^i L_i}} 
    = \vx_0^T \pa{I+ \ba{\eps}^i L_i}^{T} g^{T} \vf W^{0T} \cr
    % &= \vx_0^T \pa{I+ \ba{\eps}^i L_i^T} g^{T} \vf W^{0T} \cr 
    &= \pa{\vx+ \ba{\eps}^i [gL_i\vx_0]}^T \vf W^{0T}.
    % = \br{\vf_\mu + \ba{\eps}^i [gL_i \vf]_\mu }W^{0T} 
    % \cr
    \label{eq:L-conv-tensor}
\end{align}
Here, $\hat{L}_i=gL_i\vx_0$ is exactly the matrix analogue of pushforward vector field $\hat{L}_i$ in \eqref{eq:dfdg-general}. 
% We will make this analogy more precise below.
The equivariance of L-conv in tensor notation is again evident from the $g^T \vf$, resulting in 
\begin{align}
    Q[w\cdot \vf](g) &= \vx_0^Tv_\eps^T g^T w^{-1T} \vf W^{0T}= Q[\vf](w^{-1}g) = w\cdot Q[\vf](g)
    \label{eq:L-conv-equiv-tensor}
\end{align}

\paragraph{Tensor L-conv  layer implementation}
The discrete space L-conv \eqref{eq:L-conv-tensor} can be rewritten using the global Lie algebra basis $\hat{L}_i$
\begin{align}
    % \mbox{L-conv}:\  
    Q[\vf] &=   \pa{\vf + \hat{L}_i\vf \ba{\eps}^i} W^{0T}, & %\cr 
    Q[\vf]_\mu^a &= \vf_\mu^b [W^{0T}]_b^a + 
    [\hat{L}_i]_\mu^\nu \vf_\nu^c \br{W^i}^a_c
    \label{eq:L-conv-tensor-final}
\end{align}
Where $W^i=W^0\ba{\eps}^i$, $W^0 \in \R^{m_{in}}\otimes \R^{m_{out}}$ and $\ba{\eps}^i \in \R^{m_{in}}\otimes \R^{m_{in}}$ are trainable weights. 
The $\hat{L}_i$ can be either inserted as inductive bias or they can be learned to discover symmetries. % we will be learning $\hat{L}_i$. 
% We will discuss learning symmetries
% We derived L-conv from G-conv using kernels around identity. 

% \out{
% \begin{figure}
%     \centering
\begin{wrapfigure}{r}{.45\textwidth}
    \vspace{-10pt}
    \includegraphics[width = 1\linewidth, trim=0 0 0 20pt]{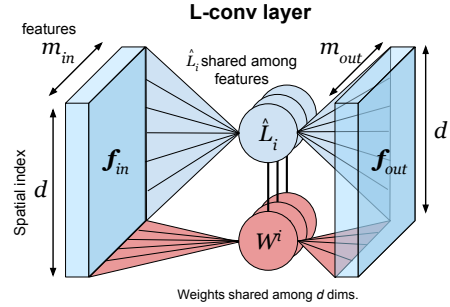}
    % {figs/L-conv-sketch-2021-02-04.pdf}
    %{figs/L-conv-sketch-2020-11-19.pdf}
    \caption{
    {
    L-conv layer architecture.
    $L_i$ only act on the $d$ flattened spatial dimensions, and $W^i$ only act on the $m_{in}$ input features and returns $m_{out}$ output features.
    For each $i$, L-conv is analogous to a GCN %Graph Convolutional Network 
    with $d$ nodes and $m_{in}$ features.
    }
    }
    \label{fig:L-conv-layer}
    \vspace{-20pt}
\end{wrapfigure}
% \end{figure}
% }
To implement L-conv, note that the formula of \eqref{eq:L-conv-tensor-final} is quite similar to a Graph Convolutional Network (GCN) \citep{kipf2016semi}. 
For each $i$, the shared convolutional weights are $\ba{\eps}^i W^{0T}$ and the aggregation function of the GCN, a function of the graph adjacency matrix, is $\hat{L}_i$ in L-conv. 
Thus, L-conv can be implemented as GCN modules for each $\hat{L}_i$, plus a residual connection for the $\vf W^{0T}$ term.

Figure \ref{fig:L-conv-layer} shows the schematic of the L-conv layer. 
In a naive implementation,  $\hat{L}_i$ can be general $d\times d $ matrices. 
However, being vector fields generated by the Lie algebra, $\hat{L}_i$ has a more constrained structure which allows them to be encoded and learned using much fewer parameters than a $d\times d $ matrix. 
Specifically, encoding the topology of $\mathcal{S}$ as a graph (see SI \ref{ap:L-conv-tensor-interpret}), the incidence matrix replaces partial derivatives \citep{schaub2020random} in \eqref{eq:dfdg-general} and the $L_i$ become weighting of the edges.
This weighting is similar to Gauge Equivariant Mesh (GEM) CNN \citep{cohen2019gauge}. 
% The main difference is that in GEM-CNN the basis of the tangent spaces $T_{\vx_\mu}\mathcal{S}$ are not fixed and are chosen by a gauge. 
Indeed, in L-conv the lift $\vx_\mu = g_\mu\vx_0$ fixes the gauge by mapping neighbors of $\vx_0$ to neighbors of $\vx_\mu$. 
% , meaning it determines how the basis at $<0>$ maps to the basis at $<\mu>$. 
Changing how the discrete $\mathcal{S}$ samples an underlying continuous space 
% $\mathcal{S}_0$ 
% is unchanged but the points sampled from it to create $\mathcal{S}$ are moved a bit, 
will change $g_\mu$ and hence the gauge.

\textbf{Choosing the number of $L_i$.} 
Beside the width of $W^0$ and $\ba{\eps}^i$, the number $n_L$ of $L_i$ is a hyperparameter in L-conv. 
For instance, if $\mathcal{S}$ is a discretization of $n$ dimensional space the symmetry group is likely $G\subset \mathrm{GL}_n(\R)\ltimes T_n $, with $n_L \sim O(n^2)$. 
Note that $n_L$ is independent of the size $d$ of the discretized space (e.g. number of pixels) and generally $n^2\ll d$.  
Choosing $n_L$ larger than the true number of $L_i$ only results in an over-complete basis and shouldn't be a problem. 
% \paragraph{Discretizing space vs group} 
% Next, we discuss these constraints. 
%
We conducted small controlled experiments to verify how multilayer L-conv approximates G-conv (SI \ref{ap:experiments}).

\paragraph{Learning symmetries using L-conv.}
\citet{rao1999learning} introduced a basic version of L-conv and showed that it can learn 1D translation and 2D rotation. 
% They showed that t
% The learned 
% % infinitesimal generator 
% $L_i$ for 1D translation reproduced finite translation well using $(I+\eps L)^N$, which is $N$ recursive L-conv layers. 
% Hence, their results prove that L-conv can be used to approximate CNN, as well as $SO(2)$ G-conv. 
% We discuss the details of L-conv approximating CNN below in sec. \ref{sec:approx}.
We  conducted  experiments to learn large rotation angle between two images 
% \nd{revise if experiments moved to SI}
(SI \ref{ap:experiments}), shown in Fig. \ref{fig:L-so2-combined}.
% These experiments required L-conv to approximate rotations by large angles, which it did successfully. 
% The learned $L$ also has the right structure (Fig. \ref{fig:L-so2-combined}).
\out{
% We can use L-conv for learning the Lie algebra as well. 
% In fact, the architecture used in \citet{rao1999learning} is a basic version of L-conv with $W^0 = 1$. 
They show that with a small fixed $\ba{\eps^i}$ they could learn learn the single $L_i$ for continuous 1D translations and for 2D rotations.
Indeed, the architecture used in \citet{rao1999learning} is a special case of L-conv with $W^0 = 1$ and $\ba{\eps}^i\in \R$. 
Fig. \ref{fig:L-so2-combined} shows experiments with L-conv learning rotation angles in pairs of random images (SI \ref{ap:exp-L-multi})
}%
% experiments for with $G=SO(2)$.
% In the first test, we used fixed small rotation angle $\pi/10$ and used 
% Figure \ref{fig:L-so2-combined} shows the learned $L$. 
Left shows the architecture for learning the rotation angles between a pair of $7\times 7$ random images $\vf$ and $R(\theta) \vf$ with $\theta \in [0,\pi/3)$. Second left is the learned $L\in SO(2)$  using $3$ recursive layer L-conv.
Middle is the $L$ learned using L-conv with fixed small rotation angle $\theta = \pi/10$ (SI \ref{ap:exp-L-small}) and right is  the exact solution $R = (YX^T)(X^TX)^{-1}$. 
While the middle $L$ is less noisy, it does not capture weights beyond first neighbors of each pixel. 
% Right shows the generator calculated using the exact solution to the linear regression problem with $\theta = \pi/10$. 
% The SGD solutions using L-conv are less noisy and capture more details.
(also see SI \ref{ap:experiments} for a discussion on symmetry discovery literature.)

L-conv can potentially replace other equivariant layers in a neural network.
We conducted limited experiments for this on small image datasets (SI \ref{ap:exp-image}). 
L-conv allows one to look for potential symmetries in data which may have been scrambled or harbors hidden symmetries. 
% Next, we will provide some examples of explicit forms of L-conv.
% Next, we discuss the form of L-conv on discrete data. 

\begin{figure}
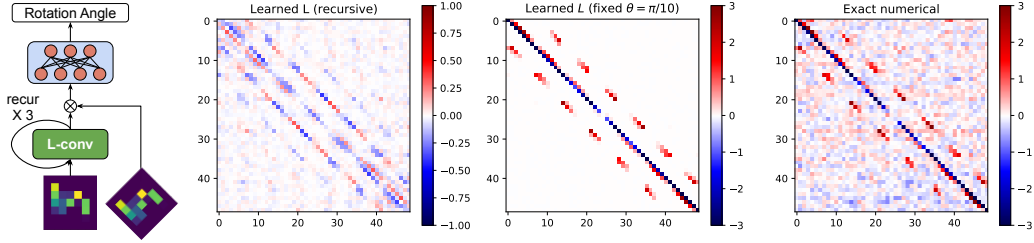

    \centering
    \includegraphics[width=.17\linewidth]{figs2/L-conv-recur-3.pdf}
    \includegraphics[width=.82\linewidth]{figs2/L-so2-combined.pdf}
    \caption{
    \textbf{Learning the infinitesimal generator of $SO(2)$} 
    Left shows the architecture for learning rotation angles between pairs of images (SI \ref{ap:exp-L-multi}). 
    Next to it is the $L$ learned using recursive L-conv in this experiment. %to learn rotation angle between pairs of images.  
    %in $3$ recursive layers to learn rotation angles between a pair of $7\times 7$ random images $\vf$ and $R(\theta) \vf$ with $\theta \in [0,\pi/3)$. 
    Middle $L$ is learned using a fixed small rotation angle $\theta = \pi/10$, and
    % While this $L$ is less noisy, it does not capture weights beyond first neighbors of each pixel. 
    right shows $L$ found using the numeric solution from the data. 
    % The SGD solutions using L-conv are less noisy and capture more details.
    }
    \label{fig:L-so2-combined}
\end{figure}
%%

% Next, we discuss the relation between L-conv and other neural architectures. 

\section{Relation to other architectures \label{sec:approx} }
% \paragraph{Interpretation}
% Let us interpret this equation.
% In addition to the connection with GEM-CNN, we will now show that other important architectures such as CNN and Graph Convolutional Networks (GCN) \citep{kipf2016semi} are either a subset of L-conv or closely related to it.
% L-conv can reproduce some familiar architectures for specific Lie groups, as we discuss now. 

\textbf{CNN.}
% While it is trivial to find permutations taking $\vx_0$ to $\vx_\mu$, most of them won't preserve the topology as \eqref{eq:lift-neigh}. 
% be an element of a given $G$. 
% For instance, 
This is a special case of expressing G-conv as L-conv when the group is continuous 1D translations. 
The arguments here generalize trivially to higher dimensions. 
\citet[sec. 4]{rao1999learning} used the Shannon-Whittaker Interpolation \citep{whitaker1915functions} to define continuous translation on periodic 1D arrays as $\vf'_\rho = g(z)_\rho^\nu \vf_\nu $.
Here $g(z)_\rho^\nu = {1\over d} \sum_{p=-d/2}^{d/2} \cos \pa{{2\pi p\over d}(z+\rho -\nu)} $ approximates the shift operator for continuous $z$.
These $g(z)$ form a 1D translation group $G$ as $g(w)g(z)=g(w+z)$ with $g(0)^\nu_\rho = \delta^\nu_\rho$.
For any $z=\mu \in \Z$,
$g_\mu=g(z=\mu)$ are circulant matrices that shift by $\mu$ as $ [g_\mu]^\rho_\nu=\delta^\rho_{\nu-\mu}$.
Thus, a 1D CNN with kernel size $k$ can be written suing $g_\mu$ as 
\begin{align}
    F(\vf)_{\nu}^a & = \sigma\pa{\sum_{\mu=0}^k  \vf_{\nu-\mu}^c [W^\mu ]^a_{c}  +b^a} = \sigma\pa{\sum_{\mu=0}^k  [g_\mu\vf]_\nu^c [W^\mu ]^a_{c}  +b^a}  
    \label{eq:CNN-1D}
\end{align}
where $W,b$ are the filter weights and biases. 
$g_\mu$ can be approximated using the Lie algebra and written as multi-layer L-conv as in sec. \ref{sec:G-conv2L-conv}. 
Using $g(0)^\nu_\rho \approx \delta(\rho-\nu)$, the single Lie algebra basis $[\hat{L}]_0=\ro_z g(z)|_{z\to0}$, acts as $\hat{L}f(z) \approx -\ro_z f(z)$ (because $\int \ro_z \delta (z-\nu)f(z) =-\ro_\nu f(\nu)$). 
% The Lie algebra has a single basis $[\hat{L}]_0=\ro_z g(z)|_{z\to0}$. 
% Since $g(0)^\nu_\rho \approx \delta(\rho-\nu)$, we have $\hat{L}f(z) \approx -\ro_z f(z)$. 
% , meaning $\hat{L}$ is a Fourier series for $\ro_z$. 
% We have
% \begin{align}
%     [\hat{L}]^\nu_\rho&=\hat{L}(\rho-\nu) = \sum_p {2\pi p \over d^2} \sin \pa{{2\pi p\over d}(\rho -\nu)},&
%     [\hat{L}\vf]_\rho &= \sum_\nu \hat{L}(\rho-\nu)\vf_\nu = [\hat{L}\star \vf]_\nu
% \end{align}
Its components are
$\hat{L}^\nu_\rho={L}(\rho-\nu) = \sum_p {2\pi p \over d^2} \sin \pa{{2\pi p\over d}(\rho -\nu)} $, which are also circulant due to the $(\rho-\nu)$ dependence.
% are approximately 
% ${L}(z)\approx {d\over \pi} \sin(\pi z) \br{z^{-2}+(d-z)^{-2}}$ on $z\in\Z$. 
% \nd{%From sinc form of $g(z)$. 
% What's the exact form?}
Hence, $[\hat{L}\vf]_\rho = \sum_\nu {L}(\rho-\nu)\vf_\nu = [{L}\star \vf]_\nu$ is a convolution. 
\citet{rao1999learning} already showed that this $\hat{L}$ can reproduce finite discrete shifts $g_\mu$ used in CNN.
They used a primitive version of L-conv with $g_\mu = (I+\eps \hat{L})^N$. 
Thus, L-conv can approximate 1D CNN. 
This result generalizes easily to higher dimensions. 
\out{
% $\hat{L}$ is a circulant matrix. 
$\hat{L}^\nu_\rho$ becomes concentrated around small $\rho-\nu$ when $d\gg1$  
% Thus, for large $d$ 
and we can approximate $L$ using $k\ll d$ nearest neighbors of each node.
Here $\mathcal{S}$ is a periodic 1D lattice with 
% $k$ nearest neighbor 
% graph encoding the topology is a circulant matrix
and adjacency matrix
$\mA_{\mu\nu} = \sum_{-k<\nu<k} \delta_{\mu,\mu+i}$ being a circulant matrix.
}%%%%

\textbf{Graph Convolutional Network (GCN).}
Let $\mA$ be the adjacency matrix of a graph. 
% Notice that 
In \eqref{eq:L-conv-tensor-final} if $\hat{L}_i = h(\mA)$, such as $\hat{L}_i =  \mD^{-1/2}\mA\mD^{-1/2}$, we obtain a GCN \citep{kipf2016semi} ($\mD_{\mu\nu} =\delta_{\mu\nu} \sum_\rho \mA_{\mu\rho} $ being the degree matrix). 
% In fact, 
% $[\hat{L}_i]_\mu^\nu=\mA_{\mu\nu}[\hat{L}_i]_\mu^\nu $ is only non-zero where the graph has. 
So in the special case where all neighbors of each node $<\mu>$ have the same edge weight, meaning $[\hat{L}_i]^\nu_\mu = [\hat{L}_i]^\rho_\mu$, $\forall\nu,\rho \in <\mu> $, \eqref{eq:L-conv} is uniformly aggregating over neighbors and L-conv reduces to a GCN. 
Note that this similarity is not just superficial. 
In GCN $h(\mA)= \hat{L}$ is in fact a Lie algebra basis. 
% As discussed above, $\hat{L}_i$ are vector fields on $T\mathcal{S}$.
When 
$\hat{L}= h(\mA)$, the vector field is the flow of isotropic diffusion $d\vf/dt = h(\mA)\vf$ from each node to its neighbors.
This vector field defines one parameter Lie group with elements $g(t) = \exp[h(\mA)t]$. Hence, %the most restricted 
L-conv for flow groups with a single generator are GCN. 
These flow groups include Hamiltonian flows and other linear dynamical systems. The main difference between L-conv and GCN is that L-conv can assign a different weight to each neighbor of the same node, similar to GEM-CNN \citep{cohen2019gauge} with a fixed gauge set by $g_\mu $.
Next, we discuss the mathematical properties of the loss functions for L-conv.

\section{Group invariant loss \label{sec:Loss} }
% \ry{add the example of cycle-consistent loss}
% $G$ being a symmetry of the system means that $f$ and transformed features $g\cdot f$ have similar probability of occurring.
% Hence, 
Loss functions of equivariant networks are rarely discussed. 
Yet, recent work by \citet{kunin2020neural} showed the existence of symmetry directions in the loss landscape. 
To understand how the symmetry generators in L-conv manifest themselves in the loss landscape, we work out the explicit example of a mean square error (MSE) loss. 
Because $G$ is the symmetry group, $f$ and $g\cdot f$ should result in the same optimal parameters. 
Hence, the minima of the loss function need to be \textit{group invariant}.
One way to satisfy this is for the loss itself to be group invariant, which can be constructed 
% group invariant functions is to 
by integrating over $G$ (global pooling \citep{bronstein2021geometric}). 
A function $I  = \int_G dg F(g)$ is $G$-invariant (SI \ref{ap:invariant-loss}). 
We can also change the integration to $\int_\mathcal{S} d^nx $ by change of variable $dg/d\vx$ (see SI \ref{ap:invariant-loss} for discussion on stabilizers). 
% Next, we discuss the mean square loss (MSE) for a single layer L-conv and its relation to physics. 

\paragraph{MSE loss and Field Theory.}
% In \eqref{eq:loss-invariance} $F(g)\in \R$
% is any scalar function of the features $f$ with a defined $G$ action. 
% For example 
% $F(g)$ 
% can be 
% Using \eqref{eq:loss-invariance},
The MSE is given by $I = \sum_n\int _G dg\|Q[f_n](g)\|^2$, where $f_n$ are data samples and $Q[f]$ is L-conv or another $G$-equivariant function. 
In supervised learning the input is a pair $f_n,y_n$. 
$G$ can also act on the labels $y_n$. 
We assme that $y_n$ are either also scalar features $y_n:\mathcal{S}\to \R^{m_y}$ with a group action $g\cdot y_n(\vx)=y_n(g^{-1}\vx)$ (e.g. $f_n$ and $y_n$ are both images), or that $y_n$ are categorical.
In the latter case $g\cdot y_n = y_n$ because the only representations of a continuous $G$ on a discrete set are constant. 
% via some representation $\pi:G\to \mathrm{GL}_c(\R)$.
%have their own $G$ action, 
% including being 
% though in classification $y_n$ is usually $G$-invariant. 
We can concatenate the inputs to 
% \ry{make sure the notation f is consistent} 
$\phi_n \equiv [f_n|y_n]$ % to define a merged feature which has 
with a well-defined $G$ action $g\cdot \phi_n = [g\cdot f_n| g\cdot y_n]$.
% Therefore, in the analysis below we use the combined features $\phi_n$.
% assume the input is a set of $f_n$ and won't distinguish between $f_n$ and $y_n$. 
The collection of combined inputs $\Phi = (\phi_1,\dots, \phi_N)^T$ is an $(m+m_y)\times N$ matrix. 
Using %the combined features $\Phi$, 
equations \ref{eq:L-conv} and  \ref{eq:dfdg-general}, the MSE loss with parameters $W = \{W^0,\ba{\eps}\}$ becomes
% \ry{include a 2d function} 
(SI
\ref{ap:Loss-MSE})
\begin{align}
    I[\Phi;W] &= \int_G dg \L[\Phi;W] = \int_G dg \left\|  W^0 \br{I + \ba{\eps}^i  [\hat{L}_i]^\alpha \ro_\alpha}\Phi(g) \right\|^2 
    \cr
    % &= \int_G dg \br{ 
    % \|W^0\Phi\|^2 + \left\|W^i [\hat{L}_i]^\alpha \ro_\alpha \Phi \right\|^2 +2\Phi^T  W^{0T}W^i [\hat{L}_i]^\alpha \ro_\alpha \Phi}
    % \label{eq:loss-MSE-expand}\\
    % &=\sum_n \int_G dg \br{ M_{ab}\phi_n^a \phi_n^b +\mathbf{h}^{\alpha\beta}_{ab}  \ro_\beta \phi_n^b \ro_\alpha \phi_n^a + [\hat{L}_i]^\alpha \ro_\alpha \pa{\phi_n^T V^i \phi_n}}
    % \label{eq:Loss-MSE-simplified}\\
    &= \int_\mathcal{S} {d^nx\over\left|\ro x\over \ro g\right|} \br{\Phi^T\mathbf{m}_2\Phi + \ro_\alpha \Phi^T \mathbf{h}^{\alpha\beta} \ro_\beta \Phi +[\hat{L}_i]^\alpha \ro_\alpha \pa{\Phi^T \mathbf{v}^i \Phi} }
    \label{eq:Loss-MSE}
    % \mbox{Invariance: } & u\in G: \quad u^T\mathbf{g}_{ab} u = \mathbf{g}_{ab}
 \end{align}
\Eqref{eq:Loss-MSE} generalizes  the free field theories in physics \citep{polyakov2018gauge}. 
% \nd{Discuss stabilizers}
Here $\left|\ro x\over \ro g\right|$ is the determinant of the Jacobian, $W^i = W^0 \ba{\eps}^i$ and
\begin{align}
    % \mathbf{m}_2 &=W^{0T}W^0 &
    % \mathbf{h}^{\alpha\beta}(\vx) &= W^{iT}W^j [\hat{L}_i]^\alpha [\hat{L}_i]^\beta &
    % \mathbf{v}^i &= W^{0T}W^i. \\
    \mathbf{m}_2 &=W^{0T}W^0, &
    \mathbf{h}^{\alpha\beta}(\vx) &= \ba{\eps}^{iT} \mathbf{m}_2 \ba{\eps}^j [\hat{L}_i]^\alpha [\hat{L}_j]^\beta, &
    \mathbf{v}^i &= \mathbf{m}_2\ba{\eps}^{i} . 
    \label{eq:MSE-params}
\end{align}
Note that $\mathbf{h}$ has feature space indices via $[\ba{\eps}^{iT} \mathbf{m}_2 \ba{\eps}^j]_{ab}$, with index symmetry
$\mathbf{h}^{\alpha\beta}_{ab}=\mathbf{h}^{\beta\alpha}_{ba} $.
When $\mathcal{F}= \R$ (i.e. $f$ is a 1D scalar), $\mathbf{h}^{\alpha\beta}$ becomes a
a Riemannian metric 
% $\mathbf{h}^{\alpha\beta}$ 
for $\mathcal{S}$.
In general $\mathbf{h}$ combines a 2-tensor $\mathbf{h}_{ab}=\mathbf{h}_{ab}^{\alpha\beta}\ro_\alpha \ro_\beta \in T\mathcal{S} \otimes T\mathcal{S}$ with an inner product $h^T\mathbf{h}^{\alpha\beta}f $ on the feature space $\mathcal{F}$. 
% Hence $\mathbf{h}\in T\mathcal{S}\otimes T\mathcal{S}\otimes \mathcal{F}^*\otimes \mathcal{F}^*$ is a $(2,2)$-tensor, with $\mathcal{F}^*$ being the dual space of $\mathcal{F}$.   
% \nd{Is it symmetric?}
% Note that $\mathbf{h}^{\alpha\beta}(\vx)$ is a function of $\vx$ because $\hat{L}_i(\vx)$ depends on $\vx$.
%

In field theory, the motivation is to preserve spatial symmetries for the metric $\mathbf{h}$.
In \eqref{eq:Loss-MSE}, 
$\mathbf{h}$ transforms equivariantly as a 2-tensor $v\cdot \mathbf{h}^{\alpha\beta}=[v^{-1}]^\alpha_\rho [v^{-1}]^\beta_\gamma\mathbf{h}^{\rho\gamma}(\vx)$ for $v\in G$ (SI \ref{ap:invariant-loss}).
% In physics the last term in \eqref{eq:Loss-MSE} is absent, which also 
The last term in \eqref{eq:Loss-MSE} vanishes for many groups (SI \ref{ap:invariant-loss}) and 
it is also absent in physics. 
% We can show that for L-conv also when $G$ is translations $T_n$, or rotations $SO(n)$, by partial integration the last term becomes a boundary term and expected to vanish (SI \ref{ap:invariant-loss}).
% For $G=T_n$, the first term in \eqref{eq:Loss-MSE} becomes $\ro_j \phi^T \mathbf{h}^{ji} \ro_i \phi$. 
% We discuss the implications of this relation between physics and equivariant neural networks below. 

% \paragraph{Connections to physics.} 
\out{
Although \eqref{eq:Loss-MSE} was % is the MSE loss for L-conv, 
derived purely using learning and equivariance,
% But 
physicists will recognize \eqref{eq:Loss-MSE} as a generalization of the non-interacting field theory action (integral of a Lagrangian). 
In field theory, $\phi(x)$ (each input feature $[f|y]$) is called a field.
$\mathbf{m}_2$ is the ``mass matrix'' whose eigenvalues are the squares of masses of the fields.
The second term $\ro\phi^T\cdot \mathbf{h}\cdot\ro\phi = \ro^\mu \phi^a \ro_\mu \phi_a$ is the ``kinetic'' term. 
% Mass matrices appear in the standard model of particle physics in the Higgs interaction  
In physics, the loss in \eqref{eq:Loss-MSE} is used in two ways. 
When the parameters $\mathbf{m}_2$ and $\mathbf{h}$ are known, 
% the features $f(\vx)$ minimizing the loss are not known. 
% In this case, the loss lacks the $\sum_n$ and $I [f]$ is a function of $f$, with $M$ and $\mathbf{g}$ fixed.
% Therefore 
a variational procedure known as the Euler-Lagrange equations is used to find optimal $\phi(\vx)= \argmin_\phi I [\phi]$. 
When $\mathbf{m}_2$ and $\mathbf{h}$ are not known, $\phi_n$ are observational data and physicists use \eqref{eq:Loss-MSE} the same way as ML. %, as an ansatz to estimate $\mathbf{m}_2$ and $\mathbf{h}$. 
In this case \eqref{eq:Loss-MSE} is a variational ansatz used to identify parameters $\mathbf{m}_2$ and $\mathbf{h}$ such that the expected loss for the observed $\phi_n$ is minimized.}
\out{
The main difference with physics is that % with the ansatz used in physics is that the metrics on $T\mathcal{S}$ and $\mathcal{F}$ are fused into a tensor $\mathbf{h}$. 
physics assumes a simpler form in which 
% the ansatz used is much simpler and 
$\mathbf{h}$ factorizes into a tensor product $\mathbf{h} = \eta \otimes \delta $, where $\eta_{\alpha\beta}$ is a Riemannian metric on $T\mathcal{S}$ and $\delta$, usually the Kronecker delta, is the metric on $\mathcal{F}$. % (internal space or fibers in physics). 
% The other difference is that in field theory the last term in \eqref{eq:Loss-MSE} is usually absent. 
% We now show via examples that this term can be important for conservation laws in ML. 
% \nd{Talk about how L-conv allows us to create more general ansaetze for physics. }
%%%%

\textbf{Loss minimization in physics and ML}
In physics, the loss in \eqref{eq:Loss-MSE} is used in two ways. 
When the parameters $\mathbf{m}_2$ and $\mathbf{h}$ are known, 
% the features $f(\vx)$ minimizing the loss are not known. 
% In this case, the loss lacks the $\sum_n$ and $I [f]$ is a function of $f$, with $M$ and $\mathbf{g}$ fixed.
% Therefore 
a variational procedure known as the Euler-Lagrange equations is used to find optimal $\phi(\vx)= \argmin_\phi I [\phi]$. 
When $\mathbf{m}_2$ and $\mathbf{h}$ are not known, $\phi_n$ are observational data and physicists use \eqref{eq:Loss-MSE} the same way as ML. %, as an ansatz to estimate $\mathbf{m}_2$ and $\mathbf{h}$. 
In this case \eqref{eq:Loss-MSE} is a variational ansatz used to identify parameters $\mathbf{m}_2$ and $\mathbf{h}$ such that the expected loss for the observed $\phi_n$ is minimized.
% \nd{Is SGD related to RG and Callan-Symantzik? }

% Next, we will discuss how to implement \eqref{eq:L-conv-tensor0} in practice and how to learn symmetries with L-conv.

% \nd{1) physics; 2) coarse-graining}
}%%%
% We conclude with some closing remarks. 
%in the discussions sec. \ref{sec:discussion}. 
% One aspect of the physics theory which may benefit machine learning is that for $\phi_n$ which minimize the loss, the loss function is robust to small perturbations. 
\out{In physics the loss is minimized around real data points $\phi_n$. 
This idea yields a natural way to probe robustness in equivariant neural networks, discussed next. 
}
% \paragraph{Generalization error}
\textbf{Robustness and Euler-Lagrange Equation.}
Equivariant neural networks are  more robust. %, meaning equivariance should improve generalization. 
To check this, we can quantify how the network would perform for an input $\phi'= \phi + \delta \phi$ which adds a small random perturbation $\delta \phi$ to a  data point $\phi$.
Robustness to such perturbation would mean that, for optimal parameters $W^*$ , the loss function would not change,
i.e. $I [\phi';W^*]=I [\phi;W^*]$, requiring $I $ to be minimized around real data points $\phi$.

This can be cast as a variational equation $\delta I [\phi;W^*] = 0$, 
\out{
Writing $I[\phi;W]= \int d^nx \L[\phi;W]$, we have 
\begin{align}
    % I[\phi;W]&= \int d^nx \L[\phi;W], &
    \delta I [\phi;W^*] &=\int_\mathcal{S} d^nx\br{
    {\ro \L \over \ro \phi^a }\delta \phi^a + {\ro \L \over \ro (\ro_\alpha \phi^a) } \ro_\alpha (\delta \phi^a) 
    }
\end{align}
Doing a partial integration on the second term, we get a vanishing boundary term plus the familiar 
}%%%%
which yield the familiar
Euler-Lagrange (EL) equation (SI \ref{ap:generalization}). 
Therefore, for an equivariant network to be robust, i.e. $\delta I [\phi;W^*]/\delta \phi=0$, we would require the data points $\phi$ to satisfy the EL equations for  optimal parameters $W^*$:
% Writing $I[\phi;W]= \int d^nx \L[\phi;W]$, we have 
\begin{align}
    \mbox{Robustness to random noise} \Longleftrightarrow \mbox{EL: }
    \quad 
    {\ro \L \over \ro \phi^b } - \ro_\alpha {\ro \L \over \ro (\ro_\alpha \phi^b) } = 0 
    \label{eq:loss-Euler-Lagrange}
\end{align}
where the partial derivative terms appear because of the L-conv layer.

% Note that \eqref{eq:loss-Euler-Lagrange} only needs to be satisfied for random noise and adversarial attacks may violate it. 
% This could serve as a way to identify such attacks. 

\out{
Applying this to the MSE loss \eqref{eq:Loss-MSE}, \eqref{eq:loss-Euler-Lagrange} becomes
\begin{align}
    % \mbox{EL:} \quad 
    \mathbf{m}_2 \phi - \ro_\alpha \pa{|J| \mathbf{h}^{\alpha\beta} \ro_\beta \phi } - \ro_\alpha\pa{|\ro_\vx g|\mathbf{v}^i [\hat{L}_i]^\alpha }\phi =0 
    \label{eq:loss-EL-MSE}
\end{align}
where % $|\mathcal{J}| = |\ro g/\ro x|$ 
$|\ro_\vx g|$ is the determinant of the Jacobian.
For the translation group, \eqref{eq:loss-EL-MSE} becomes a Helmholtz equation
\begin{align}
    \ba{\eps}^i\mathbf{m}_2 \ba{\eps}^j \ro_i\ro_j\phi = \mathbf{m}_2 \phi
\end{align}
}%%%%
% Another idea which can be transferred from physics to equivariant neural networks is that symmetries produce conserved currents, via Noether's theorem, also discussed in \citet{kunin2020neural}. 
\out{
A major tool in physics for continuous  symmetries is conserved currents, via Noether's theorem, also discussed in \citet{kunin2020neural}. 
These could serve as a new way to learn symmetries. 
We briefly discuss them below. 
}%%%

\textbf{Equivariance and Conservation laws.}
Conserved currents, via Noether's theorem provide a way to find hidden symmetries (see also \citet{kunin2020neural}).
The idea is that the equivariance condition \eqref{eq:equivairance-general} can be written for the integrand of the loss, $\L[\phi,W]$. 
If we write the equivariance equation for infinitesimal $v_\eps $, we obtain a vector field which is divergence free. 
% For $\delta I [\phi;W^*]/\delta \phi=0$, beside the EL equations, the variations at the boundary should be orthogonal to the normal to the boundary $d^{n-1}\Sigma$. 
Since $G$ is the symmetry of the system, transforming an input $\phi\to w\cdot \phi$ by $w\in G$  the integrand should change equivariantly, meaning $\L[w\cdot \phi]= w\cdot \L[\phi] $. 
% We will now show that 
When robustness error is minimized as in \eqref{eq:loss-Euler-Lagrange}, an infinitesimal $w\approx I+\eta^i L_i$, with $\delta \phi = \eps^i \hat{L}_i \phi$, results in a conserved current (SI \ref{ap:generalization})
% the correct $L_i \in \mathfrak{g}$ appear in $I[\phi;W^*]$.
% Let's see how the Lagrangian $\L$ (integrand of the loss) changes when 
% When the generalization error is minimized (i.e. EL equations are satisfied) the term that produced the vanishing boundary term can
\begin{align}
    \mbox{Noether current: } J^\alpha &= {\ro \L \over \ro (\ro_\alpha \phi^b) } \delta \phi^b - {\ro \L \over \ro \vx^\alpha } \delta \vx^\alpha , & \delta I [\phi;W^*]=0 \quad \Rightarrow \ro_\alpha J^\alpha &= 0
\end{align}
% This is the integrand of the boundary term 
% This $J$ is called the Noether current in physics.
The above equation shows that for equivariant networks with a given symmetry, the deviation in data along the symmetry direction ($\hat{L}_i$) yields a  divergence free current $J^\alpha$, known as Noether current.
% It captures the change of the Lagrangian $\L$ along symmetry direction $\hat{L}_i$.
It also provides an alternative means to discover symmetry generators $L_i$ by minimizing $\|\ro_\alpha J^\alpha\|$. 
Note that this   Noether current is the ``stress-energy'' tensor,  associated with space (or space-time) variations $\delta \vx$ \citep{landau2013classical} (SI \ref{ap:noether}). We can  potentially design more general equivariant networks leading to other Noether currents.
% It appears here because $G$ acts on the space, as opposed to acting on feature dimensions. 
% We will discuss the implication of this in the discussions. 
\out{
In physics, there are many other types of conserved charges (including electric charge) which are not the stress- energy tensor because the underlying symmetry is not a spatial symmetry. 
This also reveals that the way equivariance  is currently defined is much less general than how it is used in physics. 
This could lead to defining more general equivariant architectures in machine learning. 
}%%%

% This is the integrand of the boundary term 
% This $J$ is called the Noether current in physics.
% It captures the flow of the Lagrangian $\L$ along symmetry direction $\hat{L}_i$.
% and the requirement for the boundary term to vanish is the statement that the Noether current be conserved (i.e. divergence free) or 
% \nd{Becomes Euler-Lagrange equations}
% It would be interesting to see if these conserved currents can be used in practice as an alternative way for identifying or discovering symmetries. 
% We conclude by remarking on similarity between ML and physics problems. 

\out{
\paragraph{Relation to physics} Although \eqref{eq:Loss-MSE} was % is the MSE loss for L-conv, 
derived purely using ML and equivariance,
% But 
physicists will recognize \eqref{eq:Loss-MSE} as a generalization of the non-interacting field theory action (integral of a Lagrangian). 
In field theory, $\phi(x)$ (each input feature $[f|y]$) is called a field.
$\mathbf{m}_2$ is the ``mass matrix'' whose eigenvalues are the squares of masses of the fields.
The second term $\ro\phi^T\cdot \mathbf{h}\cdot\ro\phi = \ro^\mu \phi^a \ro_\mu \phi_a$ is the ``kinetic'' term. 
% Mass matrices appear in the standard model of particle physics in the Higgs interaction  
In physics, the loss in \eqref{eq:Loss-MSE} is used in two ways. 
When the parameters $\mathbf{m}_2$ and $\mathbf{h}$ are known, 
% the features $f(\vx)$ minimizing the loss are not known. 
% In this case, the loss lacks the $\sum_n$ and $I [f]$ is a function of $f$, with $M$ and $\mathbf{g}$ fixed.
% Therefore 
a variational procedure known as the Euler-Lagrange equations is used to find optimal $\phi(\vx)= \argmin_\phi I [\phi]$. 
When $\mathbf{m}_2$ and $\mathbf{h}$ are not known, $\phi_n$ are observational data and physicists use \eqref{eq:Loss-MSE} the same way as ML. %, as an ansatz to estimate $\mathbf{m}_2$ and $\mathbf{h}$. 
In this case \eqref{eq:Loss-MSE} is a variational ansatz used to identify parameters $\mathbf{m}_2$ and $\mathbf{h}$ such that the expected loss for the observed $\phi_n$ is minimized.
The main difference with physics is that % with the ansatz used in physics is that the metrics on $T\mathcal{S}$ and $\mathcal{F}$ are fused into a tensor $\mathbf{h}$. 
physics assumes a simpler form in which 
% the ansatz used is much simpler and 
$\mathbf{h}$ factorizes into a tensor product $\mathbf{h} = \eta \otimes \delta $, where $\eta_{\alpha\beta}$ is a Riemannian metric on $T\mathcal{S}$ and $\delta$, usually the Kronecker delta, is the metric on $\mathcal{F}$. % (internal space or fibers in physics). 
% The other difference is that in field theory the last term in \eqref{eq:Loss-MSE} is usually absent. 
% We now show via examples that this term can be important for conservation laws in ML. 
% \nd{Talk about how L-conv allows us to create more general ansaetze for physics. }

\out{
\textbf{Loss minimization in physics and ML}
In physics, the loss in \eqref{eq:Loss-MSE} is used in two ways. 
When the parameters $\mathbf{m}_2$ and $\mathbf{h}$ are known, 
% the features $f(\vx)$ minimizing the loss are not known. 
% In this case, the loss lacks the $\sum_n$ and $I [f]$ is a function of $f$, with $M$ and $\mathbf{g}$ fixed.
% Therefore 
a variational procedure known as the Euler-Lagrange equations is used to find optimal $\phi(\vx)= \argmin_\phi I [\phi]$. 
When $\mathbf{m}_2$ and $\mathbf{h}$ are not known, $\phi_n$ are observational data and physicists use \eqref{eq:Loss-MSE} the same way as ML. %, as an ansatz to estimate $\mathbf{m}_2$ and $\mathbf{h}$. 
In this case \eqref{eq:Loss-MSE} is a variational ansatz used to identify parameters $\mathbf{m}_2$ and $\mathbf{h}$ such that the expected loss for the observed $\phi_n$ is minimized.
% \nd{Is SGD related to RG and Callan-Symantzik? }

% Next, we will discuss how to implement \eqref{eq:L-conv-tensor0} in practice and how to learn symmetries with L-conv.

% \nd{1) physics; 2) coarse-graining}
}%%%
% We conclude with some closing remarks. 
}%%%%

% \section{Training L-conv}
% \input{secs/training-L-conv}

% \section{Experiments}
% \input{secs2/experiments}

\section{Conclusion and Discussions \label{sec:discussion}}

We propose the Lie algebra convolutional neural network (L-conv), an infinitesimal version of G-conv.
% In conclusion, we found that expanding group convolutional architectures into a smaller building block, L-conv, has a number of benefits. 
% L-conv is simpler than G-conv.
% We showed that 
L-conv layers do not require encoding irreps or discretizing the group, and can be combined to approximate \textit{any} feedforward equivariant networks on compact groups. 
Additionally, L-conv's universal and simple structure allows us to discover symmetries from data. 
It is easy to implement, with a formula similar to GCN. We validated that L-conv can learn the correct Lie algebra basis in a synthetic experiment.

We discover several intriguing connections between L-conv and physics.
% But perhaps the most intriguing property of L-conv is its connection with symmetry encoding models in physics.
%
% This is significant for two reasons. 
Our derivation shows that equivariant neural networks based on L-conv lead to Noether's theorem and conservation laws. Conversely, we can also optimize Noether current  to discover symmetries. 
% We alluded to this connection, but how to do so in practice requires more research. 
Furthermore, the current equivariance formulation only pertains to ``spatial symmetries'' (i.e. $G$ acts on $\mathcal{S}$). 
In physics, more general ``internal symmetries'' are quite common (e.g. particle physics). We can potentially design more general equivariant networks with  L-conv encoding such symmetries.
\out{
First, methods for dealing with symmetry in physics acquire important meanings in ML, as our discussion on the loss function, robustness and conservation laws show. }%%%%%

Our method also shed lights on  scientific machine learning, especially for physical sciences.
Physicists generally use simple polynomial forms for the Lagrangian, or the loss function. %, similar to the MSE loss above. 
These ``perturbative'' Lagrangian lead to divergences in quantum field theory. However,  it is believed the true Lagrangian is more complicated.  
% In quantum field theory any areas of physics  from this because these ansatze diverge in many cases or fit the data poorly. 
Hence, more expressive L-conv based models can potentially  provide more advanced ansatze for solving scientific problems.
\out{
Also, for machine learning, we discussed how Noether's theorem can provide a new way to discover symmetries. 
We also found that the current equivariance formulation is linked to the conservation of the so-called stress-energy tensor.  
In physics, there are many other types of conserved charges (including electric charge) which are not the stress- energy tensor because the underlying symmetry is not a spatial symmetry. 
This also reveals that the way equivariance  is currently defined is much less general than how it is used in physics. 
This could lead to defining more general equivariant architectures in machine learning.
}

\out{

\paragraph{Relation to physics} \eqref{eq:Loss-MSE} is the MSE loss for L-conv, derived purely using machine learning and equivariance.
But physicists will recognize \eqref{eq:Loss-MSE} as a generalization of the non-interacting field theory action (integral of a Lagrangian). 
In field theory, $\phi(x)$ (each input feature $[f|y]$) is called a field.
$\mathbf{m}_2$ is the ``mass matrix'' whose eigenvalues are the squares of masses of the fields.
The second term $\ro\phi^T\cdot \mathbf{h}\cdot\ro\phi = \ro^\mu \phi^a \ro_\mu \phi_a$ is the ``kinetic'' term. 
% Mass matrices appear in the standard model of particle physics in the Higgs interaction  
The main difference is that % with the ansatz used in physics is that the metrics on $T\mathcal{S}$ and $\mathcal{F}$ are fused into a tensor $\mathbf{h}$. 
physics assumes a simpler form in which 
% the ansatz used is much simpler and 
$\mathbf{h}$ factorizes into a tensor product $\mathbf{h} = \eta \otimes \delta $, where $\eta_{\alpha\beta}$ is a Riemannian metric on $T\mathcal{S}$ and $\delta$, usually the Kronecker delta, is the metric on $\mathcal{F}$. % (internal space or fibers in physics). 
% The other difference is that in field theory the last term in \eqref{eq:Loss-MSE} is usually absent. 
% We now show via examples that this term can be important for conservation laws in ML. 
\nd{Talk about how L-conv allows us to create more general ansaetze for physics. }

\paragraph{Loss minimization in physics and ML}
In physics, the loss in \eqref{eq:Loss-MSE} is used in two ways. 
When the parameters $\mathbf{m}_2$ and $\mathbf{h}$ are known, 
% the features $f(\vx)$ minimizing the loss are not known. 
% In this case, the loss lacks the $\sum_n$ and $I [f]$ is a function of $f$, with $M$ and $\mathbf{g}$ fixed.
% Therefore 
a variational procedure known as the Euler-Lagrange equations is used to find optimal $\phi(\vx)= \argmin_\phi I [\phi]$. 
When $\mathbf{m}_2$ and $\mathbf{h}$ are not known, $\phi_n$ are observational data and physicists use \eqref{eq:Loss-MSE} the same way as ML. %, as an ansatz to estimate $\mathbf{m}_2$ and $\mathbf{h}$. 
In this case \eqref{eq:Loss-MSE} is a variational ansatz used to identify parameters $\mathbf{m}_2$ and $\mathbf{h}$ such that the expected loss for the observed $\phi_n$ is minimized.
\nd{Is SGD related to RG and Callan-Symantzik? }

% As noted, \eqref{eq:Loss-MSE-simplified} is in one way more general than the ansatz used in physics.
% We can absorb the mass matrix by defining $\phi = W^0 f$ yielding the integrand $\phi^T \phi + \ro_\alpha \phi^T\gamma^{\alpha \beta}\ro_\beta $, with a new metric $\gamma^{\alpha\beta}_{ab}=[\hat{L}_i^T\ba{\eps}^i\ba{\eps}^j\hat{L}_j]^{\alpha\beta}_{ab} %= \sum_c [\ba{\eps}^i]_{ac}\hat{L}_i^\alpha [\ba{\eps}^j]_{bc}\hat{L}_j^\beta
% $ 
% In physics, usually the mass matrix $M$ is not related to the metrics $\mathbf{g}$. 

Next, we will discuss how to implement \eqref{eq:L-conv-tensor0} in practice and how to learn symmetries with L-conv.
}%%%

\begin{ack}
% Use unnumbered first level headings for the acknowledgments. All acknowledgments
% More information about this disclosure can be found at: \url{https://neurips.cc/Conferences/2021/PaperInformation/FundingDisclosure}.
R. Walters is supported by a Postdoctoral Fellowship from the Roux Institute and NSF grants \#2107256 and \#2134178.
This work was supported in part by the U. S. Army Research Office under Grant W911NF-20-1-0334, DOE ASCR 2493 and NSF Grant \#2134274.
N. Dehmamy and D. Wang were supported by the Air Force Office of Scientific Research under award number FA9550-19-1-0354.
\end{ack}

\bibliographystyle{icml2021}
\bibliography{icml2021}

%%%%%%%%%%%%%%%%%%%%%%%%%%%%%%%%%%%%%%%%%%%%%%%%%%%%%%%%%%%%
\section*{Checklist}

%%% BEGIN INSTRUCTIONS %%%
The checklist follows the references.  Please
read the checklist guidelines carefully for information on how to answer these
questions.  For each question, change the default \answerTODO{} to \answerYes{},
\answerNo{}, or \answerNA{}.  You are strongly encouraged to include a {\bf
justification to your answer}, either by referencing the appropriate section of
your paper or providing a brief inline description.  For example:
\begin{itemize}
  \item Did you include the license to the code and datasets? \answerYes{See Section}
  \item Did you include the license to the code and datasets? \answerNo{The code and the data are proprietary.}
  \item Did you include the license to the code and datasets? \answerNA{}
\end{itemize}
Please do not modify the questions and only use the provided macros for your
answers.  Note that the Checklist section does not count towards the page
limit.  In your paper, please delete this instructions block and only keep the
Checklist section heading above along with the questions/answers below.
%%% END INSTRUCTIONS %%%

\begin{enumerate}

\item For all authors...
\begin{enumerate}
  \item Do the main claims made in the abstract and introduction accurately reflect the paper's contributions and scope?
    \answerYes{}
  \item Did you describe the limitations of your work?
    \answerNo{}
  \item Did you discuss any potential negative societal impacts of your work?
    \answerNA{}
  \item Have you read the ethics review guidelines and ensured that your paper conforms to them?
    \answerYes{}
\end{enumerate}

\item If you are including theoretical results...
\begin{enumerate}
  \item Did you state the full set of assumptions of all theoretical results?
    \answerYes{}
	\item Did you include complete proofs of all theoretical results?
    \answerYes{}
\end{enumerate}

\item If you ran experiments...
\begin{enumerate}
  \item Did you include the code, data, and instructions needed to reproduce the main experimental results (either in the supplemental material or as a URL)?
    \answerYes{}
  \item Did you specify all the training details (e.g., data splits, hyperparameters, how they were chosen)?
    \answerYes{}
	\item Did you report error bars (e.g., with respect to the random seed after running experiments multiple times)?
    \answerYes{}
	\item Did you include the total amount of compute and the type of resources used (e.g., type of GPUs, internal cluster, or cloud provider)?
    \answerYes{}
\end{enumerate}

\item If you are using existing assets (e.g., code, data, models) or curating/releasing new assets...
\begin{enumerate}
  \item If your work uses existing assets, did you cite the creators?
    \answerYes{}
  \item Did you mention the license of the assets?
    \answerNA{}
  \item Did you include any new assets either in the supplemental material or as a URL?
    \answerYes{}
  \item Did you discuss whether and how consent was obtained from people whose data you're using/curating?
    \answerNA{}
  \item Did you discuss whether the data you are using/curating contains personally identifiable information or offensive content?
    \answerNA{}
\end{enumerate}

\item If you used crowdsourcing or conducted research with human subjects...
\begin{enumerate}
  \item Did you include the full text of instructions given to participants and screenshots, if applicable?
    \answerNA{}
  \item Did you describe any potential participant risks, with links to Institutional Review Board (IRB) approvals, if applicable?
    \answerNA{}
  \item Did you include the estimated hourly wage paid to participants and the total amount spent on participant compensation?
    \answerNA{}
\end{enumerate}

\end{enumerate}

%%%%%%%%%%%%%%%%%%%%%%%%%%%%%%%%%%%%%%%%%%%%%%%%%%%%%%%%%%%%

\appendix

~\newpage

{\Large\bf
Supplementary Information
}

\section{Extended derivations and proofs
\label{ap:theory-extended}
}

\paragraph{Path-ordered Exponential %\label{ap:path-ordered} 
}
Every element $g\in G$ can be written as a product $g=\prod_a \exp[t_a^i L_i]$ using the matrix exponential \citep{hall2015lie}. 
This can be done using a path $\gamma$ connecting $I$ to $g$ on the manifold of $G$.
Here, $t_a$ will be segments of the path $\gamma$ which add up as vectors to connect $I$ to $g$.  
This surjective map can be written as a ``path-ordered'' (or time-ordered in physics \citep{weinberg1995quantum}) exponential (POE). 
In the simplest form, POE can be defined by breaking $u=\prod_a \exp[t_a^i L_i]$ down into 
% $u = P\exp[\int_\gamma ds t^i(s) L_i]\equiv P\exp[t^i L_i]$ . 
% The path-ordered exponential can be written as the product of 
infinitesimal steps of size $t_a = 1/N$ with $N\to \infty$. 
Choosing $\gamma $ to be a differentiable path, we can replace the sum over segments $\sum_a t_a$ with an integral along the path $\sum_a t_a = \int_\gamma ds t(s) ds$, where $t(s) = d\gamma/ds$ is the tangent vector to the path $\gamma$, where $s\in [0,1]$ parametrizes $\gamma$
The POE is then defined as the infinitesimal $t_a$ limit of $g=\prod_a \exp[t_a^i L_i]$. 
This can be written as
\begin{align}
    g &= P\exp\br{ \int_\gamma t^i(s) L_i ds} = \lim_{N\to \infty} \prod_{a=1}^N \pa{I+ \delta s {\gamma'}^i(s_a) L_i } \cr
    &= \int_0^{s_1}ds_0  {\gamma'}^i(s_0) L_i \int_0^{s_2}ds_1  {\gamma'}^j(s_1) L_j \dots
    \int_0^{1}ds_N  {\gamma'}^k(s_N) L_k
\end{align}

\out{
On connected $G$ we can Taylor expand $\kappa(v)$ around identity. 
\begin{align}
    \kappa(v) &= \kappa\pa{ P\exp\br{\int_\gamma dt^iL_i} } = \kappa(I) + Pt^iL_i \kappa'(I) + P{(t^iL_i)^2\over 2} \kappa''(I) + \dots \cr
    &=P \sum_n {(t^iL_i)^n\over n!} {d^n\kappa(g)\over dg^n} \bigg|_{g\to I}\cr
    &=P \sum_{n=0}^\infty {1\over n!} \prod_{k=1}^n \int_{} dt^iL_i {d^n\kappa(g)\over dg^n} \bigg|_{g\to I}
\end{align}
\nd{
We want the expansion to slowly chip away at $\kappa$, moving it toward identity. 
Starting from G-conv, we cn either expand $f$ or $\kappa$. 
If we expand $\kappa$ the derivatives will be on $\kappa$. 
If we expand the $v$ in $f(gv)$, can we express the $\kappa(v)$ in terms of product of integrals of small $v_\eps$? Can we do that from the beginning? 
What happens if we break the integral down into these smaller pieces? 
Can we prove that it will remain just one integral, rather than the product of many similar integrals, which would be spurious? 
If we think about it as a propagator, then the intermediate integrals will not result in extra integrals. 
}

We can use $u=u_\eps v$ to expand $\kappa(v)$ as
\begin{align}
    \kappa(v) &= \kappa\pa{(I+\eps^iL_i) u} \approx \kappa(u) + \eps^iL_i u {d \kappa(u)\over du} + O(\eps^2) 
    \label{eq:kappa-expand}
\end{align}
While $\kappa(v)$ can be any function on $G$ (with small restrictions discussed in \citet{kondor2018generalization}), we can use \eqref{eq:kappa-expand} to replace $\kappa$ with a new kernel $\kappa_1$ which has support on most of $G$, except for a small neighborhood $\eta$ when $\|\eps\| < \eta$. 
In other words $\kappa(v) \approx \kappa_1(u) + \eps^iL_i u \kappa_1'(u) $. 
in \eqref{eq:G-conv} and rewrite G-conv as follows
\begin{align}
    [\kappa \star f](g) &\approx \int_{G/H_\eta} dv \br{\kappa(v) }
\end{align}
where $H_\eta$ is a subgroup spanned by

For continuous symmetry groups, using the Lie algebra %in the G-conv architecture 
allows us to approximate G-conv without having to integrate over the full group manifold, thus alleviating the need to discretize or sample the group. 
% This leads to a simple formulation for Lie algebra convolution akin to graph convolutional networks \citep{kipf2016semi}.

The following proposition establishes the connection between G-conv and Lie algebras.

\begin{proposition}
    Let $G$ be a Lie group, $f:\mathcal{S}\to \mathcal{F}$ a differentiable equivariant function. % and 
    %$\rho: G \to \mathcal{F}$ a representation of $G$ on space $\mathcal{F}$. 
    If a convolution kernel $\kappa: G \to \mathrm{Hom}(\mathcal{F}, \mathcal{F}') $ has support only on an infinitesimal neighborhood $\eta$ of identity, a G-conv layer of \eqref{eq:G-conv} can be written in terms of the Lie algebra. 
% \nd{either add surjective or assert}
\end{proposition}
\begin{proof}
% This is easily seen by
Consider \eqref{eq:G-conv} 
% Equivariance of $f$ yields $f(v^{-1}u) = \rho[v]\cdot f(u)$.
with $\pi: G \to \mathrm{GL}(\mathbb{R}[\mathcal{S}])$ a representation of $G$. 
Linearization over $\eta$ yields $\pi[v] \approx I + \eps^i L_i$ with $L_i \in \mathrm{Hom}(\mathbb{R}[\mathcal{S}],\mathbb{R}[\mathcal{S}])$ being a representation for the basis of the Lie algebra of $G$. 
% The inverse $u^{-1} \approx I - \eps^i L_i + O(\eps^2)$. 
Since $\kappa $ %in \eqref{eq:G-conv} 
has support only in an $\eta$ neighborhood of identity, fixing a basis $L_i$, we can reparametrize $\tilde{\kappa}(\eps)\equiv \kappa(I+\eps^i L_i)$ as a function over the Lie algebra $\tilde{\kappa }:\mathfrak{g}\to \mathrm{Hom}(\mathcal{F},\mathcal{F}') $. 
The Haar measure is also replaced by a volume element in the tangent space $d\mu(u) \to d\eps$. 
\eqref{eq:G-conv} becomes
% Denoting $g_0 = g(I) = \tilde{g}(0)$ and its 
\out{
\begin{align}
    % g(I+\eps \cdot L)& = \tilde{g}(\eps) \approx g_0 + \eps \cdot g'_0, \qquad  g'_0 = \left. \del_\eps \tilde{g}(\eps)\right|_{\eps \to 0}\cr %\left.{\ro g(u)\over du} \right|_{u\to I} \cr
     %= g_0 + \eps \cdot g'_0\cr 
    &(\kappa \star f)(\vx)  \approx \int_G %{|\eps| <\eta }  
    \tilde{\kappa}(\eps) (I+\eps\cdot L^\rho)f\pa{(I-\eps\cdot L^\pi) \vx} d \eps \cr&
    \approx %\pa{\ba{\eps}_0 I - \ba{\eps}\cdot L }
    \left.
    \pa{W^0 I + W\cdot\left[L^\rho - L^\pi \vx\cdot \del_z\right]} f(z)\right|_{z\to \vx} 
    \label{eq:G-conv-expand} \\
    % &\left.
    % \pa{W^0 I - \sum_i W^i L_i \vx\cdot \del_z} f(z)\right|_{z\to \vx} 
    % \label{eq:G-conv-expand} \\
    &W^0 \equiv \int_G \tilde{\kappa}(\eps) d\eps,  \quad
    % \ba{\eps}_i
    W^i\equiv  \int_G \eps^i \tilde{\kappa}(\eps) d\eps. %\nonumber
    \label{eq:G-conv-gbar} %\cr 
    % c &= Vol(n_L-1)\eta^{n_L-1}
\end{align}
where $L_i^\rho$ and $L_i^\pi$ are the Lie algebra basis in the $\rho$ and $\pi$ representations, respectively. 
}%%%
\begin{align}
    % g(I+\eps \cdot L)& = \tilde{g}(\eps) \approx g_0 + \eps \cdot g'_0, \qquad  g'_0 = \left. \del_\eps \tilde{g}(\eps)\right|_{\eps \to 0}\cr %\left.{\ro g(u)\over du} \right|_{u\to I} \cr
     %= g_0 + \eps \cdot g'_0\cr 
    &(\kappa \star f)(\vx)  \approx \int_G %{|\eps| <\eta }  
    \tilde{\kappa}(\eps) %(I+\eps\cdot L^\rho)
    f\pa{(I-\eps\cdot L) \vx} d \eps \cr&
    \approx %\pa{\ba{\eps}_0 I - \ba{\eps}\cdot L }
    \left.
    \pa{W^0 I - W\cdot L \vx\cdot \del_z} f(z)\right|_{z\to \vx} 
    \label{eq:G-conv-expand} \\
    % &\left.
    % \pa{W^0 I - \sum_i W^i L_i \vx\cdot \del_z} f(z)\right|_{z\to \vx} 
    % \label{eq:G-conv-expand} \\
    &W^0 \equiv \int_G \tilde{\kappa}(\eps) d\eps,  \quad
    % \ba{\eps}_i
    W^i\equiv  \int_G \eps^i \tilde{\kappa}(\eps) d\eps. %\nonumber
    \label{eq:G-conv-gbar} %\cr 
    % c &= Vol(n_L-1)\eta^{n_L-1}
\end{align}
where $L_i$ are the Lie algebra basis. 
% where $L_i^\rho$ and $L_i^\pi$ are the Lie algebra basis in the $\rho$ and $\pi$ representations, respectively. 

\end{proof}

}

\paragraph{L-conv derivation}
Let us consider what happens if the kernel in G-conv \eqref{eq:G-conv} is localized near identity. 
Let $\kappa_I (u) = c \delta_\eta (u)$, with constants $c\in \R^{m'} \otimes \R^{m}$ and kernel $\delta_\eta(u) \in \R$ which has support only on on an $\eta$ neighborhood of identity, meaning $\delta_\eta(I+\eps^iL_i) \to 0 $ if $|\eps|>\eta$.
This allows us to expand G-conv in the Lie algebra of $G$ to linear order. 
With $v_\eps = I+\eps^i L_i $, we have
% $\delta_\eta \star f$ may be expanded  near identity as
\begin{align}
    [\delta_\eta \star f](g) &= \int_G dv \delta_\eta (v) f(gv )
    =\int_{\|\eps\| <\eta } dv_\eps  \delta_\eta (v_\eps) f(gv_\eps )\cr
    &= \int d\eps \delta_\eta (I+\eps^i L_i) f(g(I+\eps^i L_i))\cr
    &= \int d\eps \delta_\eta (I+\eps^i L_i) f(g+\eps^i gL_i)\cr
    &= \int d\eps \delta_\eta (I+\eps^i L_i) \br{f(g)+ \eps^i g L_i \cdot {d\over dg}  f(g) + O(\eps^2) } %\bigg|_{u\to g}
    \cr
    &= \int d\eps \delta_\eta (I+\eps^i L_i) \br{I+ \eps^i g L_i \cdot {d\over dg} } f(g) + O(\eta^2) \cr
    &= W^0\br{I + \ba{\eps}^i g L_i\cdot {d\over dg} } f(g)  + O(\eta^2)
    \label{eq:L-conv-basic} 
\end{align}
where $d\eps $ is the integration measure on the Lie algebra induced by the Haar measure $dv$ on $G$.
The $O(\eta^2)$ term arises from integrating the $O(\eps^2)$ terms.
To see this, note that for an order $p$ function 
$\phi(\epsilon)$, for $|\epsilon| < \eta$, $|\phi(\epsilon)| < \eta^p C$ (for some constant $C$).
Substituting this bound into the integral over the kernel $\delta_\eta$ we get

\begin{align}
    &\left|\int d\epsilon \delta_\eta (I+\epsilon^i L_i) \phi(\epsilon)\right| \leq \int \left| d\epsilon \delta_\eta (I+\epsilon^i L_i) \phi(\epsilon)\right| \cr 
    &< \int \left| d\epsilon \delta_\eta (I+\epsilon^i L_i)\eta^p C\right| \leq \eta^p C \int \left| d\epsilon \delta_\eta (I+\epsilon^i L_i)\right| \leq \eta^p C
    \label{eq:O-eta-p}
\end{align}

% where we used the fact that the Haar measure $dv$ induces an integration measure on the Lie algebra, which we denoted by $d\eps$.
In matrix representations, $g L_i\cdot {df\over dg} = [g L_i]_\alpha^\beta {df\over dg_\alpha^\beta} = \Tr{[g L_i]^T {df\over dg} }$.
% $g L_i\cdot {df\over dg} = \sum_{a,b}[g L_i]_{ab} {df\over dg_{ab}} = \Tr{[g L_i]^T {df\over dg} }$.
Note that in $g(I+\eps^iL_i)\vx_0$, the $gL_i\vx_0 = \hat{L}_i{(g)}\vx$ come from the pushforward $\hat{L}_i{(g)}=gL_ig^{-1} \in T_gG$.
%of $L_i\in T_IG$ from the tangent space at $I$ to the tangent space at $g$, meaning $gL_ig^{-1} \in T_gG$. 
% This is seen from the fact that $gv_\eps$ is moving from $g$ by $\eps L_i$, hence in $T_gG$. 
% We will show this explicitly in following examples. 
Here 
\begin{align}
    W^0 &= %\int_G dv \kappa_0(v) = 
    c \int d\eps \delta_\eta (I+\eps^iL_i) \in \R^{m'} \otimes \R^m, &
    \ba{\eps}^i &= \frac{\int d\eps \delta_\eta (I+\eps^iL_i) \eps^i}{\int d\eps \delta_\eta (I+\eps^iL_i)} \in \R^{m} \otimes \R^m
\end{align}
with $\|\ba{\eps}\|<\eta $.
When $\delta_\eta$ is normalized, meaning $\int_G \delta_\eta(g)dg=1 $, we have $W^0 = c$ and 
\begin{align}
    \ba{\eps}^i &= \int d\eps \delta_\eta (I+\eps^iL_i) \eps^i \nonumber
\end{align}
Note that with $f(g)\in \R^m$, each $\eps^i \in\R^{m} \otimes \R^m $ is a matrix. 
With indices, $f(gv_\eps) $ is given by 
\begin{align}
    [f(gv_\eps)]^a &= \sum_b f^b(g(\delta^a_b + [\eps^i]^a_b L_i)) 
\end{align}
Similarly, the integration measure $d\eps$, which is induced by the Haar measure $dv_\eps \equiv d\mu(v_\eps)$, is a product $\int d\eps = \int |J| \prod d [\eps^i]^a_b$, with $J=\ro v_\eps /\ro \eps $ being the Jacobian. 

\Eqref{eq:L-conv-basic} is the core of the architecture we are proposing, the Lie algebra convolution or \textbf{L-conv}. 

\textbf{L-conv Layer} 
In general, we define Lie algebra convolution (L-conv) as follows \begin{align}
    Q[f](g)%&= \br{W^0 + W^i g L_i\cdot {d\over dg} } f(g) 
    &= W^0 \br{I + \ba{\eps}^i g L_i\cdot {d\over dg} }f(g) \cr 
    &= [W^0]_b f^a\pa{g\pa{\delta^b_a+ [\ba{\eps}^i]_a^b L_i}} +O(\ba{\eps}^2)
    \label{eq:L-conv-def1}
\end{align}
% where for $f(g)\in \R^m $, $W^0 $ is $m'\times m$ and $\ba{\eps}^i = \tilde{W}_0^{-1}W^i $ is $m\times m$, with eigenvectors $\|\ba{\eps}\|<\eta $, and found using the Moore-Penrose inverse $\tilde{W}_0^{-1} = (W^{0T}W^0)^{-1}W^0$.
\paragraph{Extended equivariance for L-conv}
From \eqref{eq:L-conv-def1} we see that $W^0$ acts on the output feature indices. 
Notice that the equivariance of L-conv is due to the way $gv_\eps = g(I+\ba{\eps}^iL_i)$ appears in the argument, since for $u\in G$ 
\begin{align}
    u\cdot Q[f](g) = W^0 f(u^{-1}gv_\eps) = W^0 [u\cdot f](gv_\eps)
\end{align}
Because of this, replacing $W^0$ with a general neural network which acts on the feature indices separately will not affect equivariance. 
For instance, if we pass L-conv through a neural network to obtain a generalized L-conv $Q_\sigma$, we have 
\begin{align}
    Q_\sigma[f](g) &= \sigma( W f(gv_\eps) + b) \cr 
    u \cdot Q_\sigma[f](g) &= Q_\sigma[f](u^{-1}g)= \sigma( W f(u^{-1}gv_\eps) + b)\cr 
    &= \sigma( W [u\cdot f](gv_\eps) + b) = Q_\sigma[u\cdot f](g)
\end{align}
Thus, L-conv can be followed by any nonlinear neural network as long as it only acts on the feature indices (i.e. $a$ in $f^a(g)$) and not on the spatial indices $g$ in $f(g)$. 

\out{
\subsection{Approximating G-conv using L-conv \label{ap:G-conv2L-conv} }
We now show that G-conv \eqref{eq:G-conv} can be approximated by composing L-conv layers.

\textbf{Universal approximation for kernels} 
Using the same argument used for neural networks \citep{hornik1989multilayer,cybenko1989approximation}, we may approximate any kernel $\kappa(v)$ as the sum of a number of kernels $\kappa_k$ with support only on a small $\eta$ neighborhood of $u_k \in G $ to arbitrary accuracy. 
The local kernels can be written as $\kappa_k (v) = c_k \delta_\eta ( u_k^{-1}v)$, with $\delta_\eta(u)$ as in \eqref{eq:L-conv-basic}  
% where $\delta_\eta(u) \in \R$ has support near identity ($\delta_\eta(I+\eps^iL_i) \to 0 $ if $|\eps|>\eta$) 
and constants $c_k\in \R^{m'} \otimes \R^{m}$. 
Using this, G-conv \eqref{eq:G-conv} becomes
\begin{align}
    [\kappa \star f](g)& = \sum_k [\kappa_k \star f](g) = \sum_k c_k \int dv \delta_\eta(u_k^{-1} v) f(gv) \cr
    &= \sum_k c_k \int dv \delta_\eta(v ) f(gu_k v) = \sum_k c_k [\delta_\eta \star f] (gu_k).
\end{align}
As we showed in \eqref{eq:L-conv-basic}, $[\delta_\eta \star f](g)$ is the definition of L-conv. 
Next, we need to show that $[\delta_\eta \star f] (gu_k)$ can also be approximated with $[\delta_\eta \star f] (g)$ and hence L-conv. 
For this we use $u_k =v_k (I+\eps_k^iL_i) $ to find $ v_k \in G$ which are closer to $I$ than $u_k$.
Taylor expanding $F_\eta = \delta_\eta\star f$ in $\eps$ we obtain
\begin{align}
    F_\eta (gu_k)& = F_\eta \pa{gv_k (I+\eps_k^iL_i) } = F_\eta (gv_k) + \eps_k^i u L_i \cdot {dF_\eta (u)\over du}\bigg|_{u\to gv_k} + O(\eps^2) 
    % \label{eq:L-conv-uk}
    \cr
    [\kappa \star f](g)& = \sum_k c_k F_\eta (gu_k) = \sum_k \br{c_k + c_k \eps_k^i u L_i \cdot {d\over du} }F_\eta (u) \bigg|_{u\to gv_k} %+ O(\eps^2)
    \cr
    &=\sum_k \br{W_k^0 + W_k^i u L_i\cdot {d\over du} }F_\eta (u) \bigg|_{u\to gv_k}
    =\sum_k Q_k[F_\eta ](gv_k) 
    \label{eq:L-conv-uk}
\end{align} 
% \eqref{eq:L-conv-uk} 
%Here $Q[f](g) = f(g(I+\eps^iL_i))$ 
% $Q_k[F_\eta](gu_k)$
Using \eqref{eq:L-conv-uk} we can progressively remove the $u_k$ as $F_\eta(gu_k) \approx Q^n_k[\dots [Q^1_k[F_\eta]]](g)$, i.e. an $n$ layer L-conv.  
Thus, we conclude that any G-conv \eqref{eq:G-conv} can be approximated by multilayer L-conv. 
}%%%%
\subsection{Approximating G-conv using L-conv \label{ap:G-conv2L-conv} }
\out{
{\bf\thref{thm:Gconv2Lconv} }(G-conv from L-convs). 
{\it 
% \begin{theorem}[G-conv from L-convs]
    % \label{thm:Gconv2Lconv}
    G-conv \eqref{eq:G-conv} can be approximated using L-conv layers.
    }
% \end{theorem}

\begin{proof}
    The procedure involves two steps, as illustrated in Fig. \ref{fig:G-conv2L-conv}: 1) approximate the kernel using localized kernels as the $\delta_\eta$ in L-conv; 2) move the kernels towards identity using multiple L-conv layers. 
    The following lemma outline the details.
\end{proof}
}%%%%

% To do so, we first note that any function $\kappa(v)$ can be approximated to arbitrary accuracy using a set of local kernels $\delta_\eta $ (using the universal approximation theorem of neural networks \citep{hornik1989multilayer,cybenko1989approximation}).
{ \bf%\thref{lem:kernel-approx} 
Lemma 1 
}(Approximating the kernel).
{\it 
% \begin{lemma}[Approximating the kernel] 
    % \nd{Robin, please check:}
    Let the kernel $\kappa:G\to \mathcal{F}'\otimes \mathcal{F}$ with $\int_G \|\kappa(g)\|^2 dg < \infty$ be continuously differentiable with $\|d\kappa(g)/dg\|^2<\xi^2 $, and with compact support over $G_0\subset G$.  
    Let $\kappa_k(g)=c_k\delta_\eta(u_k^{-1}g) $ be a set of kernels with support on an $\eta$ neighborhood of $u_k\in G$.  
    Then, $\exists c_k \in  \mathcal{F}'\otimes \mathcal{F}, u_k \in G$ such that $\tilde{\kappa}$ approximates $\kappa$, meaning
    $\int_G \|\kappa(g)-\tilde{\kappa}(g)\|^2dg < \zeta^2 $ for arbitrary small $\zeta\in \R_+$.
    }%%%
% \end{lemma}
% \textbf{Approximating the kernel}

\begin{proof}
The intuition is similar to the universal approximation  theorem for neural networks \citep{hornik1989multilayer,cybenko1989approximation}, only generalized to  a group manifold instead of $\R$. 
Let $B_0 $ be the set of $v_\eps = I+\eps^iL_i\in \mathfrak{g}$, with $\|\eps\|^2<\eta^2$.
Choose a finite set of $u_k\in G$ such that the neighborhoods $B_k = u_k B_0 \subset G$, 
% Consider a set of non-overlapping neighborhoods $B_k \subset G_0$ of size $\eta \ll 1$, meaning $\int_{B_k}dg =\eta$, 
and such that
$\bigcup_k B_k = G_0$.
\out{
$\bigcup_k B_k = G_0/\phi$ where $\phi$ is a measure zero set (i.e. $B_k$ make a dense covering of the support of $\kappa$) \nd{precise?}.
}%%
We can show that, for small enough $\eta$, the kernel does not change more than $\xi$ over each $B_k$, allowing us to replace it with a constant localized kernel $\kappa_k$ with support only on $B_k$. 
To see this, consider $g\in B_k$ and $v_\eps= I+\eps^iL_i \in\mathfrak{g} $, such that $v_\eps g \in B_k$.
Let $\gamma$ be a path connecting $g$ to $v_\eps g$, via $u(s) = (I+s\eps^iL_i)g$, with $s\in [0,1]$. 
Using the triangle inequality we have 
\begin{align}
    \|\kappa(g)-\kappa(v_\eps g)\|^2& = \norm{\int_\gamma ds {d\kappa(u(s))\over ds} }^2 \leq
    \int_\gamma ds\norm{ {d\kappa(u(s))\over ds} }^2 \cr &
    < \int_0^1 ds \norm{dg\over ds}^2 \xi^2 = \|\eps\|^2 \xi^2 
\end{align}
where $\|\eps\|^2 \leq \eta^2$ because $v_\eps g\in B_k$. 
This means that if we replace $\kappa(g)$  with $\kappa(u_k)$ for any $u_k\in B_k$, our error in approximating $\kappa$ over $B_k$ is less than $\eta^2\xi^2 $. 
Setting $\kappa_k(g) =\kappa(u_k) \delta_\eta(u_k^{-1}g)$ with the localized kernels   $\delta_\eta(g)$ being any continuously differentiable distribution such that $\int_{B_k} dg \delta_\eta(u_k^{-1}g)=1$ on all $B_k$, we get 
% $=1/|B_k|$ if $g\in B_k$ and $0$ otherwise, we get 
\begin{align}
    \int_G dg \|\kappa(g)-\tilde{\kappa}(g)\|^2 = \sum_k \int_{B_k} dg \|\kappa(g)-\kappa_k(g)\|^2 < \sum_{k} |B_k|  \eta^2 \xi^2 = |G_0|\eta^2 \xi^2
\end{align}
where $|G_0|$ is the volume of the support of $\kappa$
% where $|G_0|$ is the size of the support $G_0\subset G$ of $\kappa$.
\out{
This simply follows from the universal approximation theorem for neural networks \citep{hornik1989multilayer,cybenko1989approximation} and the same approximation errors apply. 
$\delta_\eta(g)$ can be written as the sum of two sigmoid functions 
$\delta_\eta(g) =\mathrm{sig}(v_\eta g)+\mathrm{sig}(v_\eta^{-1}g) $. 
}%%%%
\end{proof}

Thus, we can approximate a large class of kernels as $\kappa(g) \approx \sum_k \kappa_k(g)$ where the local kernels $\kappa_k (g) = c_k \delta_\eta (  u_k^{-1}g)$ have
support only on an $\eta$ neighborhood of $u_k \in G $. 
Here $c_k\in \R^{m'} \otimes \R^{m}$ are constants and $\delta_\eta(u)$ is as in \eqref{eq:L-conv}.  
Using this, G-conv \eqref{eq:G-conv} becomes
\begin{align}
    [\kappa \star f](g)
    % & = \sum_k [\kappa_k \star f](g) 
    &= \sum_k c_k \int dv \delta_\eta(u_k^{-1} v) f(gv) 
    % &= \sum_k c_k \int dv \delta_\eta(v ) f(gu_k v) 
    = \sum_k c_k [\delta_\eta \star f] (gu_k).
    \label{eq:c-k-delta0}
\end{align}
% As we showed in \eqref{eq:L-conv-basic}, $[\delta_\eta \star f](g)$ is the definition of L-conv. 
% \ry{put this as a theorem, show as an approximation error}
The kernels $\kappa_k$ are localized around $u_k$, whereas in L-conv the kernel is around identity. 
We can compose L-conv layers to move $\kappa_k$ from $u_k$ to identity. 

\out{
Can we write the movement of a $\delta$ on the group explicitly as the action of the exp map on a $\delta$ near identity? 
For a group element it is clear, but for a function over the group? 
Should we use the fact that we have convolution, i.e. integral over the group? 
With $v = P\exp[\int_\gamma dt^i L_i]$ we have 
\begin{align}
    f(g) &= \int_G du f(u)\delta(u^{-1}g) \cr 
    f(vg) &= \int_G du f(u)\delta(u^{-1}vg) = \int_G du' f(v^{-1}u')\delta({u'}^{-1}g) 
\end{align}
Can we use the integral 
}

\textbf{Lemma 2 }(Moving kernels to identity).
{\it
% \begin{lemma}[Moving kernels to identity]
    The local kernel $\kappa_k$ can be moved near identity using a multilayer L-conv.  
% \end{lemma}
}%%%%

\begin{proof}
In \eqref{eq:c-k-delta0}, write $u_k = v_\eps u'_k$, with $v_\eps = I+ \eps^i L_i \in \mathfrak{g}$. 
We have  
\begin{align}
    \delta_\eta(u_k^{-1} v) &= \left.\br{I+ \eps^i g L_i \cdot {d\over dg} }\delta_\eta(g)\right|_{g\to {u'}_k^{-1} v} + O(\eps^2) \approx Q_\eps[\delta_\eta]({u'}_k^{-1} v)
\end{align}
where $Q_{\eps} = I-\eps^i \hat{L}_i$ is an L-conv layer with $W^0=I$. 
This means that applying one L-conv layer with the parameters above moves the localized kernel along $v_\eps $ on $G$.   
Iterating this further, write $u_k$ as the product of a set of small group elements $u_k = \prod_{a=1}^p v_a$, with $v_a = I+ \eps_a^i L_i \in \mathfrak{g}$. 
Defining L-conv layers $Q_a = I-\eps_a^i \hat{L}_i $,
we can write 
\begin{align}
    \kappa_k (g) &\approx c_k Q_p \circ \cdots \circ Q_1 \circ \delta_\eta (g)  
\end{align}
meaning $\kappa_k$ localized around $u_k$ can be written as a $p$ layer L-conv acting on a kernel  $\delta_\eta(g)$, localized around the identity of the group. 
\out{
Thus, by applying multiple L-conv layers we can find $Q_p \circ \dots Q_1\circ \kappa_k\approx c_k\delta_\eta $. 
This can always be done using a set of $v_{\eps_a} = (I+\ba{\eps}_a^iL_i)$ such that $u_k \approx \prod_{a=1}^p v_{\eps_a}$. 
}%%%%
With $\|\eps_a\|<\eta$, the error in $u_k$ is $O(\eta^{p+1})$.
\end{proof}

Thus, we conclude that any G-conv \eqref{eq:G-conv} can be approximated by multilayer L-conv. 
We can take this result even further, following the main theorem in \citet{kondor2018generalization}, and show that any feedforward equivariant neural network can be approximated using multilayer L-conv with nonlinearities.

\textbf{Equivariance of nonlinearity}
Pointwise nonlinearities give equivariant maps between scalar feature maps. %\nd{ref instead?} 
To see this, 
let $\sigma:\R \to \R $.  We extend $\sigma: \gF \to \gF$ by applying $\sigma$ component-wise. Let $f:\mathcal{S}\to \mathcal{F}$ be a scalar feature map (i.e., $g\cdot f(\vx) = f(g^{-1}\vx)$).  Then 
\[
g \cdot (\sigma \circ (f))(\vx) = 
\sigma \circ (f)( g^{-1} \vx) = 
\sigma \circ (g \cdot f)(\vx).
\]
Since the composition of equivariant maps is equivariant, given equivariant linear mapping $Q:\mathcal{F}^\gS\to \mathcal{F}^{\prime\gS}$ (i.e. $g\cdot Q[f]= Q[g\cdot f]$), the layer $f \mapsto \sigma \circ Q[f]$ is equivariant.
% Then, $\sigma(Q[f](\vx))$ is still a scalar and for $g\in G$ we have 
% \begin{align}
%     % g\cdot \sigma\circ f^a(\vx)& = \sigma\circ f^a(g^{-1}\vx) = \sigma\circ g\cdot f^a(\vx) \cr
%     g\cdot \sigma\circ Q[f](\vx)& = \sigma\circ Q[f](g^{-1}\vx) = \sigma\circ g\cdot Q[f](\vx) 
% \end{align}
% meaning, point-wise nonlinearity $\sigma$ is equivariant for scalar $Q[f]$. 
Hence we have the corollary:

\begin{corollary}%[Feedforward NN]
    %Multilayer $G$-conv with pointwise non-linearity can be approximated by multilayer $L$-conv with pointwise non-linearity.
    Assume $G$ is compact and acts on $\gS$ transitively.  Then any equivariant feedforward neural network (FNN) can be approximated using multilayer L-conv with point-wise nonlinearities.  Without the compactness and transitivity hypothesis, multilayer L-conv with pointwise non-linearities can approximate multilayer G-conv with pointwise non-linearities.  
\end{corollary}
\begin{proof}
%This follows from Theorem 1 in  \citet{kondor2018generalization} and the our Theorem \ref{thm:Gconv2Lconv}. 
A FNN is defined as $\sigma_p \circ F_p[\cdots [ \sigma_1\circ  F_1  [f]](\vx)$ where $F_k$ are linear and $\sigma_k$ are point-wise nonlinearities.
By Theorem 1 of \citet{kondor2018generalization}, any linear equivariant layer is a G-conv, which by Theorem \ref{thm:Gconv2Lconv} can be approximated by multilayer L-conv.
Therefore, multilayer L-conv with nonlinearity can approximate any equivariant FNN.
\end{proof}

% As a final note, we should point out 
Finally, note 
that to our knowledge it is not known whether \textit{every equivariant function} can be approximated by equivariant FNN for a Lie group $G$. 
Hence, the corollary above is \textit{not} a universal approximation theorem for equivariant scalar functions in terms of L-conv.  However, it does show that multilayer $L$-conv is equally expressive as other equivariant networks. 

\out{
We conducted small controlled experiments to verify how multilayer L-conv approximates G-conv (SI \ref{ap:experiments}). 
We briefly discuss them here. 

\paragraph{Learning symmetries using L-conv}
\citet{rao1999learning} introduced a basic version of L-conv and showed that it can learn 1D translation and 2D rotation. 
% They showed that t
The learned 
% infinitesimal generator 
$L_i$ for 1D translation reproduced finite translation well using $(I+\eps L)^N$, which is $N$ recursive L-conv layers. 
Hence, their results prove that L-conv can be used to approximate CNN, as well as $SO(2)$ G-conv. 
We discuss the details of L-conv approximating CNN below in sec. \ref{sec:approx}.
We also conducted more complex experiments using recursive L-conv to learn large rotation angle between two images 
% \nd{revise if experiments moved to SI}
(SI \ref{ap:experiments}).
% These experiments required L-conv to approximate rotations by large angles, which it did successfully. 
% The learned $L$ also has the right structure (Fig. \ref{fig:L-so2-combined}).
\out{
% We can use L-conv for learning the Lie algebra as well. 
% In fact, the architecture used in \citet{rao1999learning} is a basic version of L-conv with $W^0 = 1$. 
They show that with a small fixed $\ba{\eps^i}$ they could learn learn the single $L_i$ for continuous 1D translations and for 2D rotations.
Indeed, the architecture used in \citet{rao1999learning} is a special case of L-conv with $W^0 = 1$ and $\ba{\eps}^i\in \R$. 
We conducted a more advanced experiment with L-conv learning rotation angles in pairs of random images (SI \ref{ap:exp-L-multi})
}%
% experiments for with $G=SO(2)$.
% In the first test, we used fixed small rotation angle $\pi/10$ and used 
Figure \ref{fig:L-so2-combined} shows the learned $L$. 
Left shows $L\in so(2)$ learned using L-conv in $3$ recursive layers to learn rotation angles between a pair of $7\times 7$ random images $\vf$ and $R(\theta) \vf$ with $\theta \in [0,\pi/3)$. 
Middle and right of Fig. \ref{fig:L-so2-combined} are experiments with fixed small rotation angle $\theta = \pi/10$ (SI \ref{ap:exp-L-small}).
Middle is the $L$ learned using L-conv and right is using the exact solution $R = (YX^T)(X^TX)^{-1}$. 
While the middle $L$ is less noisy, it does not capture weights beyond first neighbors of each pixel. 
% Right shows the generator calculated using the exact solution to the linear regression problem with $\theta = \pi/10$. 
% The SGD solutions using L-conv are less noisy and capture more details.
(also see SI \ref{ap:experiments} for a discussion on symmetry discovery literature.)

L-conv can potentially replace other equivariant layers in an architectures.
We conducted limited experiments for this on small image datasets (SI \ref{ap:exp-image}). 
L-conv allows one to look for potential symmetries in data which may have been scrambled or harbors hidden symmetries. 
% Next, we will provide some examples of explicit forms of L-conv.
Since many datasets such as images deal with discretized spaces, we first need to derive how L-conv acts on such data, discussed next. 
% Next, we discuss the form of L-conv on discrete data. 

\begin{figure}
    \centering
    \includegraphics[width=.17\linewidth]{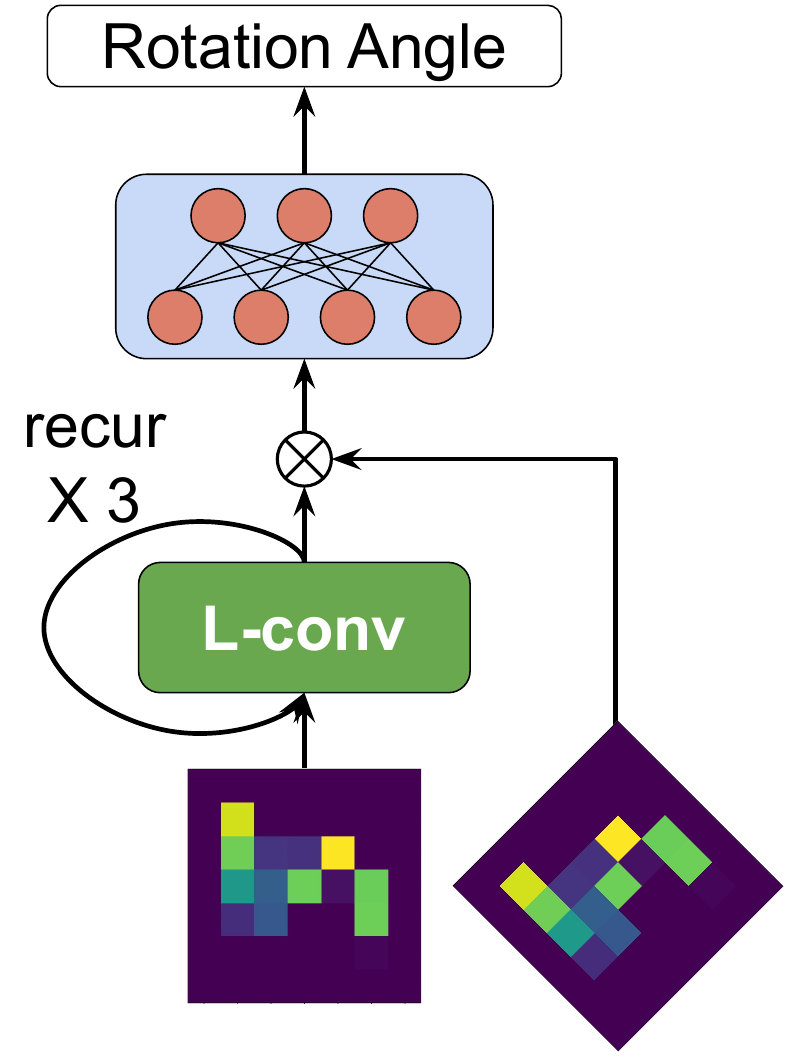}
    \includegraphics[width=.82\linewidth]{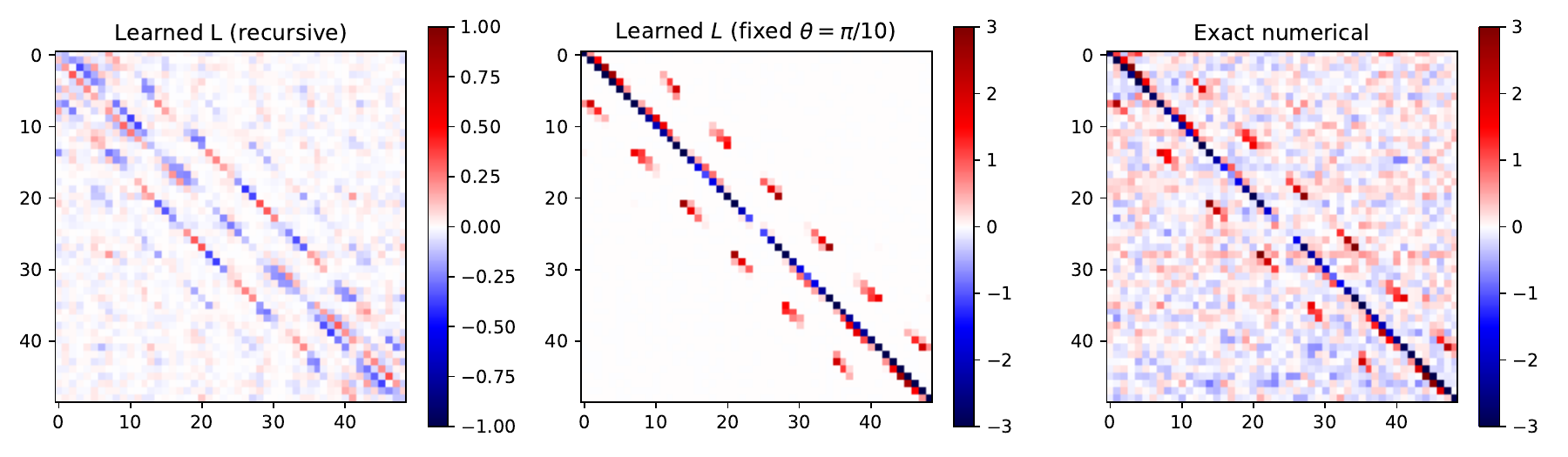}
    \caption{
    \textbf{Learning the infinitesimal generator of $SO(2)$} 
    Left shows the architecture for learning rotation angles between pairs of images (SI \ref{ap:exp-L-multi}). 
    Next to it is the $L$ learned using recursive L-conv in this experiment. %to learn rotation angle between pairs of images.  
    %in $3$ recursive layers to learn rotation angles between a pair of $7\times 7$ random images $\vf$ and $R(\theta) \vf$ with $\theta \in [0,\pi/3)$. 
    Middle $L$ is learned using a fixed small rotation angle $\theta = \pi/10$, and
    % While this $L$ is less noisy, it does not capture weights beyond first neighbors of each pixel. 
    right shows $L$ found using the numeric solution from the data. 
    % The SGD solutions using L-conv are less noisy and capture more details.
    }
    \label{fig:L-so2-combined}
\end{figure}
}%%%%

Since many datasets such as images deal with discretized spaces, we first need to derive how L-conv acts on such data, discussed next.

\subsection{Example of continuous L-conv \label{ap:examples-continuous}}

% \subsection{Interpretation and Examples \label{sec:interpret-continuous} }
% \paragraph{General case} 
The $gL_i\cdot df/dg$ in \eqref{eq:L-conv} can be written in terms of partial derivatives $\ro_\alpha f(\vx) = \ro f/\ro \vx^\alpha$. 
In general, using $\vx^\rho = g^\rho_\sigma \vx_0^\sigma$, we have
% ${df(g\vx_0)\over dg^\alpha_\beta} =\vx_0^\beta \ro_\alpha f(\vx) $, and so
% \ry{add a cartoon}
\begin{align}
    % \vx^\mu &= g^\rho_\sigma \vx_0^\sigma &
    {df(g\vx_0)\over dg^\alpha_\beta}
    &= {d(g^\rho_\sigma \vx_0^\sigma) \over dg^\alpha_\beta}\ro_\rho f(\vx) = \vx_0^\beta \ro_\alpha f(\vx)
    \label{eq:dgx0-dg}
    \\
    gL_i\cdot {df\over dg} &= [gL_i]^\alpha_\beta \vx_0^\beta \ro_\alpha f(\vx) = [gL_i\vx_0]\cdot \del f(\vx) = \hat{L}_if(\vx)
    \label{eq:dfdg-general0}
\end{align}
Hence, for each $L_i$, the pushforward $gL_i$ generates a flow on $\mathcal{S}$ through the vector field $\hat{L}_i\equiv gL_i\cdot d/dg = [gL_i\vx_0]^\alpha \ro_\alpha$ (Fig. \ref{fig:Lie-group-S}). 
% $\hat{L}_i$ is related to the Maurer-Cartan form $\omega = g^{-1} \ro_\alpha g d\vx^\alpha$, which encodes the pushforward.
% Write $f(gv_\eps) = f(g+ \eta(\eps)^\alpha \ro_\alpha g)$. 
% We have $ \eta^\alpha \ro_\alpha g = \eps^i g L_i$, so $ \eps^i L_i \eta^\alpha g^{-1} \ro_\alpha g = \eta \cdot \omega $.
% The Maurer-Cartan form defines a way to parallel transport (pushforward) and define a basis frame on  
% Vector fields are sections of the tangent bundle $T\mathcal{S}$.
% We will encounter $gL_i\vx_0$ again on a discretized $T\mathcal{S}$ again below. 
Being a vector field $\hat{L}_i\in T\mathcal{S}$ (i.e. 1-tensor), $\hat{L}_i$ is basis independent, meaning 
for $v \in G$, $\hat{L}_i(v\vx) = \hat{L}_i$. 
Its components transform as $[\hat{L}_i(v \vx)]^\alpha =[vgL_i\vx_0]^\alpha = v^\alpha_\beta \hat{L}_i(\vx)^\beta $, while the partial transforms as $\ro / \ro[v\vx]^\alpha = [v^{-1}]^{\gamma}_\alpha \ro_\gamma$. 
% \out
{
Using this relation and Taylor expanding \eqref{eq:L-conv-equiv}, we obtain a second form for the group action on L-conv.
For $w\in G$, with $\vy= w^{-1}\vx$ we have
\begin{align}
    Q[f](w^{-1}g\vx_0)&= W^0\br{I+ 
    \ba{\eps}^i[\hat{L}_i]^\alpha[w^{-1}]^\beta_\alpha  {\ro \over \ro \vy^\beta}}f(\vy)\big|_{\vy\to w^{-1}\vx}
    \label{eq:L-conv-hat-L-equiv}
\end{align}
}%%%
% We will derive a similar equation for the equivariance of L-conv in tensor notation in sec. \ref{sec:tensor}.

\textbf{1D Translation:}
Let $G=T_1 = (\R,+)$. 
A matrix representation for $G$ is found by encoding $x$ as a a 2D vector $(x,1)$. %, adding a dummy dimension. 
The lift is given by $\vx_0 =(0,1) $ as the origin and $g = \begin{pmatrix}1&x\\0&1\end{pmatrix}$. 
The Lie algebra basis is $L = \begin{pmatrix}0&1\\0&0\end{pmatrix}$. 
It is easy to check that $gg'\vx_0 = (x+x',1)$. 
We also find $gL = L $, meaning $L$ looks the same in all $T_gG$. 
Close to identity $I=0$, $v_\eps = I+\eps L = \eps $. 
We have $gv_\eps \vx_0 = (g+ \eps gL)x_0 = (x + \eps,1)$.
Thus, $f(g(I+\eps L)\vx_0) \approx f(\vx)+ \eps df(\vx)/d \vx $. 
This readily generalizes to $n$D translations $T_n$ (SI \ref{ap:example-Tn}), yielding $f(\vx) + \eps^\alpha \ro_\alpha f(\vx)$. 

\paragraph{2D Rotation:} Let $G=SO(2)$. 
The space which $SO(2)$ can lift is not the full $\R^2$, but a circle of fixed radius $r=\sqrt{x^2+y^2}$.
Hence we choose $\mathcal{S}=S^1$ embedded in $\R^2$, with $x=r\cos\theta$ and $y=r\sin\theta$. 
For the lift, we use the standard 2D representation. 
We have $\vx_0 = (r,0)$ and (see SI \ref{ap:example-so2})
\begin{align}
    L&=\mat{0&-1\\1&0},&
    g&=\exp[\theta L]= {1\over r}\mat{x &-y \\ y & x},&
    % gL& = \begin{pmatrix}- y & - x\\x & - y\end{pmatrix} & 
    gL \cdot {df\over dg} &= 
    \pa{x \ro_y 
    - y \ro_x } f.
    \label{eq:L-so20}
\end{align}
Physicists will recognize $\hat{L} \equiv \pa{x \ro_y - y \ro_x } = \ro_\theta$ as the angular momentum operator in quantum mechanics and field theories, which generates rotations around the $z$ axis. 

\paragraph{Rotation and scaling}
Let $G= SO(2)\times \R^+$, where the $\R^+=[0,\infty)$ is scaling. 
The infinitesimal generator for scaling is identity $L_2 = I$. 
This group is also Abelian, meaning $[L_2,L]=0$ ($L\in so(2)$ \eqref{eq:L-so20}). 
$\R^2/{0}$ can be lifted to $G$ by choosing $\vx_0 = (1,0)$ in polar coordinates and $\vx = g\vx_0 = rL_2 \exp[\theta L]\vx_0$. %, using $L$ in \eqref{eq:L-so2}.
We again have $gL\cdot df/dg = \ro_\theta f $. 
We also have $gL_2 = g$, so $gL_2\vx_0 = (x,y)$ and from \eqref{eq:dfdg-general0}, $gL_2 \cdot df/dg = (x\ro_x+y\ro_y)f = r\ro_r f $, which is the scaling operation.

\subsubsection{Rotation $SO(2)$ \label{ap:example-so2} }
With  $\vx_0 = (r,0)$
\begin{align}
    % x_0 &= \begin{pmatrix}r\\0\end{pmatrix} & 
    g&
    =\begin{pmatrix}\cos \theta &-\sin\theta \\\sin\theta &\cos\theta \end{pmatrix}, 
    ={1\over r}\begin{pmatrix}x &-y \\ y & x \end{pmatrix}, & L&=\begin{pmatrix}0&-1\\1&0\end{pmatrix} & gL& = {1\over r}\begin{pmatrix}- y & - x\\x & - y\end{pmatrix} \\
    gL\vx_0 &= \mat{-y\\x} = \mat{-\sin\theta \\ \cos\theta} 
    \label{eq:gLx0-so2}
\end{align}
To calculate $df/dg$ we note that even after the lift, the function $f$ was defined on $\mathcal{S}$. 
So we must include the $\vx_0$ in $f(g\vx_0)$. 
Using \eqref{eq:dfdg-general}, we have
\begin{align}
    {df(g\vx_0)\over dg} &= {1\over r} \vx_0^T 
    \mat{\ro_xf&-\ro_yf\\ \ro_yf&\ro_xf } = \mat{\ro_xf&-\ro_yf\\ 0&0}\cr
gL \cdot {df\over dg} &= \mathrm{Tr}
    \begin{pmatrix}x \ro_y f
    - y \ro_x f
    & x \ro_x f
    + y \ro_y f 
    \\ 0&0
    % \\- x \ro_x f
    % - y \ro_y f
    % & x \ro_y f
    % - y \ro_x f
    \end{pmatrix} = 
    \pa{x \ro_y f
    - y \ro_x f}
\end{align}

\subsubsection{Translations $T_n$
\label{ap:example-Tn} }
Generalizing the $T_1$ case, we add a dummy dimension $0$ and $\vx_0=(1,0,\dots,0)$. 
The generators are $[L_i]_\mu^\nu = \delta_{i\mu}\delta_0^\nu$ and $g = I+ x^i L_i$. 
Again, $gL_i = L_i + x^j L_j L_i = L_i$ as $L_jL_i = 0$ for all $i,j$. 
Hence, $[\hat{L}_i]^\alpha = [gL_i\vx_0]^\alpha =\delta^\alpha_i $. 

\out{
\subsection{Haar measure}
\nd{Take out}
To calculate $\int_G dg$ we need to calculate the Haar measure $dg = d\mu(g) $. 
For $G= \mathrm{GL}_d(\R)$, the measure is 
\begin{align}
    dg = {1\over det(g)^d} \bigwedge_{i,j} dg_{ij}.
\end{align}
which has $d^2$ different $dg_{ij}$
This can be seen in terms of the generators of $\mathrm{GL}_d(\R)$, which are all one-hot matrices $\mE_{ij}$. 
For subgroups $G\subset \mathrm{GL}_d(\R)$ we can find independent components by writing 
$g^{-1}dg =dt^i g^{-1}L_ig $, which is Maurer-Cartan form.  
When the Lie algebra basis has $n=|\mathfrak{g}|$ elements, we have $n$ independent variables $t^i$. 

\subsubsection{$SO(2)\times \R_+$ Haar measure}
When the group is Abelian, meaning $[L_i,L_j]=0$, the Haar measure computation simplifies. 
For $SO(2)\times \R_+$ with $g=r\exp[\theta L_\theta]$ we have
\begin{align}
    g &= r \mat{\cos\theta & -\sin \theta \\
    \sin\theta & \cos\theta },\qquad  g^{-1} = {1\over r} \mat{\cos\theta & \sin \theta \\
    -\sin\theta & \cos\theta }\cr
    g^{-1}dg &= g^{-1}\br{{1\over r}dr g+r d\theta \mat{-\sin\theta & -\cos\theta \\ \cos\theta & -\sin \theta } } = {dr\over r} I + d\theta L_\theta.   
\end{align}
To build the volume form from this we can start from the Euclidean volume form $dx\wedge dy$, which is simply a different basis, and express in terms of $d\theta\wedge dr$. 
The Jacobian for this change of variables is 
\begin{align}
    J= \mat{{\ro \theta\over \ro x} & {\ro \theta\over \ro y}\\
    {\ro r\over \ro x} & {\ro r\over \ro y}} &= \mat{{-1\over  y} & {1\over  x}\\ {x\over  r} & {y\over  r}},& |\det(J)| = {2\over r} 
\end{align}
And so we recover the polar integration measure $d^2 x =|\det(J)|^{-1} dr d\theta = r dr d\theta $

}%%%%

\subsection{Group invariant loss \label{ap:invariant-loss} }
\out{
$G$ being a symmetry of the system means that $f$ and transformed features $g\cdot f$ have similar probability of occurring. 
Hence, the loss function needs to be \textit{group invariant} \nd{cite?}. 
One way to construct group invariant functions is to integrate over $G$ (global pooling \citep{bronstein2021geometric}). 
A function $ I = \int_G dg F(g)$ is $G$-invariant as for $w\in G$   
\begin{align}
    w\cdot  I &=\int_G w\cdot F(g)dg = \int F(w^{-1}g) dg = \int_G F(g') d(w g') = \int_G F(g') dg'
    \label{eq:loss-invariance}
\end{align}
where we used the invariance of the Haar measure $d(wg') = dg'$.
}%%%

% \ry{add the example of cycle-consistent loss}
% $G$ being a symmetry of the system means that $f$ and transformed features $g\cdot f$ have similar probability of occurring.
% Hence, 
Because $G$ is the symmetry group, $f$ and $g\cdot f$ should result in the same optimal parameters. 
Hence, the minima of the loss function need to be \textit{group invariant}.
One way to satisfy this is for the loss itself to be group invariant, which can be constructed 
% group invariant functions is to 
by integrating over $G$ (global pooling \citep{bronstein2021geometric}). 
A function $ I = \int_G dg F(g)$ is $G$-invariant as for $w\in G$   
\begin{align}
    w\cdot  I &=\int_G w\cdot F(g)dg = \int F(w^{-1}g) dg = \int_G F(g') d(w g') = \int_G F(g') dg'
    \label{eq:loss-invariance}
\end{align}
where we used the invariance of the Haar measure $d(wg') = dg'$.
We can change the integration to $\int_\mathcal{S} d^nx $ by change of variable $dg/d\vx$.
Since we need $\mathcal{S}$ to be lifted to $G$, the lift$:\mathcal{S}\to G$ is injective, the a map $ G\to\mathcal{S}$ need not be. 
% This means that 
$\mathcal{S}$ is homeomorphic to $G/H$, where $H\subset G$ is the stabilizer of the origin, i.e. $h\vx_0 = \vx_0,\forall h\in H$.
Since $F(g\vx_0) = F(gh\vx_0)$, we have 
\begin{align}
     I& = \int_G F(g) dg = \int_H dh \int_{G/H} dg' F(g') = V_H\int_{G/H} dg' F(g')
\end{align}
Since $G/H\sim \mathcal{S}$, the volume forms $dg'=V_H d^nx$ can be matched for some parametrization.  

\subsubsection{MSE Loss \label{ap:Loss-MSE}}
\out{
In \eqref{eq:loss-invariance} $F(g)$ is any scalar function of the features $f$ with a defined $G$ action. 
For example $F(g)$ can be mean square loss (MSE) $F(g)= \sum_n\|Q[f_n](g)\|^2$, where $f_n$ are data samples and $Q[f]$ is L-conv or another $G$-equivariant function. 
In supervised learning the input is a pair $f_n,y_n$. 
The labels $y_n$ have their own $G$ action, including being $G$-invariant. 
We can concatenate $\phi_n [f_n|y_n]$ to define a merged feature which has a well-defined $G$ action.
% Hence, in the analysis below we will assume the input is a set of $f_n$ and won't distinguish between $f_n$ and $y_n$. 
Using \eqref{eq:L-conv-def0} and  \eqref{eq:dfdg-general} defining an MSE loss yields
\begin{align}
     I[W] &= \sum_n \int_G dg \left\|  W^0 \br{I + \ba{\eps}^i g L_i\cdot {d\over dg} }\phi_n(g) \right\|^2 \cr
    &= \sum_n \int_G dg \br{ 
    \|W^0\phi_n\|^2 + \left\|W^i g L_i\cdot {d\phi_n\over dg}\right\|^2 +2\phi_n^T  W^{0T}W^i g L_i\cdot {d\over dg}\phi_n  
    }
    \label{eq:loss-MSE0}
\end{align}
where $W^i= W^0\ba{\eps}^i$. 
First we simplify the first two terms in \eqref{eq:loss-MSE0}. 
The first term is $f_n^TMfn$ where $M=W^{0T}W^0$. 
For the second term, let $U_i(g) = gL_i\vx_0$. 
Using \eqref{eq:dfdg-general} we have 
\begin{align}
     I[W] &= \sum_n \int_G dg \left\|  W^0 \br{I + \ba{\eps}^i U_i^\alpha \ro_\alpha }f_n(g) \right\|^2 +  I_\ro\cr
    &= \sum_n \int_G dg \br{  \|W^0f_n\|^2 +\|W^iU_i^\alpha \ro_\alpha f_n\|^2   } +  I_\ro \cr
     &= \sum_n \int_G dg \br{  f_n^TMf_n +\ro_\alpha f_n^T\mathbf{g}^{\alpha\beta}  \ro_\beta f_n } +  I_\ro \cr
     &= \sum_n \int_G dg \br{ M_{ab}f_n^a f_n^b +\mathbf{g}^{\alpha\beta}_{ab}  \ro_\beta f_n^b \ro_\alpha f_n^a} +  I_\ro 
\end{align}
Where $\mathbf{g}^{\alpha\beta}_{ab} = \sum_c [W^i]_{ac}U_i^\alpha [W^j]_{bc}U_j^\beta $ acts as a Riemannian metric $\mathbf{g}^{\alpha\beta}$ on $\mathcal{S}$ and a metric on the feature space $\mathcal{F}$ as $\mathbf{g}_{ab}$.

\paragraph{MSE loss}

}%%%%
% In \eqref{eq:loss-invariance} $F(g)\in \R$
% is any scalar function of the features $f$ with a defined $G$ action. 
% For example 
% $F(g)$ 
% can be 
% Using \eqref{eq:loss-invariance},
The MSE is given by $I = \sum_n\int _G dg\|Q[f_n](g)\|^2$, where $f_n$ are data samples and $Q[f]$ is L-conv or another $G$-equivariant function. 
In supervised learning the input is a pair $f_n,y_n$. 
$G$ can also act on the labels $y_n$. 
We assme that $y_n$ are either also scalar features $y_n:\mathcal{S}\to \R^{m_y}$ with a group action $g\cdot y_n(\vx)=y_n(g^{-1}\vx)$ (e.g. $f_n$ and $y_n$ are both images), or that $y_n$ are categorical.
In the latter case $g\cdot y_n = y_n$ because the only representations of a continuous $G$ on a discrete set are constant. 
% via some representation $\pi:G\to \mathrm{GL}_c(\R)$.
%have their own $G$ action, 
% including being 
% though in classification $y_n$ is usually $G$-invariant. 
We can concatenate the inputs to 
% \ry{make sure the notation f is consistent} 
$\phi_n \equiv [f_n|y_n]$ % to define a merged feature which has 
with a well-defined $G$ action $g\cdot \phi_n = [g\cdot f_n| g\cdot y_n]$.
% Therefore, in the analysis below we use the combined features $\phi_n$.
% assume the input is a set of $f_n$ and won't distinguish between $f_n$ and $y_n$. 
The collection of combined inputs $\Phi = (\phi_1,\dots, \phi_N)^T$ is an $(m+m_y)\times N$ matrix. 
Using %the combined features $\Phi$, 
equations \ref{eq:L-conv} and  \ref{eq:dfdg-general}, the MSE loss with parameters $W = \{W^0,\ba{\eps}\}$ becomes
% \ry{include a 2d function} 
% (SI\ref{ap:Loss-MSE})
\begin{align}
    I[\Phi;W] &= \int_G dg \L[\Phi;W] = \int_G dg \left\|  W^0 \br{I + \ba{\eps}^i  [\hat{L}_i]^\alpha \ro_\alpha}\Phi(g) \right\|^2 
    \cr
    &= 2\int_G dg \br{ 
    \|W^0\Phi\|^2 + \left\|W^i [\hat{L}_i]^\alpha \ro_\alpha \Phi \right\|^2 +2\Phi^T  W^{0T}W^i [\hat{L}_i]^\alpha \ro_\alpha \Phi}
    \label{eq:loss-MSE-expand}\\
    % &=\sum_n \int_G dg \br{ M_{ab}\phi_n^a \phi_n^b +\mathbf{h}^{\alpha\beta}_{ab}  \ro_\beta \phi_n^b \ro_\alpha \phi_n^a + [\hat{L}_i]^\alpha \ro_\alpha \pa{\phi_n^T V^i \phi_n}}
    % \label{eq:Loss-MSE-simplified0}\\
    &= \int_\mathcal{S} {d^nx\over\left|\ro x\over \ro g\right|} \br{\Phi^T\mathbf{m}_2\Phi + \ro_\alpha \Phi^T \mathbf{h}^{\alpha\beta} \ro_\beta \Phi +[\hat{L}_i]^\alpha \ro_\alpha \pa{\Phi^T \mathbf{v}^i \Phi} }
    \label{eq:Loss-MSE0}
    % \mbox{Invariance: } & u\in G: \quad u^T\mathbf{g}_{ab} u = \mathbf{g}_{ab}
 \end{align}
\out{
\begin{align}
    I[\Phi;W] &= \int_G dg \L[\Phi;W] = \int_G dg \left\|  W^0 \br{I + \ba{\eps}^i g L_i\cdot {d\over dg} }\Phi(g) \right\|^2 
    \cr
    &= \int_G dg \br{ 
    \|W^0\Phi\|^2 + \left\|W^i [\hat{L}_i]^\alpha \ro_\alpha \Phi \right\|^2 +2\Phi^T  W^{0T}W^i [\hat{L}_i]^\alpha \ro_\alpha \Phi}
    \label{eq:loss-MSE-expand}\\
    % &=\sum_n \int_G dg \br{ M_{ab}\phi_n^a \phi_n^b +\mathbf{h}^{\alpha\beta}_{ab}  \ro_\beta \phi_n^b \ro_\alpha \phi_n^a + [\hat{L}_i]^\alpha \ro_\alpha \pa{\phi_n^T V^i \phi_n}}
    % \label{eq:Loss-MSE-simplified}\\
    &= \int_\mathcal{S} {d^nx\over\left|\ro x\over \ro g\right|} \br{\Phi^T\mathbf{m}_2\Phi + \ro_\alpha \Phi^T \mathbf{h}^{\alpha\beta} \ro_\beta \Phi +[\hat{L}_i]^\alpha \ro_\alpha \pa{\Phi^T \mathbf{v}^i \Phi} }
    \label{eq:Loss-MSE}
    % \mbox{Invariance: } & u\in G: \quad u^T\mathbf{g}_{ab} u = \mathbf{g}_{ab}
 \end{align}
 }%%%
% \nd{Discuss stabilizers}
where $\left|\ro x\over \ro g\right|$ is the determinant of the Jacobian, $W^i = W^0 \ba{\eps}^i$ and
\begin{align}
    % \mathbf{m}_2 &=W^{0T}W^0 &
    % \mathbf{h}^{\alpha\beta}(\vx) &= W^{iT}W^j [\hat{L}_i]^\alpha [\hat{L}_i]^\beta &
    % \mathbf{v}^i &= W^{0T}W^i. \\
    \mathbf{m}_2 &=W^{0T}W^0, &
    \mathbf{h}^{\alpha\beta}(\vx) &= \ba{\eps}^{iT} \mathbf{m}_2 \ba{\eps}^j [\hat{L}_i]^\alpha [\hat{L}_j]^\beta, &
    \mathbf{v}^i &= \mathbf{m}_2\ba{\eps}^{i} . 
    \label{eq:MSE-params0}
\end{align}
From \eqref{eq:loss-MSE-expand} to \ref{eq:Loss-MSE0} we used the fact that $W^0$ and $W^i$ do not depend on $\vx$ (or $g$) to write 
\begin{align}
    2\Phi^T  W^{0T}W^i [\hat{L}_i]^\alpha \ro_\alpha \Phi & = [\hat{L}_i]^\alpha \ro_\alpha \pa{\Phi^T  W^{0T}W^i \Phi } = [\hat{L}_i]^\alpha \ro_\alpha \pa{\Phi^T \mathbf{m}_2 \ba{\eps}^i \Phi }
\end{align}
Note that $\mathbf{h}$ has feature space indices via $[\ba{\eps}^{iT} \mathbf{m}_2 \ba{\eps}^j]_{ab}$, with index symmetry
$\mathbf{h}^{\alpha\beta}_{ab}=\mathbf{h}^{\beta\alpha}_{ba} $.
When $\mathcal{F}= \R$ (i.e. $f$ is a 1D scalar), $\mathbf{h}^{\alpha\beta}$ becomes a
a Riemannian metric 
% $\mathbf{h}^{\alpha\beta}$ 
for $\mathcal{S}$.
In general $\mathbf{h}$ combines a 2-tensor $\mathbf{h}_{ab}=\mathbf{h}_{ab}^{\alpha\beta}\ro_\alpha \ro_\beta \in T\mathcal{S} \otimes T\mathcal{S}$ with an inner product $h^T\mathbf{h}^{\alpha\beta}f $ on the feature space $\mathcal{F}$. 
Hence $\mathbf{h}\in T\mathcal{S}\otimes T\mathcal{S}\otimes \mathcal{F}^*\otimes \mathcal{F}^*$ is a $(2,2)$-tensor, with $\mathcal{F}^*$ being the dual space of $\mathcal{F}$.   
% \nd{Is it symmetric?}
% Note that $\mathbf{h}^{\alpha\beta}(\vx)$ is a function of $\vx$ because $\hat{L}_i(\vx)$ depends on $\vx$.
%
\paragraph{Loss invariant metric transformation} 
% We can show that the $(2,2)$-tensor $\mathbf{h}$ is invariant under $G$ in the same sense as a Riemannian metric.
The metric
$\mathbf{h}$ transforms equivariantly as a 2-tensor. % under the $G$ action. 
As discussed under \eqref{eq:dfdg-general}, 
$[\hat{L}_i(v \vx)]^\alpha = v^\alpha_\beta \hat{L}_i(\vx)^\beta $ and 
\begin{align}
    v\cdot \mathbf{h}^{\alpha\beta}&= \mathbf{h}^{\alpha\beta}(v^{-1}\vx) =[v^{-1}]^\alpha_\rho [v^{-1}]^\beta_\gamma
    \mathbf{h}^{\rho\gamma}(\vx), &
    (v&\in G).
    \label{eq:h-metric-covariance}
\end{align}
Note that $v\cdot \mathbf{m}_2=\mathbf{m}_2$ since $f_n$ and $y_n$ are scalars.
For example, let $G=SO(2)$ and $R(\xi)  \in SO(2)$ be rotation by angle $\xi$.  
Since there is only one $L_i=L$, the metric factorizes to
\begin{align}
    \mathbf{h}^{\alpha\beta}_{ab} = [\ba{\eps}^T\mathbf{m}_2\ba{\eps}]_{ab} \otimes [\hat{L}\hat{L}^T]^{\alpha\beta}
\end{align}
To find $R(\xi)\cdot \mathbf{h}$ we only need to calculate $R(\xi)^{-1} \hat{L}$. 
With $g= R(\theta)$, we have $\hat{L}(\vx) = R(\theta) L \vx_0 = (-y,x) = r(-\sin\theta,\cos\theta)$ from \eqref{eq:gLx0-so2}. 
Therefore, $R(\xi)^{-1}\hat{L} =R(\theta -\xi) = \hat{L}(R(\xi^{-1} \vx)$. 
Using \eqref{eq:MSE-params} in \eqref{eq:h-metric-covariance}, the transformed metric becomes 
\begin{align}
    % R(\xi) &= \mat{\cos\xi&-\sin\xi\\ \sin\xi &\cos\xi}, \quad 
    % R(\xi)^{-1}g &= R(\xi)^{-1}R(\theta) = R(\theta - \xi) \cr 
    R(\xi)\cdot \mathbf{h}^{\alpha\beta}(R(\theta)\vx_0) = \ba{\eps}^T\mathbf{m}_2 \ba{\eps} \otimes [R(-\xi)\hat{L}]^\alpha [R(-\xi)\hat{L}]^\beta = \mathbf{h}^{\alpha\beta}(R(\theta-\xi)\vx_0) ,
\end{align}

\subsubsection{Third term as a boundary term}
% \nd{For SI}
Since terms in \eqref{eq:Loss-MSE} are scalars, they can be evaluated in any basis.
If $\mathcal{S}$ can be lifted to multiple Lie groups, either group can be used to evaluate \eqref{eq:Loss-MSE}. 
For example $\R^n/0$ can be lifted to both $T_n$ and $SO(n)\times \R_+$. 
% The last term in \eqref{eq:Loss-MSE} is straightforward 
For the translation group $G=T_n$ we have $gL_i = L_i$ and $[\hat{L}_i]^\alpha
%=[gL_i\vx_0]^\alpha 
= \delta^\alpha_i$ (SI \ref{ap:example-Tn}) and $|\ro g/\ro x| = 1$ and $dg = d^n x$. 
Thus, the last term in \eqref{eq:Loss-MSE} simplifies to a complete divergence 
%$[L_i\vx_0]^\alpha \ro_\alpha (\Phi^T \mathbf{v}^i \Phi) = $
$\int d^nx \ro_i (\Phi^T \mathbf{v}^i \Phi)$. %, with $V^i = W^{0T}W^i$. 
% Therefore, the last term is a complete divergence 
Using the generalized Stoke's theorem $\int_\mathcal{S} dw = \int_{\ro \mathcal{S}} w $, the last term in \eqref{eq:Loss-MSE} becomes a boundary term.
When $\mathcal{S}$ is non-compact, the last term is 
$ I_\ro= \int_{\ro \mathcal{S}} d\Sigma_i \Phi^T \mathbf{v}^i \Phi$, where $d\Sigma_i$ is the normal times the volume form of the $(n-1)$D boundary $\ro \mathcal{S}$ and is in the radial direction (e.g. for $\mathcal{S}=\R^n$ the  boundary is a hyper-sphere $\ro \mathcal{S}= S^{n-1}$).
Generally we expect the features $\phi_n$ to be concentrated in a finite region of the space and that they go to zero as $r\to \infty$ (if they don't the loss term $\Phi^T\mathrm{m}_2\Phi$ will diverge).
Thus, the last term in \eqref{eq:Loss-MSE} generally becomes a vanishing boundary term and does not matter.

% Note that $[\hat{L}_i(\vx)]^\alpha\in \R$ is a number and so is $\bk{\mathbf{v}^i}\equiv\Phi^T \mathbf{v}^i\Phi \in \R $. 

\out{
\nd{Show $SO(2)\times \R_+$ and $T_n$ examples.}
If $\ro_\alpha \pa{\left|\ro g\over \ro x\right|[\hat{L}_i]^\alpha}=0$, the last term in \eqref{ap:Loss-MSE} becomes a total derivative $\ro_\alpha (\left|\ro g\over \ro x\right|[\hat{L}_i]^\alpha \Phi^T\mathbf{v}^i\Phi)$. 
This happens for instance for the translation group, where the Jacobian is identity and  

\nd{
shouldn't $\ro^\alpha [\hat{L}_i]_\alpha$ be zero usually? I know it doesn't hold for scaling. 
But it holds for $SO(2)$ and translations. 
Is there a rule for compact groups? 
$\hat{L}_i$ is related to the Maurer-Cartan form $\omega = g^{-1} \ro_\alpha g d\vx^\alpha$, which encodes the pushforward.
Write $f(gv_\eps) = f(g+ \eta(\eps)^\alpha \ro_\alpha g)$. 
We have $ \eta^\alpha \ro_\alpha g = \eps^i g L_i$, so $ \eps^i L_i \eta^\alpha g^{-1} \ro_\alpha g = \eta \cdot \omega $.
If $ gL_i=c_i^\alpha g \omega_\alpha = dg\cdot c^i $, since $dg$ is an exact form we have $d^2g =0 $, but $\ro_\alpha (|dg/dx|\ro^\alpha g) = d * dg$ and that's not necessarily zero.    
}
% Let $H\subset G$ be a compact subgroup of $G$. 
% For $L^H_i\in T_IH \subset T_IG$, 

}%%

\subsubsection{MSE Loss for translation group $T_n$ }
% In \eqref{eq:Loss-MSE}, the first term is simply $\phi^T \mathbf{m}_2 \phi$ where $M= {W^0}^TW^0$. 
% For the second term, from \eqref{eq:dfdg-general}, $gL_i\cdot df/dg = \hat{L}_i f= [gL_i\vx_0]^\alpha \ro_\alpha f(\vx)$. 
% We need to multiply this term by its transpose. 
% Let's first consider the case where $f(\vx)\in\R $ has only one component, so that $W^0\in \R$ and $\ba{\eps}^i\in \R$.
% Thus, the second term is $\hat{L}_i(\phi)^T W^{iT}W^j\hat{L}_j(\phi)$.
% As an example, let us work out $T_n$ ($n$D translation). 
% \nd{For SI}
We have $gL_i = L_i$ and $[\hat{L}_i]^\alpha 
%=[gL_i\vx_0]^\alpha 
= \delta^\alpha_i$ (SI \ref{ap:example-Tn}),
the last term in \eqref{eq:Loss-MSE} becomes a complete divergence 
%$[L_i\vx_0]^\alpha \ro_\alpha (\Phi^T \mathbf{v}^i \Phi) = $
$I _\ro = \int d^nx \ro_i (\Phi^T \mathbf{v}^i \Phi)$. %, with $V^i = W^{0T}W^i$. 
% Therefore, the last term is a complete divergence 
Using the generalized Stoke's theorem $\int_\mathcal{S} dw = \int_{\ro \mathcal{S}} w $, 
%the last term in \eqref{eq:Loss-MSE} becomes a boundary term.
when $\mathcal{S}$ is non-compact, %the last term is 
$I _\ro= \int_{\ro \mathcal{S}} d\Sigma_i \Phi^T \mathbf{v}^i \Phi$.
Here $d\Sigma_i$ is the normal times the volume form of the $(n-1)$D boundary $\ro \mathcal{S}$ (e.g. for $\mathcal{S}=\R^n$ the  boundary is a hyper-sphere $\ro \mathcal{S}= S^{n-1}$).
Generally we expect the features $\phi_n$ to be concentrated in a finite region of the space and that they go to zero as $r\to \infty$ (if they don't the loss term $\Phi^T\mathrm{m}_2\Phi$ will diverge).
Thus, the last term in \eqref{eq:Loss-MSE} generally becomes a vanishing boundary term and does not matter.  

Next, the second term in \eqref{eq:Loss-MSE} can be worked out as
% Using this in the $\|W^igL_i \cdot d\phi/dg\|^2$ term we obtain  
\begin{align}
    \hat{L}_i^\alpha \ro_\alpha \phi^T \ba{\eps}^{iT} \mathbf{m}_2 \ba{\eps}^j\hat{L}_j^\beta \ro_\beta \phi
    % W^i[\hat{L}_i]^\alpha \ro_\alpha \phi {W^j}^T[\hat{L}_j]^\beta \ro_\beta \phi^T 
    &=  \ro_j \phi^T\ba{\eps}^{iT} \mathbf{m}_2 \ba{\eps}^j \ro_i \phi = \ro_j \phi^T \mathbf{h}^{ji} \ro_i \phi
\end{align}
where $\mathbf{h}^{ji}=\ba{\eps}^{iT} \mathbf{m}_2 \ba{\eps}^j  $ is a general, space-independent metric compatible with translation symmetry. 
When the weights $[W^i]^a_b \sim \mathcal{N}(0,1)$ are random Gaussian, we have $W^{jT} W^i \approx m^2 \delta^{ij}$ and we recover the Euclidean metric. 
With the lst term vanishing, the loss function \eqref{eq:Loss-MSE} has a striking resemblance to a Lagrangian used in physics, as we discuss next.
% Next, we will discuss this relation to physics and the consequences of that. 

\subsubsection{Boundary term with spherical symmetry}
When $\mathcal{S}\sim \R^n$ and  $G=T_n$, the third term becomes a boundary term. 
But we can also have $G= SO(n)\times \R_+$ (spherial symmetry and scaling). 
The boundary $\ro \mathcal{S}\sim S^{n-1}$, which has an $SO(n)$ symmetry. 
% Does this mean that $d\Sigma_\alpha [gL_i\vx_0]^\alpha =0 $? 
% When we fix $\vx_0=(1,0,\dots)$, only $n-1$ out of the $n(n-1)/2$ of the $L_i\in so(3)$ yield $L_i\vx_0\ne0$. 
% These are the $L_i$ acting on the first dimension. 
% We can index these as $L_\alpha = E_{0\alpha}-E_{\alpha0}$ for $\alpha\in \{1\to n-1\}$.
% We have $[L_\alpha \vx_0]^\beta = \delta_\alpha^\beta $
The normal 
$d\Sigma(\vx)$ is a vector pointing in the radial direction and $g$ is the lift for $\vx$. 
Since $g\in SO(n)$, we have 
\begin{align}
    d\Sigma_\beta [gL_\alpha \vx_0]^\beta &= d\Sigma^T g L_\alpha \vx_0 = [g^T d\Sigma]^T L_\alpha \vx_0 
\end{align}
Since $g\in SO(n)$, $g^T = g^{-1}$ and $g^Td\Sigma(g\vx_0) = d\Sigma(\vx_0) =V_{n-1} \vx_0$, meaning the normal vector is rotated back toward $\vx_0$. 
Here $V_{n-1}$ is the volume of the boundary $S^{n-1}$. 
Hence we have
\begin{align}
    d\Sigma^T g L_i \vx_0 = \vx_0^T L_i \vx_0 = 0 
\end{align}
for all generators $L_i \in so(n)$ because $L_i= -L_i^T$ and hence diagonal entries like $\vx_0^T L_i \vx_0$ are zero. 
Only the scaling generator $L_0=I$ we have $\vx_0^T L_0 \vx_0 = 1$.
This means that the last term in \eqref{eq:Loss-MSE} can be nonzero at the boundary only if $\Phi\mathbf{v}^i\Phi$ is in the radial direction, meaning $\ba{\eps}^0 = 0$, and $\Phi$ does not vanish at the boundary. 
However, a non-vanishing $\Phi$ at the boundary results in diverging loss unless the mass matrix $\mathbf{m}_2$ has eigenvalues equal to zero.
This is what happens relativistic theories where light rays can have nonzero $\Phi$ at infinity because they are massless.

\out{

\paragraph{Boundary term and conserved quantity}
Since terms in \eqref{eq:Loss-MSE} are scalars, they can be evaluated in any basis.
If $\mathcal{S}$ can be lifted to multiple Lie groups, either group can be used to evaluate \eqref{eq:Loss-MSE}. 
For example $\R^n/0$ can be lifted to both $T_n$ and $SO(n)\times \R_+$. 
% The last term in \eqref{eq:Loss-MSE} is straightforward 
For the translation group $G=T_n$ we have $gL_i = L_i$ and $[\hat{L}_i]^\alpha
%=[gL_i\vx_0]^\alpha 
= \delta^\alpha_i$ (SI \ref{ap:example-Tn}) and $|\ro g/\ro x| = 1$ and $dg = d^n x$. 
Thus, the last term in \eqref{eq:Loss-MSE} simplifies to a complete divergence 
%$[L_i\vx_0]^\alpha \ro_\alpha (\Phi^T \mathbf{v}^i \Phi) = $
$\int d^nx \ro_i (\Phi^T \mathbf{v}^i \Phi)$. %, with $V^i = W^{0T}W^i$. 
% Therefore, the last term is a complete divergence 
Using the generalized Stoke's theorem $\int_\mathcal{S} dw = \int_{\ro \mathcal{S}} w $, the last term in \eqref{eq:Loss-MSE} becomes a boundary term.
When $\mathcal{S}$ is non-compact, the last term is 
$ I_\ro= \int_{\ro \mathcal{S}} d\Sigma_i \Phi^T \mathbf{v}^i \Phi$, where $d\Sigma_i$ is the normal times the volume form of the $(n-1)$D boundary $\ro \mathcal{S}$ and is in the radial direction (e.g. for $\mathcal{S}=\R^n$ the  boundary is a hyper-sphere $\ro \mathcal{S}= S^{n-1}$).
Generally we expect the features $\phi_n$ to be concentrated in a finite region of the space and that they go to zero as $r\to \infty$ (if they don't the loss term $\Phi^T\mathrm{m}_2\Phi$ will diverge).
Thus, the last term in \eqref{eq:Loss-MSE} generally becomes a vanishing boundary term and does not matter.

% Note that $[\hat{L}_i(\vx)]^\alpha\in \R$ is a number and so is $\bk{\mathbf{v}^i}\equiv\Phi^T \mathbf{v}^i\Phi \in \R $. 

\nd{Show $SO(2)\times \R_+$ and $T_n$ examples.}
If $\ro_\alpha \pa{\left|\ro g\over \ro x\right|[\hat{L}_i]^\alpha}=0$, the last term in \eqref{ap:Loss-MSE} becomes a total derivative $\ro_\alpha (\left|\ro g\over \ro x\right|[\hat{L}_i]^\alpha \Phi^T\mathbf{v}^i\Phi)$. 
This happens for instance for the translation group, where the Jacobian is identity and  
\nd{
shouldn't $\ro^\alpha [\hat{L}_i]_\alpha$ be zero usually? I know it doesn't hold for scaling. 
But it holds for $SO(2)$ and translations. 
Is there a rule for compact groups? 
$\hat{L}_i$ is related to the Maurer-Cartan form $\omega = g^{-1} \ro_\alpha g d\vx^\alpha$, which encodes the pushforward.
Write $f(gv_\eps) = f(g+ \eta(\eps)^\alpha \ro_\alpha g)$. 
We have $ \eta^\alpha \ro_\alpha g = \eps^i g L_i$, so $ \eps^i L_i \eta^\alpha g^{-1} \ro_\alpha g = \eta \cdot \omega $.
If $ gL_i=c_i^\alpha g \omega_\alpha = dg\cdot c^i $, since $dg$ is an exact form we have $d^2g =0 $, but $\ro_\alpha (|dg/dx|\ro^\alpha g) = d * dg$ and that's not necessarily zero.    
}
% Let $H\subset G$ be a compact subgroup of $G$. 
% For $L^H_i\in T_IH \subset T_IG$, 

}%%%%%
\out{
\nd{revise!!!}
The last term can be written as $gL_i\cdot {d\over dg} (f_n^T W^{0T}W^i f_n)$. 
Separating out $\ba{\eps}^i$ and defining  $f'_n \equiv W^{0}f_n$, the last term reads $[\ba{\eps}^i]_{ab} gL_i\cdot {d\over dg}[{f'_n}^a{f'_n}^b]$.
Next, we note this is a complete derivative, as $\ba{\eps}^i gL_i\cdot df/dg = \delta \vx^\alpha \ro_\alpha f(\vx)$ was the first-order Taylor expansion of $f(\vx+\delta \vx ) = f(g(I+\ba{\eps}^iL_i)\vx_0)$. 
\nd{Not always... if $\delta\vx$ is space dependent then it doesn't become a surface integral}
Using $\delta \vx$ and changing $\int_G dg$ to $\int_\mathcal{S} d\vx $, and Stoke's theorem ($\int_\mathcal{S} d f = \int_{\ro \mathcal{S}} f$ ) 
we find that the last term $ I_\ro$ is a total derivative and yields a boundary term. 
\begin{align}
     I_\ro =&\sum_n\int_G 2f_n^T W^{0T} W^i g L_i\cdot {df_n\over dg} dg  = \sum_n\int_G dg [\ba{\eps}^i]_{ab} gL_i\cdot {d\over dg}[{f'_n}^a{f'_n}^b]\cr
    =& \sum_n\int_\mathcal{S} d\vx  [\delta\vx^\alpha]_{ab} \ro_\alpha [{f'_n}^a{f'_n}^b] 
    = \left. \Sigma_\alpha(\vx) [\delta\vx^\alpha]_{ab} \br{\sum_n {f'_n}^a{f'_n}^b} \right|_{\vx\to \ro \mathcal{S}}
\end{align}
where $\sigma_\alpha (\vx)$ is the normal vector of a hyper-surface defining the boundary $\ro\mathcal{S}$. 
If $ \mathcal{S}$ is compact, its boundary $\ro \mathcal{S} = \emptyset$ so boundary terms vanish. 
When $\mathcal{S}$ is non-compact, the boundary term may capture important things such as conserved quantities. 
% generally learnable features have finite support, so $f(g)\to 0 $ when $g \to \ro \mathcal{S}$. 
\nd{
but if it is non-zero it can be a conserved quantity. 
Check its Noether's Theorem. 
}

\nd{
\paragraph{Equivariance of parameters}
Derive what $w\cdot \phi$ does to $\mathbf{h}$ etc.
$w\cdot Q[f](g) = W^0f(w^{-1}g(I+\ba{\eps}^iL_i)) = $

\paragraph{Conservation laws and Noether's theorem}
Derive Noether's theorem for MSE. 
The version for $\delta \phi$ should be the familiar one. Is there a version for $W$? 
}
}%%%%

\subsection{Robustness to random noise
\label{ap:generalization}
}
Equivariant neural networks are hoped to be more robust than others. 
% Equivariance should improve generalization. 
One way to check this is to see how the network would perform for an input $\phi'= \phi + \delta \phi$ which adds a small perturbation $\delta \phi$ to a real data point $\phi$. 
Robustness to such perturbation would mean that, for optimal parameters $W^*$, the loss function would not change,
i.e. $I [\phi';W^*]=I [\phi;W^*]$. 
This can be cast as a variational equation, requiring $I $ to be minimized around real data points $\phi$. 
Writing $I[\phi;W]= \int d^nx \L[\phi;W]$, we have 
\begin{align}
    % I[\phi;W]&= \int d^nx \L[\phi;W], &
    \delta I [\phi;W^*] &=\int_\mathcal{S} d^nx\br{
    {\ro \L \over \ro \phi^a }\delta \phi^a + {\ro \L \over \ro (\ro_\alpha \phi^a) } \ro_\alpha (\delta \phi^a) 
    }
\end{align}
Doing a partial integration on the second term, we get 
\begin{align}
    \delta I [\phi;W^*] &=\int_\mathcal{S} d^nx\br{
    {\ro \L \over \ro \phi^b } - \ro_\alpha {\ro \L \over \ro (\ro_\alpha \phi^b) }  
    } \delta \phi^b + \int_\mathcal{S} d^nx\ro_\alpha \br{
    {\ro \L \over \ro (\ro_\alpha \phi^b) } \delta \phi^b 
    } \cr
    &= \int_\mathcal{S} d^nx\br{
    {\ro \L \over \ro \phi^b } - \ro_\alpha {\ro \L \over \ro (\ro_\alpha \phi^b) }  
    } \delta \phi^b + \int_{\ro \mathcal{S}} d^{n-1}\Sigma_\alpha \br{
    {\ro \L \over \ro (\ro_\alpha \phi^b) } \delta \phi^b 
    } 
    \label{eq:loss-Euler-Lagrange-derivation}
\end{align}

If we want equivariant networks to be robust, then both terms  in \eqref{eq:loss-Euler-Lagrange-derivation} need to be zero.  
We show that the first term is the classic Euler-Lagrange (EL) equation, and the second term is related to conservation laws. 

We use the Stoke's theorem  to change the second term to a boundary integral. Since features $\phi$ have finite support, $\phi(\vx)\to0$ as $|\vx|\to \infty $, and the boundary term vanishes.
The first term in \eqref{eq:loss-Euler-Lagrange-derivation} is the classic Euler-Lagrange (EL) equation. 
Thus, requiring robustness, i.e. $\delta I [\phi;W^*]/\delta \phi=0$ means for optimal parameters $W^*$, the  data $\phi$ satisfies the EL equations
\begin{align}
    \mbox{Robustness to random noise} \Longleftrightarrow \mbox{EL: }
    \quad 
    {\ro \L \over \ro \phi^b } - \ro_\alpha {\ro \L \over \ro (\ro_\alpha \phi^b) } = 0 
    \label{eq:loss-Euler-Lagrange0}
\end{align}

Applying this to the MSE loss \eqref{eq:Loss-MSE}, \eqref{eq:loss-Euler-Lagrange0} becomes
\begin{align}
    % \mbox{EL:} \quad 
    \mathbf{m}_2 \phi - \ro_\alpha \pa{|J| \mathbf{h}^{\alpha\beta} \ro_\beta \phi } - \ro_\alpha\pa{|J|\mathbf{v}^i [\hat{L}_i]^\alpha }\phi =0 
    \label{eq:loss-EL-MSE0}
\end{align}
where $|J| = |\ro g/\ro x|$ is the determinant of the Jacobian.
For the translation group, \eqref{eq:loss-EL-MSE0} becomes a Helmholtz equation
\begin{align}
    \mathbf{h}^{ij}\ro_i\ro_j\phi= \ba{\eps}^i\mathbf{m}_2 \ba{\eps}^j \ro_i\ro_j\phi = \mathbf{m}_2 \phi
\end{align}
where $\mathbf{h}^{ij}\ro_i\ro_j = \del^2$ is the Laplace-Beltrami operator with $\mathbf{h}$ as the metric. 

\subsection{Conservation laws \label{ap:noether} }
The equivariance condition \eqref{eq:equivairance-general} can be written for the integrand of the loss $\L[\phi,W]$. 
% If we write the equivariance equation for infinitesimal $v_\eps $, we obtain a vector field which is divergence free. 
% For $\delta I [\phi;W^*]/\delta \phi=0$, beside the EL equations, the variations at the boundary should be orthogonal to the normal to the boundary $d^{n-1}\Sigma$. 
Since $G$ is the symmetry of the system, transforming an input $\phi\to w\cdot \phi$ by $w\in G$  the integrand changes equivariantly as $\L[w\cdot \phi]= w\cdot \L[\phi] $. 
% We will now show that 
Now, let $w$ be an infinitesimal $w\approx I+\eta^i L_i$. 
The action $w\cdot \phi $ can be written as a Taylor expansion, similar to the one in L-conv, yielding
\begin{align}
    w\cdot \phi(\vx)& = \phi(w^{-1}\vx) = \phi((I-\eta^i L_i)\vx) = \phi(\vx) - \eta^i [L_i \vx]^\alpha \ro_\alpha \phi(\vx) \cr
    & = \phi(\vx) + \delta \vx^\alpha \ro_\alpha \phi(\vx) = \phi(\vx) + \delta \phi(\vx) 
    \label{eq:delta-phi}
\end{align}
with $\delta \vx^\alpha = - \eta^i [L_i \vx]^\alpha $ and $\delta \phi = \delta \vx^\alpha \ro_\alpha \phi$. 
Similarly, we have $w\cdot \L =\L + \delta \vx^\alpha \ro_\alpha \L$. 
Next, we can use the chain rule to calculate $\L[w\cdot \phi]$.
\begin{align}
    \L[w\cdot \phi]& = \L[\phi(\vx) + \delta \phi(\vx) ]= \L[\phi] + {\ro \L \over \ro \phi^b } \delta \phi^b + {\ro \L \over \ro (\ro_\alpha \phi^b) } \delta \ro_\alpha \phi^b \cr
    &= \L[\phi] + \br{{\ro \L \over \ro \phi^b }  -\ro_\alpha {\ro \L \over \ro (\ro_\alpha \phi^b) }}\delta \phi^b + \ro_\alpha \pa{ {\ro \L \over \ro (\ro_\alpha \phi^b) }\delta \phi^b }
    \label{eq:Noether-expand}
\end{align}

where we used the fact that $\delta \vx = \eta^i L_i \vx $ can vary independently from $\vx $ (because of $\eta^i$), and so $\delta \ro_\alpha \phi^b =  \ro_\alpha \delta\phi^b $.
The same way, $\delta \vx^\alpha \ro_\alpha \L = \ro_\alpha (\delta \vx^\alpha \L)$. 
Now, if $\phi$ are the real data and the parameters in $\L$ minimize generalization error, then $\L$ satisfies \eqref{eq:loss-EL-MSE0}. 
This means that the first term in \eqref{eq:Noether-expand} vanishes. 
Setting the second term equal to $w\cdot \L$ we get 
\begin{align}
    \L[w\cdot \phi] - w\cdot \L[\phi]& = \ro_\alpha \br{{\ro \L \over \ro (\ro_\alpha \phi^b) }\delta \phi^b - \delta \vx^\alpha \L } =0
\end{align}
Thus, the terms in the brackets are divergence free. 
These terms are called a Noether conserved current $J^\alpha$. 
In summary
% When generalization error is minimized as in \eqref{eq:loss-Euler-Lagrange}, an infinitesimal $w\approx I+\eta^i L_i$, with $\delta \phi = \eps^i \hat{L}_i \phi$, results in a conserved current (SI \ref{ap:generalization})
% the correct $L_i \in \mathfrak{g}$ appear in $I[\phi;W^*]$.
% Let's see how the Lagrangian $\L$ (integrand of the loss) changes when 
% When the generalization error is minimized (i.e. EL equations are satisfied) the term that produced the vanishing boundary term can
\begin{align}
    \mbox{Noether current: } J^\alpha &= {\ro \L \over \ro (\ro_\alpha \phi^b) } \delta \phi^b - {\ro \L \over \ro \vx^\alpha } \delta \vx^\alpha , & \delta I [\phi;W^*]=0 \quad \Rightarrow \ro_\alpha J^\alpha &= 0
\end{align}
% This is the integrand of the boundary term 
% This $J$ is called the Noether current in physics.
$J$ captures the change of the Lagrangian $\L$ along symmetry direction $\hat{L}_i$.
Plugging $\delta \phi = \delta \vx^\alpha \ro_\alpha \phi$from \eqref{eq:delta-phi} we find
\begin{align}
    &\ro_\alpha \br{{\ro \L \over \ro (\ro_\alpha \phi^b) }\delta \phi^b - \delta \vx^\alpha \L } =\delta \vx^\beta \ro_\alpha \br{{\ro \L \over \ro (\ro_\alpha \phi^b) }\ro_\beta \phi^b - \delta^\alpha_\beta  \L } = \delta \vx^\beta \ro_\alpha T^\alpha_\beta.
\end{align}
$T^\alpha_\beta$ is known as the stress-energy tensor in physics \citep{landau2013classical}. 
It is the Noether current associated with space (or space-time) variations $\delta \vx$. 
It appears here because $G$ acts on the space, as opposed to acting on feature dimensions. 
For the MSE loss we have
\begin{align}
    T^\alpha_\beta \equiv& {\ro \L \over \ro (\ro_\alpha \phi^b) }\ro_\beta \phi^b - \delta^\alpha_\beta  \L = \ro_\rho \phi^T\pa{\delta^\lambda_\beta  \mathbf{h}^{\alpha\rho}- \delta^\alpha_\beta \mathbf{h}^{\rho\lambda} }\ro_\lambda \phi - \phi^T \mathbf{m}_2 \phi
\end{align}
\out{
\nd{
Does the loss not change along $T^\alpha_\beta$? 
We are putting real data into the optimized network (i.e. optimal weights). then we are asking how the loss function integrand (Lagrangian) changes if we move infinitesimally over space. 
The answer is that it changes such that $T^\alpha_\beta$ remains divergence free. 
It's like saying as time progresses, the momentum doesn't change. 
Except, here we don't have time and it's about directions in space. 
% Along the directions of $T^\alpha_\beta$ 
% $T$ is a tensor, with $T^\alpha_\beta$ determining
}

}%%%%
% If we write the equivariance equation for infinitesimal $v_\eps $, we obtain a vector field which is divergence free. 
% and the requirement for the boundary term to vanish is the statement that the Noether current be conserved (i.e. divergence free) or 
% \nd{Becomes Euler-Lagrange equations}
It would be interesting to see if the conserved currents can be used in practice as an alternative way for identifying or discovering symmetries.
% We conclude by remarking on similarity between ML and physics problems. 

\section{Tensor notation details \label{ap:tensor-long} }

\begin{figure}
    \centering
    \includegraphics[width=.8\linewidth]{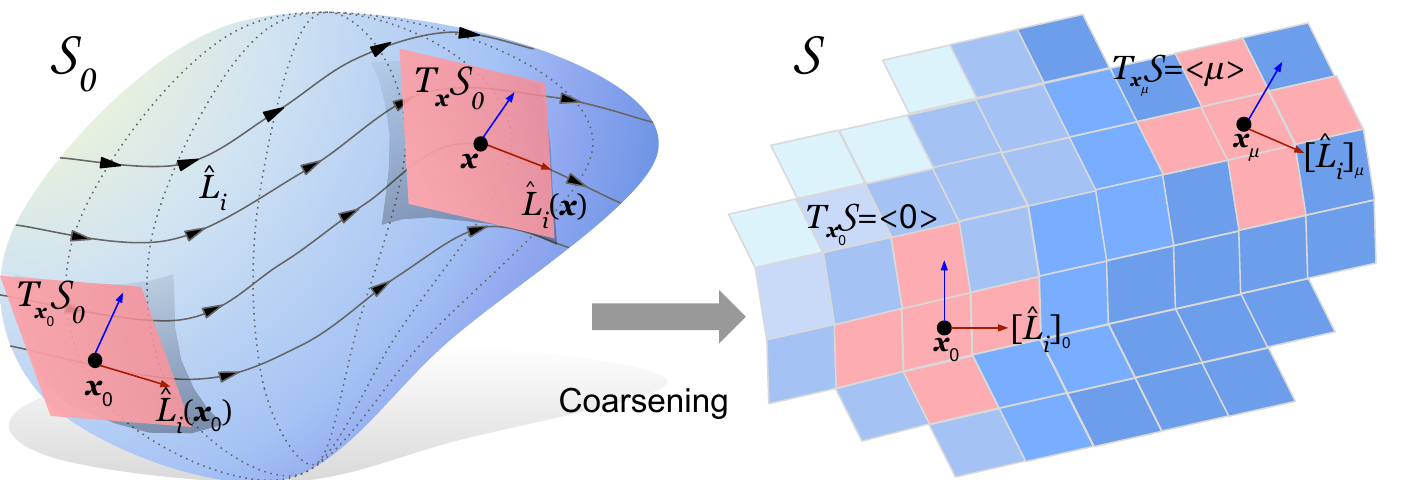}
    \caption{\textbf{
    Manifold vs. discretized Space}
    While real systems can have a continuous manifold  
    $\mathcal{S}_0$ as their base space, often the data collected from is a discrete array $\mathcal{S}$. 
    The discretization (coarsening) will induce some of the topology of $\mathcal{S}_0$ on $\mathcal{S}$ as a graph. 
    Graph neighborhoods $<\mu>$ on the discrete $\mathcal{S}$ represent tangent spaces $T_{\vx_\mu}\mathcal{S}$ and approximate $T_\vx \mathcal{S}_0$.  
    % The coarsening (discretization) produces a space $\mathcal{S}$.  The Lie group $G$ acts on $\mathcal{S}$ and closely approximates $G_0$ ($G\simeq G_0$, e.g. image rotations approximating real rotations).
    The lift takes $\vx\in \mathcal{S}_0$ to $g\in G$, and maps the tangent spaces $T_\vx \mathcal{S}_0\to T_gG$. 
    Each Lie algebra basis $L_{i}\in \mathfrak{g}=T_IG$ generates a vector field on the tangent bundle $TG$ via the pushforward as $L^{(g)} = gL_{i}g^{-1}$. 
    Due to the lift, each $L_i$ also generates a vector field $\hat{L}_i^\alpha(\vx)\ro_\alpha = [gL_i\vx_0]^\alpha \ro_\alpha$, with $\vx=g\vx_0$. 
    Analogously, on $\mathcal{S}$ we get 
    % Since $\mathcal{S}_0$ is lifted to $G$, $L_{0i}$ also generate a flow on $\mathcal{S}_0$, captured by the vector field $\hat{L}_i^\alpha(\vx)\ro_\alpha$. 
    % Similarly, the Lie algebra basis $L_i$ of $G\simeq G_0$ generates a flow on $TG$ and 
    a vector field $[\hat{L}_i]_\mu= g_\mu L_i\vx_0$ on $T\mathcal{S}$, with $\vx_\mu=g_\mu\vx_0$. 
    % Note that tangent spaces on $\mathcal{S}$ are $T_{\vx_\mu}\mathcal{S}=<\mu>$, the neighborhoods of each $\vx_\mu$, and vector fields are weighted graphs, modifying neighbor weights in the graph $\mA$ encoding the topology of $\mathcal{S}$.
    % tensors flowing from every $\vx_\mu$ within the neighborhood $<\mu>$. 
    Note that depicting $\mathcal{S}$ and $\mathcal{S}_0$ as 2D is only for convenience. 
    They may have any dimensions. 
    }
    % \label{fig:Lie-group-3-spaces}
    \label{fig:S0S}
\end{figure}

If the dataset being analyzed is in the form of $f(\vx)$ for some sample of points $\vx$, together with derivatives $\del f(\vx) $, we can use the L-conv formulation above. 
However, in many datasets, such as images, $f(\vx)$ is given as a finite dimensional array or tensor, with $\vx$ taking values over a grid. 
% In this case all the discussions above still hold, as discuss now. 
% In particular, e
Even though the space $\mathcal{S}$ is now discrete, the group which acts on it can still be continuous (e.g. image rotations). 
Let $\mathcal{S}= \{\vx_0,\dots \vx_{d-1}\}$ contain $d$ points. 
Each $\vx_\mu$ represents a coordinate in higher dimensional grid.
For instance, on a $10\times 10$ image, $\vx_0$ is $(x,y)=(0,0)$ point and $ \vx_{99}$ is $(x,y) =(9,9)$.

\paragraph{Feature maps}
To define features $f(\vx_\mu)\in \R^m$ for $\vx_\mu \in \mathcal{S}$, we embed $\vx_\mu \in \R^d$ and 
encode them as the canonical basis (one-hot) vectors with components $[\vx_\mu]^\nu = \delta_\mu^\nu$ (Kronecker delta), e.g. $\vx_0 = (1,0,\dots, 0)$.
The feature space becomes $\mathcal{F}= \R^d \otimes \R^m $, meaning feature maps $\vf \in \mathcal{F}$ are $d\times m$ tensors, with $f(\vx_\mu) = \vx_\mu^T \vf = \vf_\mu $.

\paragraph{Group action}
% \nd{Change to $\vx_\mu^T g^T$ for less confusion.}
Any subgroup $G\subseteq \mathrm{GL}_d(\R)$ of the general linear group (invertible $d\times d$ matrices) acts on $\R^d$ and $\mathcal{F}$. 
Since $\vx_\mu \in \R^d$, $g \in G$ also naturally act on $\vx_\mu$. 
The resulting $\vy = g \vx_\mu $ is a linear combination $\vy = c^\nu \vx_\nu$ of elements of the discrete $\mathcal{S}$, not a single element.
The action of $G$ on $\vf$ and $\vx$, can be defined in multiple equivalent ways.
% The most elegant one is if w
We define $f(g \cdot \vx_\mu) = \vx_\mu^T g^T\vf, \forall g\in G$.
% consistent with the lift. 
% This makes the group action on $\vf$ the natural action of $G$, meaning 
For $w\in G$ we have 
\begin{align}
    w\cdot f(\vx_\mu)&= f(w^{-1}\cdot \vx_\mu) = \vx_\mu^T w^{-1T}\vf = %\vx_\mu^T [w^T]^{-1} \vf = 
    [w^{-1}\vx]^T \vf
    \label{eq:G-action-tensor-T0}
\end{align}
Dropping the position $\vx_\mu$, the transformed features are matrix product $w\cdot \vf= w^{-1T}\vf$.

\paragraph{G-conv and L-conv in tensor notation} 
Writing G-conv \eqref{eq:G-conv} in the tensor notation we have 
\begin{align}
    [\kappa\star f](g\vx_0) = \int_G \kappa(v)f(gv\vx_0) dv = \vx_0^T \int_G v^{T} g^{T} \vf \kappa^T(v) dv 
    \equiv \vx_0^T [\vf \star \kappa](g) 
    \label{eq:G-conv-tensor0}
\end{align}
where we moved $\kappa^T(v) \in \R^m \otimes \R^{m'}$ to the right of $\vf$ because it acts as a matrix on the output index of $\vf$.
The equivariance of \eqref{eq:G-conv-tensor0} is readily checked with $w\in G$
\begin{align}
    w\cdot [\vf \star \kappa](g) & = [\vf \star \kappa](w^{-1}g) 
    % = \int_G v^{T} [w^{-1}g]^{T} \vf \kappa^T(v) dv %\cr &
    = \int_G v^{T} g^T w^{-1T} \vf \kappa^T(v) dv %\cr &
    %= \int_G v^{-1} g^{-1} w\vf \kappa^T(v) dv 
    % = [w^{-1T} \vf \star \kappa](g)
    = [(w\cdot\vf) \star \kappa](g)
\end{align}
where we used $[w^{-1}g]^{T}\vf = g^{T} w^{-1T}\vf $.
Similarly, we can rewrite L-conv \eqref{eq:L-conv} in the tensor notation. 
Defining $v_\eps =I+ \ba{\eps}^i L_i $ 
\begin{align}
    Q[\vf](g) 
    &= W^0 f\pa{g\pa{I+ \ba{\eps}^i L_i}} 
    = \vx_0^T \pa{I+ \ba{\eps}^i L_i}^{T} g^{T} \vf W^{0T} \cr
    % &= \vx_0^T \pa{I+ \ba{\eps}^i L_i^T} g^{T} \vf W^{0T} \cr 
    &= \pa{\vx_0+ \ba{\eps}^i [gL_i\vx_0]}^T \vf W^{0T}.
    % = \br{\vf_\mu + \ba{\eps}^i [gL_i \vf]_\mu }W^{0T} 
    % \cr
    \label{eq:L-conv-tensor0}
\end{align}
Here, $\hat{L}_i=gL_i\vx_0$ is eactly the tensor analogue of pushforward vector field $\hat{L}_i$ in \eqref{eq:dfdg-general}. 
We will make this analogy more precise below.
The equivariance of L-conv in tensor notation is again evident from the $g^T \vf$, resulting in 
\begin{align}
    Q[w\cdot \vf](g) &= \vx_0^Tv_\eps^T g^T w^{-1T} \vf W^{0T}= Q[\vf](w^{-1}g) = w\cdot Q[\vf](g)
    \label{eq:L-conv-equiv-tensor0}
\end{align}

Next, we will discuss how to implement \eqref{eq:L-conv-tensor0} in practice and how to learn symmetries with L-conv.
% means and how it impacts the learning of symmetries using L-conv. 
We will also discuss the relation between L-conv and other neural architectures.

\subsection{Constraints from topology on tensor L-conv \label{ap:L-conv-tensor-interpret}
}
To implement \eqref{eq:L-conv-tensor0} we need to specify the lift and the form of $L_i$. 
We will now discuss the mathematical details leading to and easy to implement form of \eqref{eq:L-conv-tensor0}. 

\paragraph{Topology}
Although he discrete space $\mathcal{S}$ is a set of points, in many cases it has a topology.
For instance, $\mathcal{S}$ can be a discretization of a manifold $\mathcal{S}_0$, or vertices on a lattice or a general graph. 
% If the points in $\mathcal{S}$ is a graph, a lattice, or the discretization of a manifold $\mathcal{S}_0$, it has a topology. 
We encode this topology in an undirected graph (i.e. 1-simplex) with vertex set $\mathcal{S}$ edge set $\mathcal{E}$.
Instead of the commonly used graph adjacency matrix $\mA$, we will use the incidence matrix $\mB:\mathcal{S}\times  \mathcal{E}\to \{0,1,-1\}$.
$\mB^\mu_\alpha=1$ or $-1$ if edge $\alpha$ starts or ends at node $\mu$, respectively, and $\mB^\mu_\alpha=0$ otherwise (undirected graphs have pairs of incoming and outgoing edges ).
Similar to the continuous case we will denote the topological space $(\mathcal{S},\mathcal{E},\mB)$ simply by $\mathcal{S}$. 
\out{
The Laplacian $\mL = \mD-\mA$ becomes $\mL = \mB\mB^T$.
A weighted graph can be encoded using a diagonal weight matrix $\mW_{ab}$ for the edges, with the weighted Laplacian $\mL = \mB \mW\mB^T$.
}%

Figure \ref{fig:S0S} summarizes some of the aspects of the discretization as well as analogies between $\mathcal{S}_0$ and $\mathcal{S}$. 
Technically, the group $G_0$ acting on $\mathcal{S}_0$ and $G$ acting on $\mathcal{S}$ are different.
But we can find a group $G$ which closely approximates $G_0$ (see SI \ref{ap:approx}).
% For example, \citet{rao1999learning} used the Shannon-Whittaker Interpolation \citep{whitaker1915functions} to translate discrete 1D signals by arbitrary, continuous amounts as $f'_\mu = g(z)_\mu^\nu f_\nu $.
% Here $g(z)_\mu^\nu = {1\over d} \sum_{p=-d/2}^{d/2} \cos \pa{{2\pi p\over d}(z+\mu -\nu) } $ approximates the shift operator for continuous $z$.
% The $g(z)$ form a group because $g(w)g(z)=q(w+z)$, which is a representation for periodic 1D shifts. 
For instance, \citet{rao1999learning} used Shannon-Whittaker interpolation theorem \citep{whitaker1915functions} to define continuous 1D translation and 2D rotation groups on discrete data. 
We return to this when expressing CNN as L-conv in \ref{sec:approx}.

\paragraph{Neighborhoods as discrete tangent bundle}
$\mB$ is useful for extending differential geometry to graphs \citep{schaub2020random}. 
% The action $[\mB]_\alpha \vf = \vf_\mu - \vf_\nu$ is a discrete version of $\ro_\alpha \vf $, where $\mu$ and $\nu$ are the endpoints of edge $\alpha$.
% The graph Laplacian $\mL = \mB\mB^T$ generalizes the continuous Euclidean Laplace operator $\del^2 = \ro^\alpha \ro_\alpha $. 
% A discrete version of tangent spaces can be defined using neighborhoods of nodes.
Define the neighborhood $<\mu>=\left\{ \alpha \in \mathcal{E} \big|\mB_\alpha^\mu =-1 \right\}$ of $\vx_\mu$ as the set of outgoing edges. 
\out{
\nd{Should we try to identify $<\mu>$ with a set of vectors instead? Like $\vx_\nu -\vx_\mu$? How can we talk about $L_i\vx_0$ being in $<0>$? 
Or should we think of it as an operator? 
Which acts on a vector and yields a number? 
That is true of the transpose, as $\vx_0^T L_i^T \vf \in \mathcal{F}$. We say that $C \in <\mu>$ if $C\vf = \sum_{\alpha\in <\mu>} c^\alpha \mB_\alpha^\nu \vf_\nu  $
}
}%%%%
$<\mu>$ can be identified with $T_{\vx_\mu} \mathcal{S}$ as for $\alpha \in <\mu>$, $\mB_\alpha \vf = \vf_\mu - \vf_\nu \sim \ro_\alpha \vf $, where $\mu$ and $\nu$ are the endpoints of edge $\alpha$.
This relation becomes exact when $\mathcal{S}$ is an $n$D square lattice with infinitesimal lattice spacing.  
The set of all neighborhoods is $\mB$ itself and encodes the approximate tangent bundle $T\mathcal{S}$.
For some operator $C:\mathcal{S}\otimes\mathcal{F} \to \mathcal{F} $ acting on $\vf $ we will say $C \in <\mu>$ if its action remains within vertices connected to $\mu$, meaning 
\begin{align}
    C \in <\mu>: \qquad C\vf = \sum_{\alpha\in <\mu>} \tilde{C}^\alpha \mB_\alpha^\nu \vf_\nu 
\end{align}
% \nd{A bit of abuse of notation... but consistent with the vector field notation $C = C^\alpha \ro_\alpha$}

\paragraph{Lift and group action}
% Define the lift via $\vx_\mu = g_\mu \vx_0$. 
% Since we chose $\vx_\mu$ to be the canonical basis (one-hot), $\vx_\nu^T\vx_\mu = \delta_\mu^\nu $. 
% Thus, we can pick $g_\mu$ such that $g_\mu^T= g_\mu^{-1}$, meaning orthogonal matrices. 
Lie algebra elements $L_i$ by definition take the origin $\vx_0$ to points close to it.
Thus, for small enough $\eta$, 
$(I+\eta L_i) \vx_0 \in <0> $ and so $[L_i  \vx_0]^T\vf=[\hat{L}_i]_0^\rho \vf_\rho = \sum_{\alpha \in <0>} 
% \Gamma_{i 0 }^\rho 
[\hat{\ell}_i]_0^\alpha \mB^\rho_\alpha \vf_\rho $. 
%for a set of 
The coefficients $[\hat{\ell}_i]_0^\alpha \in \R$ are in fact the discrete version of $\hat{L}_i$ components from \eqref{eq:dfdg-general}.  
% Define the lift via $\vx_\mu = g_\mu \vx_0$. 
For the pushforward $gL_ig^{-1}$, we define the lift via $\vx_\mu = g_\mu \vx_0$.
We require the $G$-action to preserve the topology of $\mathcal{S}$, meaning points which are close remain close after the $G$-action. 
As a result, $<\mu>$ can be reached by pushing forward elements in $<0>$.
Thus, for each $i$, $\exists \eta \ll 1 $ such that $g_\mu (I+\eta L_i) \vx_0 \in<\mu>$, meaning for a set of coefficients $ [\hat{L}_i]_\mu^\nu \in \R$ 
%$ \Gamma_{i \mu }^\rho \in \R$ 
we have
\begin{align}
    [g_\mu (I+\eta L_i) \vx_0]^T\vf & = \vf_\mu+ \eta \sum_{\alpha\in <\mu>}  
    [\hat{\ell}_i]_\mu^\alpha \mB_\alpha ^\nu\vf_\nu 
    %= g_\mu \pa{\vx_0+ \eps \sum_{\rho\in <0>} \mA_{0\rho} \Gamma_{i 0 }^\rho \vx_\rho}
    % \cr
    % &&
    % \Gamma_{i \mu }^\nu &= \sum_{\rho\in <0>} \Gamma_{i 0 }^\rho \vx_\nu^T g_\mu\vx_\rho = [g_\mu]^\nu_\rho \Gamma_{i 0 }^\rho
\end{align}
where $\vf_\mu = \vx_\mu^T\vf $.
% Multiplying from the left by $\vx_\nu^T$ to obtain $\vx^T_\nu g_\mu L_i\vx_0 = [g_\mu L_i\vx_0]^\nu $ and 
Acting with $[g_\mu L_i \vx_0]^T\vx_\nu $ and 
inserting $I= \sum_\rho \vx_\rho \vx_\rho^T$ we have 
\begin{align}
    [\hat{L}_i]_\mu^\nu &= 
    % \Gamma_{i \mu }^\nu 
    [\hat{\ell}_i]_\mu^\alpha \mB_\alpha^\nu  
    = [g_\mu L_i\vx_0]^\nu =\sum_{\rho} [g_\mu\vx_\rho \vx^T_\rho L_i \vx_0]^T\vx_\nu
    \cr & = \sum_{\rho\in <0>}  [L_i]_0^\rho [\vx_\nu^T g_\mu\vx_\rho]^T
    = [L_i]_0^\rho [g_\mu]^\nu_\rho = [\hat{\ell}_i]_0^\alpha \mB_\alpha^\rho [g_\mu]^\nu_\rho
    \label{eq:hat-L-discrete0}
\end{align}
This $\hat{L}_i \equiv \hat{\ell}_i^\alpha \mB_\alpha $ is the discrete $\mathcal{S}$ version of the vector field $\hat{L}_i(\vx) = [g L_i \vx_0]^\alpha \ro_\alpha $ in \eqref{eq:dfdg-general}.

\out{
\paragraph{Discrete vector field and tangent bundle}
We can show (SI \ref{ap:L-conv-tensor-interpret}) that when $\mathcal{S}$ has a topology $\hat{L}_i= g L_i\vx_0 $ is indeed the discrete version of the $\hat{L}_i$ in \eqref{eq:dfdg-general}. 
The topology, to lowest order, using a graph (1-simplex) is encoded as an
% Instead of the commonly used graph adjacency matrix $\mA$, we will use 
incidence matrix $\mB:\mathcal{S}\times  \mathcal{E}\to \{0,1,-1\}$
with vertices $\mathcal{S}$ and edge set $\mathcal{E}$. 
$\mB^\mu_\alpha=1$ or $-1$ if edge $\alpha$ starts or ends at node $\mu$, respectively, and $\mB^\mu_\alpha=0$ otherwise. % (undirected graphs have pairs of incoming and outgoing edges ).
% Similar to the continuous case we will denote the topological space $(\mathcal{S},\mathcal{E},\mB)$ simply by $\mathcal{S}$.   
$\mB$ is used to extend differential geometry to graphs \citep{schaub2020random}, since for $\alpha $ an outgoing edge from vertex $\mu$, we have 
$\mB_\alpha \vf = \vf_\mu - \vf_\nu \sim \ro_\alpha \vf $. 
We show that for $G$ to preserve the topology, components of $\hat{L}_i$ should act on the local neighborhood of each node and related to each other via the pushforward
\begin{align}
    [\hat{L}_i]_\mu^\nu &= [g_\mu L_i\vx_0]^\nu = [\hat{\ell}_i]_\mu^\alpha \mB_\alpha^\nu 
    = [\hat{\ell}_i]_0^\alpha \mB_\alpha^\rho [g_\mu]^\nu_\rho.
    \label{eq:hat-L-discrete}
\end{align}
$\hat{L}_i \equiv \hat{\ell}_i^\alpha \mB_\alpha $ is the discrete $\mathcal{S}$ version of the vector field $\hat{L}_i(\vx) = [g L_i \vx_0]^\alpha \ro_\alpha $ in \eqref{eq:dfdg-general}.
Here $ [\hat{\ell}_i]_0^\alpha \in \R $ are weightings for edges $\alpha $ connected to $\vx_0$ and capture the action of $L_i$. 
This weighting is similar to Gauge Equivariant Mesh (GEM) CNN \citep{cohen2019gauge}. 
The main difference is that in GEM-CNN the basis of the tangent spaces $T_{\vx_\mu}\mathcal{S}$ are not fixed and are chosen by a gauge. 
In L-conv the lift $\vx_\mu = g_\mu\vx_0$ fixes the gauge by mapping neighbors of $\vx_0$ to neighbors of $\vx_\mu$. 
% , meaning it determines how the basis at $<0>$ maps to the basis at $<\mu>$. 
Changing how the discrete $\mathcal{S}$ samples an underlying continuous space 
% $\mathcal{S}_0$ 
% is unchanged but the points sampled from it to create $\mathcal{S}$ are moved a bit, 
changes $g_\mu$ and hence the gauge, similar to GEM-CNN. 
% have the form $[\hat{L}_i]_\mu^\nu = [\hat{\ell}_i]_\mu^\alpha \mB_\alpha^\nu $. 
% Additionally, we show that the lift $g_\mu$ relates components of $\hat{L}_i$ via the pushforward
% $[\hat{\ell}_i]_\mu^\alpha \mB_\alpha^\nu = [\hat{\ell}_i]_0^\alpha \mB_\alpha^\rho [g_\mu]^\nu_\rho$. 
% Thus, to implement 
% means and how it impacts the learning of symmetries using L-conv. 
}%%%%%

\subsection{Approximating a symmetry and discretization error \label{ap:approx}}
\paragraph{Discretization error}
While systems such as crystalline solids are discrete in nature, many other datasets such as images result from discretization of continuous data. 
The discretization (or ``coarsening'' \citep{bronstein2021geometric}) will modify the groups that can act on the space. 
For example, first rotating a shape by $SO(2)$ then taking a picture is different from first taking a picture then rotating the picture (i.e. group action and discretization do not commute). 
% the picture of a rotated object, is different from rotating the
% rotating an object by $SO(2)$, then taking a picture is different from first taking a picture and then rotating the image by an image rotation group $G$. 
Nevertheless, in most cases in physics and machine learning the symmetry group $G_0$ of the space before discretization has a small Lie algebra dimension $n$, (e.g. $SO(3), SE(3), SO(3,1)$ etc). 
Usually the resolution of the discretization is $ d \gg n$. 
In this case, there always exist some $G\subseteq \mathrm{GL}_d(\R)$ which approximates $G_0$ reasonably well. 
The approximation means $\forall g_0 \in G_0, \exists g\in G $ such that the error $\mathcal{L}_G=\|g_0\cdot f(\vx_\mu) -\vx_\mu^T g \vf\|^2<\eta^2 $ where $\eta$ depends on the resolution of the discretization.
Minimizing the error $\mathcal{L}_G$ can be the process of identifying the $G$ which best approximates $G_0$. 
We will denote this approximate similarity as $G \simeq G_0$.
For example, \citet{rao1999learning} used the Shannon-Whittaker Interpolation theorem \citep{whitaker1915functions} to translate discrete 1D signals (features) by arbitrary, continuous amounts. 
In this case the transformed features are $\vf'_\mu = g(z)_\mu^\nu \vf_\nu $, where $g(z)_\mu^\nu = {1\over d} \sum_{p=-d/2}^{d/2} \cos \pa{{2\pi p\over d}(z+\mu -\nu) } $ approximates the shift operator for continuous $z$.
The $g(z)$ form a group because $g(w)g(z)=g(w+z)$, which is a representation for periodic 1D shifts. 
\citet{rao1999learning} also use a 2D version of the interpolation theorem to approximate $SO(2)$. 
% \nd{Rao's examples of groups approximating shift and $SO(2)$}
% We show below how the Lie algebra of $G$ and $G_0$ are related. 
In practice, we can assume the true symmetry to be $G$, as we only have access to the discretized data and can't measure $G_0$ directly. 
% When learning symmetries, since we only have access to the data via the discretized dataset, we assume the symmetry group to be $G\subseteq \mathrm{GL}_d(\R)$.

% \ry{how do you compute this integral?}
% Thus, to learn the parameters $\ba{\eps}_i$  we can apply the operators $\ba{L}= \br{I, L_1, \cdots L_{n_L}}$ to the input $h$ and apply a fully connected layer to the $ 1+ n_L$ channels of the output to learn the shared parameters. 
% \end{proof}
% This leads to a simple formulation for Lie algebra convolution akin to graph convolutional networks \citep{kipf2016semi}.

% \section{Relation to other architectures}
% Symmetries induce constraints on weights \cite{ravanbakhsh2017equivariance}.

\out{

\paragraph{Discretization error}
While systems such as crystalline solids are discrete in nature, many other datasets such as images result from discretizion of continuous data. 
The discretization (or ``coarsening'' \citep{bronstein2021geometric}) will affect the group action. 
For example, rotating an object by $SO(2)$, then taking a picture is different from first taking a picture and then rotating the image by an image rotation group $G$. 
Nevertheless, we expect that there exist $G\subseteq \mathrm{GL}_d(\R)$ which closely approximates rotations, or any group $G_0$ that acts on the continuous space. 
The approximation means $\forall g_0 \in G_0, \exists g\in G $ such that, $\|g_0\cdot f(\vx_\mu) -\vx_\mu^T g \vf\|<\eta $ where $\eta$ depends on the resolution of the discretization (the group action $\vx_\mu^T g \vf$ is explained below).
We will denote this approximate similarity as $G \simeq G_0$.
For example, \citet{rao1999learning} used the Shannon-Whittaker Interpolation theorem \citep{whitaker1915functions} to translate discrete 1D signals (features) by arbitrary, continuous amounts. 
In this case the transformed features are $f'_\mu = q(z)_\mu^\nu f_\nu $, where $q(z)_\mu^\nu = {1\over d} \sum_{p=-d/2}^{d/2} \cos \pa{{2\pi p\over d}(z+\mu -\nu) } $ approximates the shift operator for continuous $z$.
The $q(z)$ form a group because $q(w)q(z)=q(w+z)$, which is a representation for periodic 1D shifts. 
\citet{rao1999learning} also use a 2D version of the interpolation theorem to approximate $SO(2)$. 
% \nd{Rao's examples of groups approximating shift and $SO(2)$}
We show below how the Lie algebra of $G$ and $G_0$ are related. 
When learning symmetries, since we only have access to the data via the discretized dataset, we assume the symmetry group to be $G\subseteq \mathrm{GL}_d(\R)$.
}

\out{
\paragraph{Lift}
We want the lift $\vx_\mu = g_\mu \vx_0$ to preserve the topology $(\mathcal{S},\mA)$.
% , it is important to note that $g_\mu \in G$ restricts the form of $g_\mu$. 
% For instance, 
while it is trivial to find permutations taking $\vx_0$ to $\vx_\mu$, most of them won't preserve all neighborhoods $<\mu>$.
% be an element of a given $G$. 
For instance, if $G$ is approximating 1D translations the lift must move adjacent points together. 
Adjacency, captured in $\mA$ is also found by moving points in $S$ using infinitesimal group elements $v_\eps = I+\eps^iL_i$. 
Since 1D shifts have a single Lie algebra basis $L$, 
there's only one way to define $g_\mu$, namely shifting all points by index $\mu$. 
This can be achieved using circulant matrices.
Hence, the lift for periodic 1D translations $G\simeq S^1$ is 
% \nd{add higher windings}
\begin{align}
    % \mbox{for 1D Translation}&&
    g_\mu & = \sum_{\nu = 0}^{d } \vx_{\mu + \nu } \vx_{\nu}^T \quad (\mu + \nu\  \mathrm{mod}\  d) 
    % \label{eq:gmu-circulant}
    &
    [g_\mu]^\rho_\lambda &=\sum_{\nu} \delta_{\mu+\nu, \lambda} \delta^\rho_\nu %= \delta_{\mu + \rho, \lambda}
    = \delta^\rho_{\lambda-\mu}
    \label{eq:gmu-circ-comp0}
\end{align}
With $ g_\mu^{-1}=g_\mu^T$. 
The uniqueness of $g_\mu$ here was because there was a single $L_i$, which generates a unique flow on $\mathcal{S}$. 
For groups with multiple $L_i$, using the path-ordered notation $g_\mu = P\exp[\int_\gamma dt^i L_i]$ we see that at every step along a path $\gamma $ a different $L_i$ step is taken. 
For non-Abelian groups where $[L_i,L_j]\ne 0$, the order of these steps matters and
% What this means is that 
the lift is not uniquely defined.
% for non-Abelian groups. 
% Note that in the tensor notation, $\mathcal{S}$ represents a whole manifold and $g_\mu$ is not moving just one point, but entire neighborhoods.
% Moreover, $g_\mu \in G$ can act on any $\vx_\nu$ and it should again preserve the topology $(\mathcal{S},\mA)$. 
% The Lie algebra basis $L_i$ only moves points $\vx_\nu$ within $<\nu>$. 
In fact, as we see later, every choice of the lift defines a ``gauge'' on the tangent bundle $T\mathcal{S}$. 
Gauges have been used in Gauge Equivariant Mesh (GEM) CNN \citet{cohen2019gauge} (see \citet{bronstein2021geometric} for a review) and
% The continuous version, called gauge fields, is 
used extensively in physics
\citep{polyakov2018gauge}.
% The set of lifts $g_\mu $ also define the pushforward $g_\mu L_i$.

\nd{Equivariance of L-conv}

Since in connected Lie groups, larger group elements can be constructed as $u=\prod_\alpha \exp[t_\alpha \cdot L]$ \citep{hall2015lie} from elements near identity. It follows that any G-conv layer can be constructed from multiple L-conv layers .

}%%%%
\out{
\subsection{Fully-connected as G-conv and L-conv \label{ap:FC2L-conv} }
\nd{remove}

Consider the case where the topology of $\mathcal{S}$ is a complete graph. 
In this case, the neighborhood of each $\vx_\mu$ includes all other vertices. 
Therefore, the topology of $\mathcal{S}$ puts no constraint on the Lie algebra of $G$ and any linear transformation in $G=\mathrm{GL}_d(\R)$ preserves the topology of $\mathcal{S}$. 
The Lie algebra $L_i\mathrm{gl}_d(\R)$ has $d^2$ elements, one for each matrix entry. 
Thus, we have also have $d^2$ weights $W^i = W^0\ba{\eps}^i$. 
The L-conv output \eqref{eq:L-conv-tensor} $Q[\vf] %v_\eps\cdot\vf W^0
=\vf W^0 + \hat{L}_i\vf W^i$ has $d$ spatial dimensions $\mu$ and $d^2\times m'$ filter dimensions $i,a$. 
Pooling over the spatial index $\mu$, we get $\sum_\mu (\vf_\mu^a+[\ba{\eps}^i]^a_b [\hat{L}_i]_\mu^\nu \vf^b_\nu) = V^{\nu a}_b \vf_\nu^b $.
Here $V$ is a general linear transformation on $\vf$, similar to one linear perceptron. 
Multiple L-conv layers, both parallel and in series with residuals, can emulate wider fully-connected networks. 
These are equivariant under $w\in \mathrm{GL}_d(\R) $, as we elaborate now. 
The idea is that in FC layers $ w\cdot [V \vf] = Vw^{-1T}\vf = V'\vf$ captures their equivariance. 
Recall that $[\hat{L}_i]_\mu^\nu = [g_\mu L_i \vx_0]^\nu$.
For simplicity, let L-conv act only on the $\mu$, meaning $[\ba{\eps}_i]^a_b = \eps_i \delta^a_b $, and $W^0=I$
The action \eqref{eq:L-conv-tensor} becomes
\begin{align}
    Q[\vf]_\mu = Q[\vf](g_\mu) = \vx_0^T v_\eps^Tg_\mu^T \vf
= [(I+\ba{\eps}^i\hat{L}_i)\vf]_\mu .
\end{align}
% $Q[\vf] = v_\eps \cdot \vf W^0= v_\eps^T \vf W^0 
% = (I+\ba{\eps}^i\hat{L}_i)\vf W^0 $. 
% $[v_\eps \cdot \vf]_\mu^a= [v_\eps^T]^a_b \vf^b 
% = (I+[\ba{\eps}^i]^a_b[\hat{L}_i]^\nu_\mu)\vf^b_\nu $. 
Its equivariance under $w\in \mathrm{GL}_d(\R)$ follows \eqref{eq:L-conv-equiv-tensor} as 
\begin{align}
    Q[w\cdot \vf](g_\mu) & = \vx_0^T v_\eps^Tg_\mu^T w^{-1T}\vf = Q[ \vf](w^{-1} g_\mu)\cr
    &= [(I+\ba{\eps}^i \hat{L}_i)w^{-1T}\vf]_\mu .
\end{align}
Combining multiple layers of L-conv, we get $v\vf =v_{\eps_n}\dots v_{\eps_1}\vf \approx \exp[t^i \hat{L}_i]\vf $, where $v\in \mathrm{GL}_d(\R)$.

When there is no restriction on the weights, meaning $G=\mathrm{GL}_d(\R)$, L-conv becomes a fully-connected layer, as shown in the following proposition. 
% Additionally, we note that if we do not  using all generators of the full $GL(d,\R)$ is equivalent to a fully-connected layer. 
% Since there is no constraint on elements of $\mathrm{GL}_d(\R)$ other than being invertible, a basis for its Lie algebra $\mathfrak{gl}_d(\R)$ is all $d^2$ one-hot matrices with single 1 entry and zero elsewhere.

\out{
\begin{proposition}\thlabel{prop:FC}
    A fully-connected neural network layer can be written as an L-conv layer using $\mathrm{GL}_d(\R)$ generators, followed by a sum pooling and nonlinear activation. 
\end{proposition}
\begin{proof}
The generators of $\mathrm{GL}_d(\R)$ are one-hot $\mE \in \R^{d\times d}$ matrices $L_i = \mE_{(\alpha, \beta)}$ which are non-zero only at index $i=(\alpha, \beta)$ \footnote{We may also label them by single index like $i = \alpha + \beta d$, but two indices is more convenient.}
%$(\alpha, \beta)$ as $L_{(\alpha, \beta)} $ 
with elements written using Kronecker deltas
\begin{align}
    \mathrm{GL}_d(\R)\mbox{ generators}: L_{i,\mu}^\nu = [\mE_{(\alpha, \beta)}]_\mu^\nu &= \delta_{\mu\alpha} \delta^{\nu}_\beta
    \label{eq:GL-generators}
\end{align}
Now, consider the weight matrix $w\in \R^{m \times d} $ and bias $b \in \R^m$ of a fully connected layer acting on $h\in \R^d$ as $F(h) = \sigma(w\cdot h+b)$. 
The matrix element can be written as 
\begin{align}
    w_b^{\nu} & = \sum_\mu  \sum_{\alpha,\beta} w_b^{\alpha}\mathbf{1}_\beta [\mE_{(\alpha,\beta)}]^\nu_\mu \cr 
    &= \sum_\mu  \sum_{\alpha,\beta} W_{(\alpha,\beta)}^{b,1} [\mE_{(\alpha,\beta)}]^\nu_\mu 
    = \sum_\mu W^{b,1} \cdot [L]^\nu_\mu 
    % \cr w_{(\alpha,\beta)}^{a,1}
\end{align}
L-conv with weights $W_{(\alpha,\beta)}^{b,1} = w_b^\alpha \mathbf{1}_\beta$ (1 input channel, and $\mathbf{1}$ being a vector of ones) %using generators of $\mathrm{GL}_d(\R)$. 
followed by pooling over $\mu$ is the same as a fully connected layer with weights $w$.
\end{proof}

% \nd{Can we understand what this means? We say if the layer is allowed to incorporate the \textbf{highest amount of symmetry}, we recover FC, whereas restricted symmetries lead to G-conv. This sound a bit counterintuitive at first, but I think it is correct.}
% The takeaway of Proposition \ref{prop:FC} is that 
% To reiterate, on the one extreme where there is no restriction on the symmetry group and hence $G= GL(d,\R)$, L-conv becomes a fully connected layer after aggregating over the spatial output index. 
% Thus, interestingly, more restricted groups, rather than large, meaning groups with fewer $L_i$, lead to more parameter sharing. 
}%%%%%%

}%%%%%%

\subsection{Comparison with other symmetry discovery methods \label{ap:symm-disc-lit} }

\paragraph{Meta-learning Symmetries by Reparameterization}
Recently \citet{zhou2020meta} also introduced an architecture which can learn equivariances from data. 
We would like to highlight the differences between their approach and ours, specifically Proposition 1 in \citet{zhou2020meta}.  
Assuming a discrete group $G=\{g_1,\dots, g_n\}$, they decompose the weights $W\in \R^{s\times s}$ of a fully-connected layer, acting on $\vx \in \R^s$ as $\mathrm{vec}(W) = U^Gv$ where $U^G\in \R^{s\times s}$ are the ``symmetry matrices'' and $v\in \R^s$ are the ``filter weights''. 
Then they use meta-learning to learn $U^G$ and during the main training keep $U^G$ fixed and only learn $v$.
We may compare MSR to our approach by setting $d=s$. 
First, note that although the dimensionality of $U\in \R^{nd\times d}$ seems similar to our $L \in \R^{n\times d\times d}$, the $L_i$ are $n$ matrices of shape $d\times d$, whereas $U$ has shape $(nd) \times d$ with many more parameters than $L$. 
Also, the weights of L-conv $W\in \R^{n\times m_l \times m_{l-1}}$, with $m_l$ being the number of channels, are generally much fewer than MSR filters $v\in \R^d$. 
Finally, the way in which $Uv$ acts on data is different from L-conv, as the dimensions reveal. 
The prohibitively high dimensionality of $U$ requires MSR to adopt a sparse-coding scheme, mainly Kronecker decomposition.
Though not necessary, we too choose to use a sparse format for $L_i$, finding that very low-rank $L_i$ often perform best. 
A Kronecker decomposition may bias the structure of $U^G$ as it introduces a block structure into it.

\paragraph{Augerino}
In a concurrent work,
\citep{benton2020learning} propose Augerino, a method to learn equivariance with neural networks, but restricted to
% Augerino uses data augmentation to transform the input data, which means it is restricting the group to be  
a subgroup of the augmentation transformations. 
Augerino learns which subset of the augmentations improved the prediction. 
This is done by writing 
The data augmentation is written as $g_\eps = \exp\pa{\sum_i \eps_i \theta_i L_i} $ (equation (9) in \cite{benton2020learning}), with randomly sampled $\eps_i\in [-1,1]$. 
$\theta_i$ are trainable weights which determine which $L_i$ helped with the learning task. 
Furthermore, 
Their Lie algebra is fixed to affine transformations in 2D (translations, rotations, scaling and shearing). 
% In contrast, 
Our approach is more general.  
We learn the $L_i$ directly without restricting to known symmetries. % them to be a known set of generators.  
Additionally, we do not use the exponential map or matrix logarithm, hence, our method is easy to implement.  
Lastly, Augerino uses sampling to effectively cover the space of group transformations. 
Since we work with the Lie algebra rather the group itself, we do not require sampling.

\out{
\nd{Add our own experiments here.}
\paragraph{Symmetry Discovery Literature}
In addition to simplifying the construction of equivariant architectures, our method can also learn the symmetry generators from data. 
Learning symmetries from data has been studied before, but mostly in restricted settings.
Examples include commutative Lie groups as in \citet{cohen2014learning}, 2D rotations and translations in 
\citet{rao1999learning}, \citet{sohl2010unsupervised} or permutations \citep{anselmi2019symmetry}. 
% Perhaps the closest to our work, in spirit, is 
\citet{zhou2020meta} uses meta-learning to automatically learn symmetries in the data. 
% \citet{zhou2020meta} is about defining a weight-sharing scheme to decompose the weights of a fully-connected layer and not group convolution as in our work. 
Yet their weight-sharing scheme and the encoding of the symmetry generators is very different from ours. % (see appendix \ref{ap:compare} for detailed discussion). 
% \ry{do we compare with them?} 
% Also their use of Kronecker factorization to sparsify the symmetry matrix restricts the symmetries to be local and, thus, introduces bias.
% Assuming a discrete group $G=\{g_1,\dots, g_n\}$, they decompose the weights $W\in \R^{s\times s}$ of a fully-connected layer, acting on $\vx \in \R^s$ as $\mathrm{vec}(W) = U^Gv$ where $U^G\in \R^{s\times s}$ are the ``symmetry matrices'' and $v\in \R^s$ are the ``filter weights''. 
% Then they use meta-learning to learn $U^G$ and during the main training keep $U^G$ fixed and only learn $v$.
% We may compare MSR to our approach by setting $d=s$. 
% First, note that although the dimensionality of $U\in \R^{nd\times d}$ seems similar to our $L \in \R^{n\times d\times d}$, the $L_i$ are $n$ matrices of shape $d\times d$, whereas $U$ has shape $(nd) \times d$ with many more parameters than $L$. 
% Also, the weights of L-conv $W\in \R^{n\times m_l \times m_{l-1}}$, with $m_l$ being the number of channels, are generally much fewer than MSR filters $v\in \R^d$. 
% Finally, the way in which $Uv$ acts on data is different from L-conv, as the dimensions reveal.
% The prohibitively high dimensionality of $U$ requires MSR to adopt a sparse-coding scheme, mainly Kronecker decomposition.
% Though not necessary, we too choose to use a sparse format for $L_i$, finding that very low-rank $L_i$ often perform best. 
% A Kronecker decomposition may bias the structure of $U^G$ as it introduces a block structure into it.
% In a concurrent work,
\citep{benton2020learning} propose Augerino, a method to learn equivariance with neural networks, but restricted to
% Augerino uses data augmentation to transform the input data, which means it is restricting the group to be  
a subgroup of the augmentation transformations. 
% Augerino learns which subset of the augmentations improved the prediction. 
% This is done by writing 
% The data augmentation is written as $g_\eps = \exp\pa{\sum_i \eps_i \theta_i L_i} $ (equation (9) in \cite{benton2020learning}), with randomly sampled $\eps_i\in [-1,1]$. 
% $\theta_i$ are trainable weights which determine which $L_i$ helped with the learning task. 
% Furthermore, 
Their Lie algebra is fixed to affine transformations in 2D (translations, rotations, scaling and shearing). 
% In contrast, 
Our approach is more general.  
We learn the $L_i$ directly without restricting to known symmetries. % them to be a known set of generators.  
Additionally, we do not use the exponential map or matrix logarithm, hence, our method is easy to implement.  
Lastly, Augerino uses sampling to effectively cover the space of group transformations. 
Since we work with the Lie algebra rather the group itself, we do not require sampling.
%and only need to learn a basis for the Lie algebra, which generally a small number of Lie algebra elements.
% The main differences with our approach 
% , which mean they are assuming the group to be a subgroup of their augmentation transformations. 
% Specifically, they assume the group to be affine transformations and learn which sub-group of it is used in the dataset. 
% In eq. 9, the $G_i$ (equivalent to our $L_i$) are the generators of affine transformations and the state below eq. 9 the $G_1,...G_6$ are predefined translations, rotations, scaling and shearing.  
% In contrast, we are learning $L_i$ and do not restrict them to a known set of generators. 

Next, we will provide some examples.
% implementation details and relations with physics.
% implementation details of L-conv and how it relates to other architectures such as CNN \citep{lecun1989backpropagation} and graph convolutional networks (GCN) \citep{kipf2016semi}.  
}
\out{
\subsection{Interpretation and Examples \label{sec:interpret-continuous} }
% \paragraph{General case} 
The $gL_i\cdot df/dg$ in \eqref{eq:L-conv} can be written in terms of partial derivatives $\ro_\alpha f(\vx) = \ro f/\ro \vx^\alpha$. 
In general, using $\vx^\rho = g^\rho_\sigma \vx_0^\sigma$, we have ${df(g\vx_0)\over dg^\alpha_\beta} =\vx_0^\beta \ro_\alpha f(\vx) $, and so
% \ry{add a cartoon}
\begin{align}
    % \vx^\mu &= g^\rho_\sigma \vx_0^\sigma &
    % {df(g\vx_0)\over dg^\alpha_\beta}
    % &= {d(g^\rho_\sigma \vx_0^\sigma) \over dg^\alpha_\beta}\ro_\rho f(\vx) = \vx_0^\beta \ro_\alpha f(\vx)
    % \label{eq:dgx0-dg}
    % \\
    gL_i\cdot {df\over dg} &= [gL_i]^\alpha_\beta \vx_0^\beta \ro_\alpha f(\vx) = [gL_i\vx_0]\cdot \del f(\vx) = \hat{L}_if(\vx)
    \label{eq:dfdg-general}
\end{align}
Hence, for each $L_i$, the pushforward $gL_i$ generates a flow on $\mathcal{S}$ through the vector field $\hat{L}_i\equiv gL_i\cdot d/dg = [gL_i\vx_0]^\alpha \ro_\alpha$ (Fig. \ref{fig:Lie-group-S}). 
% $\hat{L}_i$ is related to the Maurer-Cartan form $\omega = g^{-1} \ro_\alpha g d\vx^\alpha$, which encodes the pushforward.
% Write $f(gv_\eps) = f(g+ \eta(\eps)^\alpha \ro_\alpha g)$. 
% We have $ \eta^\alpha \ro_\alpha g = \eps^i g L_i$, so $ \eps^i L_i \eta^\alpha g^{-1} \ro_\alpha g = \eta \cdot \omega $.
% The Maurer-Cartan form defines a way to parallel transport (pushforward) and define a basis frame on  
% Vector fields are sections of the tangent bundle $T\mathcal{S}$.
% We will encounter $gL_i\vx_0$ again on a discretized $T\mathcal{S}$ again below. 
Being a vector field $\hat{L}_i\in T\mathcal{S}$ (i.e. 1-tensor), $\hat{L}_i$ is basis independent, meaning 
for $v \in G$, $\hat{L}_i(v\vx) = \hat{L}_i$. 
Its components transform as $[\hat{L}_i(v \vx)]^\alpha =[vgL_i\vx_0]^\alpha = v^\alpha_\beta \hat{L}_i(\vx)^\beta $, while the partial transforms as $\ro / \ro[v\vx]^\alpha = [v^{-1}]^{\gamma}_\alpha \ro_\gamma$. 

\paragraph{Explicit examples}
Using \eqref{eq:dfdg-general} we can calculate the form of L-conv for specific groups (details in SI \ref{ap:examples-continuous}). 
For \textbf{translations} $G=T_n = (\R^n,+)$ we find the generators become simple partial derivatives $\hat{L}_i = \ro_i$ (SI \ref{ap:example-Tn}), yielding $f(\vx) + \eps^\alpha \ro_\alpha f(\vx)$. 
For \textbf{2D rotations} (SI \ref{ap:example-so2}) 
we find $\hat{L} \equiv \pa{x \ro_y - y \ro_x } = \ro_\theta$, which in physics is called the angular momentum operator about the z-axis in quantum mechanics and field theories. %, which generates rotations around the $z$ axis. 
For rotatons with scaling, 
\textbf{Rotation and scaling} $G= SO(2)\times \R^+$, 
we have two $L_i$, one $\hat{L}_\theta=\ro_\theta $ from $so(2)$ and a scaling with $L_r = I$, yielding
$\hat{L}_r = x\ro_x+y\ro_y= r\ro_r $. % (SI \ref{ap:examples-continuous}) 
}%%%%
% Next, we discuss the form of L-conv on discrete data. 

\section{Experiments \label{ap:experiments}}

We conduct a set of experiments to see how well L-conv can extract infinitesimal 
generators.

\paragraph{Nonlinear activation}
% Regarding nonlinear activation functions, as discussed in \nd{cite prior work}, it depends on whether they commute with the group action. 
% When $G$ is not a Lie group, many point-wise nonlinearities can preserve it. 
% As many nonlinear activation functions are not compatible with equivariance. 
% Note that, similar to discussion in
As noted in \cite{weiler20183d}, an arbitrary nonlinear activation $\sigma $ may not keep the architecture equivariant under $G$. %, but linear activation remains equivariant.
However, as we showed in SI \ref{ap:theory-extended} (\textbf{Extended equivariance for L-conv}), the feature dimensions can pass through any nonlinear neural network without affecting the equivariance of L-conv. 
This means that the weights of the nonlinear layer should act only on $\mathcal{F}$ and not $\mathcal{S}$. 

\out{
In particular, for a fixed $i$ and $\eps\ll 1$, we have 
\begin{align}
    \vx_i &= (I+\eps L_i) \vx_\mu = \vx_\mu+ \eps \sum_{\nu\ne \mu} \mA_{\mu\nu} \Gamma_{i \mu }^\nu \vx_\nu \in <\mu> \\
    [L_i]^\nu_\mu &= \vx_\nu^T L_i \vx_\mu = {1\over \eps }\pa{\vx_\nu^T\vx_i- \delta^\nu_\mu}= \mA_{\mu\nu}\Gamma_{i\mu}^\nu
    % \vx_\nu^T\vx_i &= \delta^\nu_\mu + \eps [L_i]^\nu_\mu = \delta^\nu_\mu + \eps \mA_{\mu\nu}[l_i]^\nu 
\end{align}
Next, we note that writing $g_\mu = P\exp[\int_\gamma dt^iL_i]$ we can use the Lie algebra $[L_i,L_j={c_{ij}}^kL_k$ to pass $L_i$ through

Let $(I+\eps L_i) \vx_0 \in <0> $ for small $\eps \ll 1$. 
The continuity of $G$ action means that for each $i$, $\exists \eps \ll 1 $ such that $g_\mu (I+\eps L_i) \vx_0 \in<\mu>$. 
Therefore
\begin{align}
    (I+\eps L_i) \vx_0& = \vx_0+ \eps \sum_{\rho\ne 0} \mA_{0\rho} \Gamma_{i 0 }^\rho \vx_\rho \cr
    g_\mu (I+\eps L_i) \vx_0& = \vx_\mu+ \eps \sum_{\nu\ne \mu} \mA_{\mu\nu} \Gamma_{i \mu }^\nu \vx_\nu = g_\mu \pa{\vx_0+ \eps \sum_{\rho\ne 0} \mA_{0\rho} \Gamma_{i 0 }^\rho \vx_\rho}\cr
    \mA_{\mu\nu} \Gamma_{i \mu }^\nu   \sum_{\rho\ne 0} \mA_{0\rho} \Gamma_{i 0 }^\rho \vx_\nu^T g_\mu\vx_\rho
\end{align}

This means that any small $v_\eps = I+\eps^iL_i \in T_IG$ must move points $\vx_\mu$ within their neighborhood, meaning for $v_\eps = I+\eps^iL_i \in G$ close to identity and $\vx_\mu \in \mathcal{S}$ we have
\begin{align}
    % \forall v_\eps &= I+\eps^iL_i \in G, \quad \forall\vx_\mu \in \mathcal{S},\quad 
    &\exists c\in \R^d, & v_\eps \vx_\mu &= \sum_\nu c^\nu \mA_{0\nu}\vx_\nu \in <\mu >. 
    \label{eq:G-neigh-preserve}
\end{align}
}%%%

% Next, we establish relations between L-conv and other commonly used architectures. 

\paragraph{Implementation}
The basic way to implement L-conv is as multiple parallel GCN units with aggregation function $f(A)$ being (propagation rule) being $\hat{L}_i$. 
We do not use Deep Graph Library (DGL) or other libraries, as we want to make $\hat{L}_i$ learnable for discovering symmetries. 
A more detailed way to implement L-conv is to encode
$\hat{L}_i$ in the form of
% \eqref{eq:hat-L-discrete}
\eqref{eq:hat-L-discrete0} to ensure that there is an underlying shared generator $\hat{\ell}_i$ for all $\mu$ which is pushed forward using $g_\mu$, shared for all $i$.  
To implement L-conv this way, we need the lift $g_\mu $ and the edge weights $w_i^\alpha = [\hat{\ell}_i]_0^\alpha $. 
% When the graph for the topology of $\mathcal{S}$ is known, L-conv can be implemented starting from a graph convolutional network (GCN) \citep{kipf2016semi}, using
% $\mB$ for the edge list.
% Next, we add edge features $[\hat{\ell}_i]^\alpha_\mu$, which differs from GCN. 
With known symmetries, the learnable parameters are $W^0$ and $\ba{\eps}^i$. 
When the topology and hence $\mB$ is known, we can encode it into the geometry and only learn the edge weights $[\hat{\ell}_i]^\alpha_\mu$, similar to edge features in message passing neural networks (MPNN) \citep{gilmer2017neural}. In general each $\hat{\ell}_i$ has $ |\mathcal{E}|$ (i.e. number of edges) components. 
We can further reduce these using \eqref{eq:hat-L-discrete0}, where instead of $n$ matrices $\hat{\ell}_i$, we learn one $g_\mu$ shared for all, and a small set of elements $[\hat{\ell}_i]_0^\alpha $. 
This is easiest when the graph is a regular lattice and each vertex has the same number of neighbors. 
When the topology of the underlying space is not known (e.g. point cloud or scrambled coordinates), we can learn $\hat{L}_i$ as $d\times d$ matrices. 
We do this for the scrambled image tests, where we encode $\hat{L}_i$ as low-rank matrices.

\paragraph{Symmetry Discovery Literature}
In addition to simplifying the construction of equivariant architectures, our method can also learn the symmetry generators from data. 
Learning symmetries from data has been studied before, but mostly in restricted settings.
Examples include commutative Lie groups as in \citet{cohen2014learning}, 2D rotations and translations in 
\citet{rao1999learning}, \citet{sohl2010unsupervised} or permutations \citep{anselmi2019symmetry}. 
% Perhaps the closest to our work, in spirit, is 
\citet{zhou2020meta} uses meta-learning to automatically learn symmetries in the data. 
% \citet{zhou2020meta} is about defining a weight-sharing scheme to decompose the weights of a fully-connected layer and not group convolution as in our work. 
Yet their weight-sharing scheme and the encoding of the symmetry generators is very different from ours. % (see appendix \ref{ap:compare} for detailed discussion). 
% \ry{do we compare with them?} 
% Also their use of Kronecker factorization to sparsify the symmetry matrix restricts the symmetries to be local and, thus, introduces bias.
% Assuming a discrete group $G=\{g_1,\dots, g_n\}$, they decompose the weights $W\in \R^{s\times s}$ of a fully-connected layer, acting on $\vx \in \R^s$ as $\mathrm{vec}(W) = U^Gv$ where $U^G\in \R^{s\times s}$ are the ``symmetry matrices'' and $v\in \R^s$ are the ``filter weights''. 
% Then they use meta-learning to learn $U^G$ and during the main training keep $U^G$ fixed and only learn $v$.
% We may compare MSR to our approach by setting $d=s$. 
% First, note that although the dimensionality of $U\in \R^{nd\times d}$ seems similar to our $L \in \R^{n\times d\times d}$, the $L_i$ are $n$ matrices of shape $d\times d$, whereas $U$ has shape $(nd) \times d$ with many more parameters than $L$. 
% Also, the weights of L-conv $W\in \R^{n\times m_l \times m_{l-1}}$, with $m_l$ being the number of channels, are generally much fewer than MSR filters $v\in \R^d$. 
% Finally, the way in which $Uv$ acts on data is different from L-conv, as the dimensions reveal.
% The prohibitively high dimensionality of $U$ requires MSR to adopt a sparse-coding scheme, mainly Kronecker decomposition.
% Though not necessary, we too choose to use a sparse format for $L_i$, finding that very low-rank $L_i$ often perform best. 
% A Kronecker decomposition may bias the structure of $U^G$ as it introduces a block structure into it.
% In a concurrent work,
\citep{benton2020learning} propose Augerino, a method to learn equivariance with neural networks, but restricted to
% Augerino uses data augmentation to transform the input data, which means it is restricting the group to be  
a subgroup of the augmentation transformations. 
% Augerino learns which subset of the augmentations improved the prediction. 
% This is done by writing 
% The data augmentation is written as $g_\eps = \exp\pa{\sum_i \eps_i \theta_i L_i} $ (equation (9) in \cite{benton2020learning}), with randomly sampled $\eps_i\in [-1,1]$. 
% $\theta_i$ are trainable weights which determine which $L_i$ helped with the learning task. 
% Furthermore, 
Their Lie algebra is fixed to affine transformations in 2D (translations, rotations, scaling and shearing). 
% In contrast, 
Our approach is more general.  
We learn the $L_i$ directly without restricting to known symmetries. % them to be a known set of generators.  
Additionally, we do not use the exponential map or matrix logarithm, hence, our method is easy to implement.  
Lastly, Augerino uses sampling to effectively cover the space of group transformations. 
Since we work with the Lie algebra rather the group itself, we do not require sampling.
%and only need to learn a basis for the Lie algebra, which generally a small number of Lie algebra elements.
% The main differences with our approach 
% , which mean they are assuming the group to be a subgroup of their augmentation transformations. 
% Specifically, they assume the group to be affine transformations and learn which sub-group of it is used in the dataset. 
% In eq. 9, the $G_i$ (equivalent to our $L_i$) are the generators of affine transformations and the state below eq. 9 the $G_1,...G_6$ are predefined translations, rotations, scaling and shearing.  
% In contrast, we are learning $L_i$ and do not restrict them to a known set of generators. 

% \subsection{Implementation of tensor L-conv 
% \label{sec:implememnt}
% }

\subsection{Approximating 1D CNN}
\begin{figure}
    \centering
    \includegraphics[width=\linewidth]{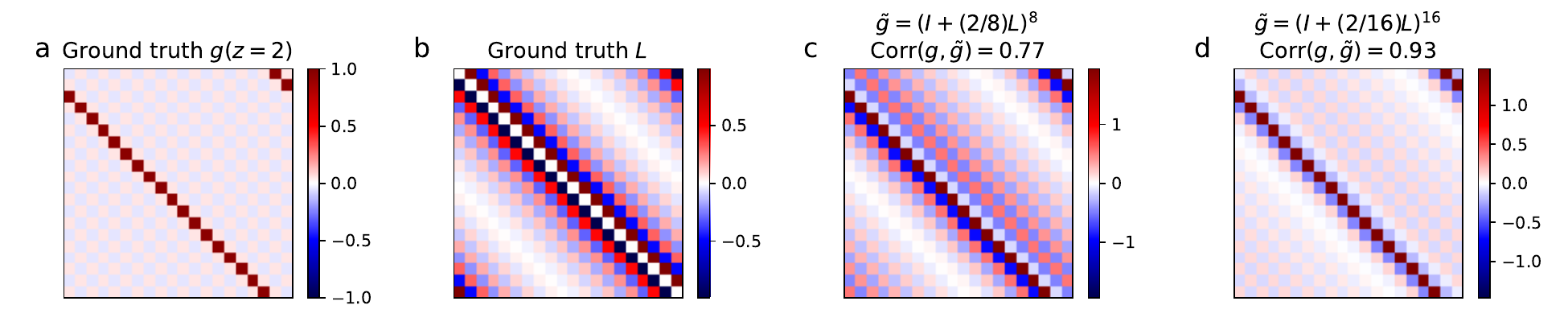}
    \caption{{1D Translation:} 
    Using the Shannon-Whittaker Interpolation (SWI) one can generate continuous shifts on discrete data. 
    These include integer shifts (a, ground truth).
    (SWI) also yields an infinitesimal generator $L$ for shifts (b). 
    This $L$ can be used to approximate finite shifts using $\tilde{g}_n(z)= (I+z/n L)^n$, with $n\to\infty$ yielding $\exp[zL]$.  
    (c) and (d) show the approximation of a shift by two pixels using $n=8$ and $n=16$. 
    % The correlation with ground truth $g(z)$ is calculated using $Corr(g,\tilde{g}) = \Tr{g^T g}/(\|g\|\|\tilde{g}\|)$ 
    }
    \label{fig:1D-shift}
\end{figure}
% While it is trivial to find permutations taking $\vx_0$ to $\vx_\mu$, most of them won't preserve the topology as \eqref{eq:lift-neigh}. 
% be an element of a given $G$. 
% For instance, 
As discussed in the text, 
% This is a special case of expressing G-conv as L-conv when the group is continuous 1D translations. 
% The arguments here generalize trivially to higher dimensions. 
\citet[sec. 4]{rao1999learning} used the Shannon-Whittaker Interpolation (SWI) \citep{whitaker1915functions} to define continuous translation on periodic 1D arrays as $\vf'_\rho = g(z)_\rho^\nu \vf_\nu $.
Here $g(z)_\rho^\nu = {1\over d} \sum_{p=-d/2}^{d/2} \cos \pa{{2\pi p\over d}(z+\rho -\nu)} $ approximates the shift operator for continuous $z$.
These $g(z)$ form a 1D translation group $G$ as $g(w)g(z)=g(w+z)$ with $g(0)^\nu_\rho = \delta^\nu_\rho$.
For any $z=\mu \in \Z$,
$g_\mu=g(z=\mu)$ are circulant matrices that shift by $\mu$ as $ [g_\mu]^\rho_\nu=\delta^\rho_{\nu-\mu}$.
\out{
Thus, a 1D CNN with kernel size $k$ can be written suing $g_\mu$ as 
\begin{align}
    F(\vf)_{\nu}^a & = \sigma\pa{\sum_{\mu=0}^k  \vf_{\nu-\mu}^c [W^\mu ]^a_{c}  +b^a} = \sigma\pa{\sum_{\mu=0}^k  [g_\mu\vf]_\nu^c [W^\mu ]^a_{c}  +b^a}  
    \label{eq:CNN-1D}
\end{align}
where $W,b$ are the filter weights and biases. 
}%%%%
$g_\mu$ can be approximated using the Lie algebra and written as multi-layer L-conv as in sec. \ref{sec:G-conv2L-conv}. 
Using $g(0)^\nu_\rho \approx \delta(\rho-\nu)$, the single Lie algebra basis $[\hat{L}]_0=\ro_z g(z)|_{z\to0}$, acts as $\hat{L}f(z) \approx -\ro_z f(z)$ (because $\int \ro_z \delta (z-\nu)f(z) =-\ro_\nu f(\nu)$). 
% The Lie algebra has a single basis $[\hat{L}]_0=\ro_z g(z)|_{z\to0}$. 
% Since $g(0)^\nu_\rho \approx \delta(\rho-\nu)$, we have $\hat{L}f(z) \approx -\ro_z f(z)$. 
% , meaning $\hat{L}$ is a Fourier series for $\ro_z$. 
% We have
% \begin{align}
%     [\hat{L}]^\nu_\rho&=\hat{L}(\rho-\nu) = \sum_p {2\pi p \over d^2} \sin \pa{{2\pi p\over d}(\rho -\nu)},&
%     [\hat{L}\vf]_\rho &= \sum_\nu \hat{L}(\rho-\nu)\vf_\nu = [\hat{L}\star \vf]_\nu
% \end{align}
Its components are
$\hat{L}^\nu_\rho={L}(\rho-\nu) = \sum_p {2\pi p \over d^2} \sin \pa{{2\pi p\over d}(\rho -\nu)} $, which are also circulant due to the $(\rho-\nu)$ dependence.
% are approximately 
% ${L}(z)\approx {d\over \pi} \sin(\pi z) \br{z^{-2}+(d-z)^{-2}}$ on $z\in\Z$. 
% \nd{%From sinc form of $g(z)$. 
% What's the exact form?}
Hence, $[\hat{L}\vf]_\rho = \sum_\nu {L}(\rho-\nu)\vf_\nu = [{L}\star \vf]_\nu$ is a convolution. 
\citet{rao1999learning} already showed that this $\hat{L}$ can reproduce finite discrete shifts $g_\mu$ used in CNN.
They used a primitive version of L-conv with $g_\mu = (I+\eps \hat{L})^N$. 
Thus, L-conv can approximate 1D CNN. 
This result generalizes easily to higher dimensions. 
\out{
% $\hat{L}$ is a circulant matrix. 
$\hat{L}^\nu_\rho$ becomes concentrated around small $\rho-\nu$ when $d\gg1$  
% Thus, for large $d$ 
and we can approximate $L$ using $k\ll d$ nearest neighbors of each node.
Here $\mathcal{S}$ is a periodic 1D lattice with 
% $k$ nearest neighbor 
% graph encoding the topology is a circulant matrix
and adjacency matrix
$\mA_{\mu\nu} = \sum_{-k<\nu<k} \delta_{\mu,\mu+i}$ being a circulant matrix.
}%%%%

Figure \ref{fig:1D-shift} shows how this approximation works. 
(b) shows the analytical form of $\hat{L}$. 
(c) and (d) show two approximations of $g_2 = g(z=2)$, shift by two pixels, using $\tilde{g}(z)= (I+z/n \hat{L})^n$ with $n=8$ and $n=16$. 
We can evaluate the quality of these approximations using their cosine correlation defined as $\mathrm{Corr}(g,\tilde{g}) = \Tr{g^T \tilde{g}}/(\|g\|\|\tilde{g}\|)$
where $\|A\| = \sqrt{\sum_{i,j}(A^i_j)^2}$ is the Frobenius $L_2$ norm. 
(c) shows $0.77$ correlation and (d) has $0.93$.

\subsection{Extracting 2D rotation generator for fixed small rotations 
\label{ap:exp-L-small}
}
\paragraph{Ground truth}
We can use the same SWI 1D translation generators discussed above for CNN as $\ro_x$ and $\ro_y$ to construct the rotation generator $L_\theta = x\ro_y-y\ro_x$ (Fig. \ref{fig-ap:L-so2-combined}, a). 
We will use cosine correlation with this  $L_\theta $ to evaluate the quality of the learned $\hat{L}$. 
As we will find below, the best outcome is from $\hat{L}$ learned using a recursive L-conv learning the angle of rotation between a pair of images (Fig. \ref{fig-ap:L-so2-combined}, b) with $0.70$ correlation. 

\out{
\paragraph{Learning symmetries using L-conv}
We can use L-conv for learning the Lie algebra as well. 
In fact, the architecture used in \citet{rao1999learning} is a basic version of L-conv with $W^0 = 1$. 
They show that with a small fixed $\ba{\eps}^i$ they could learn learn the single $L_i$ for continuous 1D translations and for 2D rotations.
Indeed, the architecture used in \citet{rao1999learning} is a special case of L-conv with $W^0 = 1$ and $\ba{\eps}^i\in \R$. 
We conducted a more advanced experiment with L-conv learning rotation angles in pairs of random images. % (SI \ref{ap:exp-L-multi})
% experiments for with $G=SO(2)$.
% In the first test, we used fixed small rotation angle $\pi/10$ and used 
Figure \ref{fig-ap:L-so2-combined} (c,d) shows the learned $L$. 
(c) shows $L\in so(2)$ learned using L-conv in $3$ recursive layers to learn rotation angles between a pair of $7\times 7$ random images $\vf$ and $R(\theta) \vf$ with $\theta \in [0,\pi/3)$. 
Middle and right of Fig. \ref{fig-ap:L-so2-combined} are experiments with fixed small rotation angle $\theta = \pi/10$ (SI \ref{ap:exp-L-small}).
Middle is the $L$ learned using L-conv and right is using the exact solution $R = (YX^T)(X^TX)^{-1}$. 
While this $L$ is less noisy, it does not capture weights beyond first neighbors of each pixel. 
Right shows the generator calculated using the exact solution to the linear regression problem with $\theta = \pi/10$. 
The SGD solutions using L-conv are less noisy and capture more details.
}%%%%

\begin{figure}
    \centering
    \includegraphics[width=.35\linewidth]{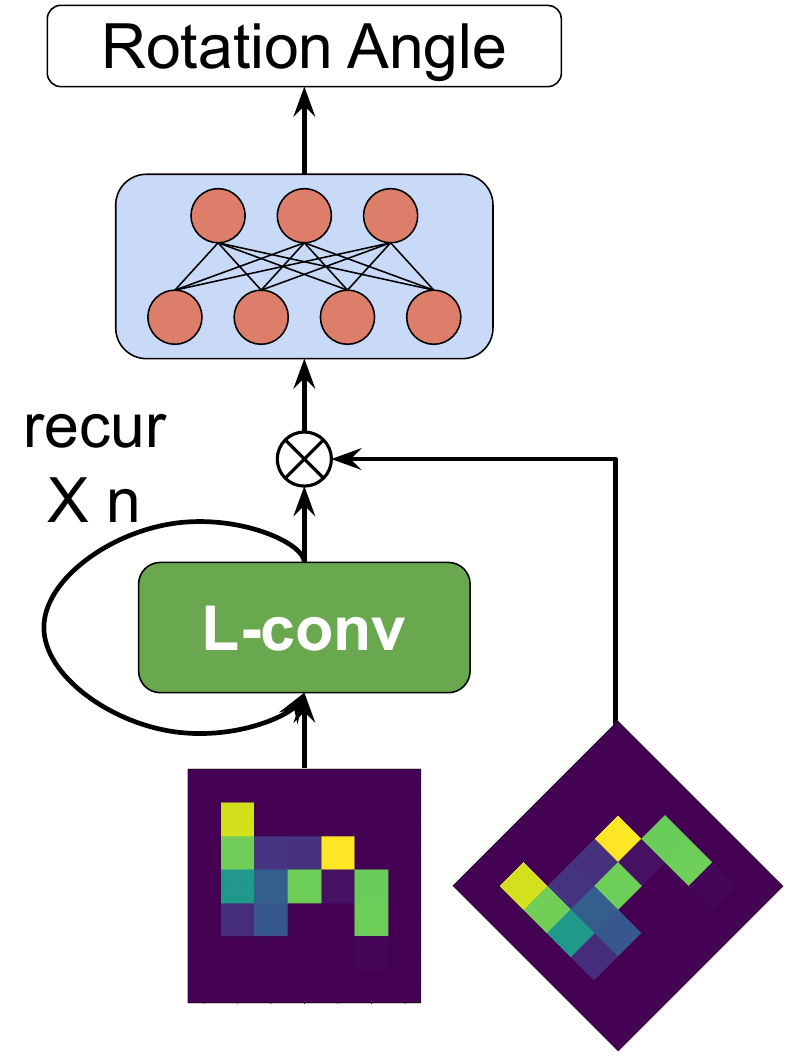}
    \includegraphics[width=.64\linewidth]{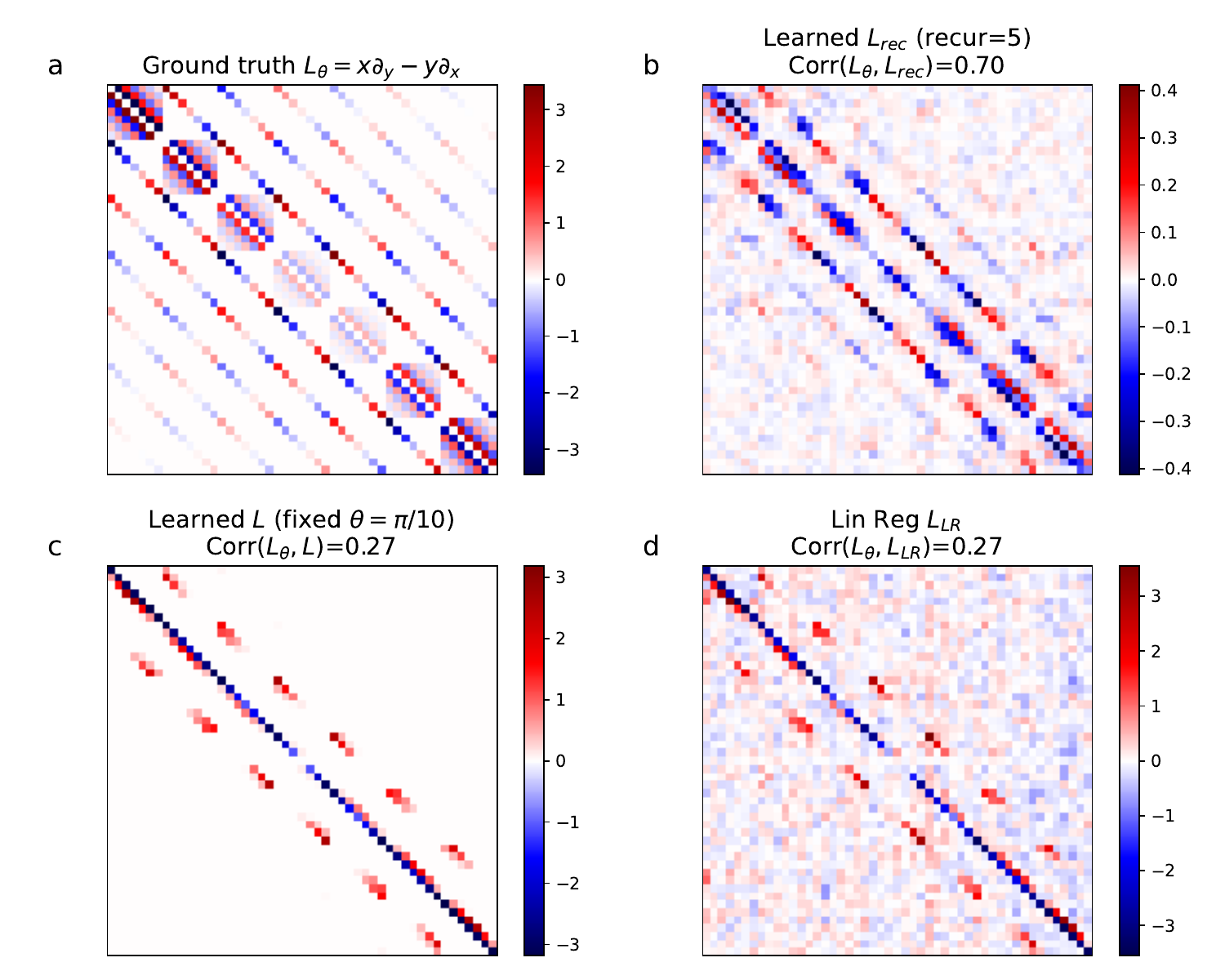}
    \caption{
    \textbf{Learning the infinitesimal generator of $SO(2)$} 
    Left shows the architecture for learning rotation angles between pairs of images. % (SI \ref{ap:exp-L-multi}). 
    (a) shows the rotation generator calculated analytically using Shannon-Whittaker interpolation. 
    (b) is the $\hat{L}$ learned using recursive L-conv learning rotation angle between a pair of images. %to learn rotation angle between pairs of images.  
    %in $3$ recursive layers to learn rotation angles between a pair of $7\times 7$ random images $\vf$ and $R(\theta) \vf$ with $\theta \in [0,\pi/3)$. 
    (c) is an $L$ learned using a fixed small rotation angle $\theta = \pi/10$, and
    % While this $L$ is less noisy, it does not capture weights beyond first neighbors of each pixel. 
    (d) shows $\hat{L}$ found using the numeric linear regression solution from the fixed angle data.
    (b) has the highest cosine correlation (0.70) with the ground truth, compared to 0.27 for $\hat{L}$ extracted using small angles. 
    % The SGD solutions using L-conv are less noisy and capture more details.
    }
    \label{fig-ap:L-so2-combined}
\end{figure}

\out{
\begin{figure}
    \centering
    \includegraphics[width=.17\linewidth]{figs2/L-conv-recur-3.pdf}
    \includegraphics[width=.82\linewidth]{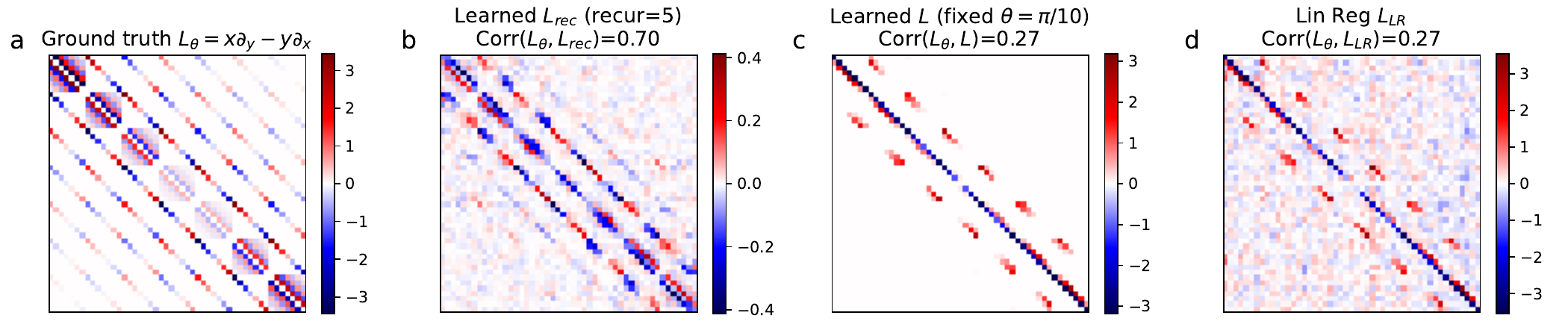}
    \caption{
    \textbf{Learning the infinitesimal generator of $SO(2)$} 
    Left shows the architecture for learning rotation angles between pairs of images. % (SI \ref{ap:exp-L-multi}). 
    Next to it is the $L$ learned using recursive L-conv in this experiment. %to learn rotation angle between pairs of images.  
    %in $3$ recursive layers to learn rotation angles between a pair of $7\times 7$ random images $\vf$ and $R(\theta) \vf$ with $\theta \in [0,\pi/3)$. 
    Middle $L$ is learned using a fixed small rotation angle $\theta = \pi/10$, and
    % While this $L$ is less noisy, it does not capture weights beyond first neighbors of each pixel. 
    right shows $L$ found using the numeric solution from the data. 
    % The SGD solutions using L-conv are less noisy and capture more details.
    }
    \label{fig-ap:L-so2-combined}
\end{figure}
}%%%%

% \nd{Add our own experiments here.}
\paragraph{Using fixed small angle}
In the first experiment we try to learn a small rotation with angle $\theta=\pi/10$ using a single layer L-conv (Fig. \ref{fig-ap:L-so2-combined},c).
This experiment was already done in \citep{rao1999learning}. 
The input is a random $7\times 7$ image $\vf$ with pixels chosen in $[-.5,0.5)$. 
The output $\vf'$ is the same image rotated by $\theta$ using pytorch affine transform.
Our training set contains 50,000 images., the test set was 10,000 images, 
batch size was 64. 
The code was implemented in pytorch and we used the Adam optimizer with learning rate $10^{-2}$. 
The experiments were run for 20 epochs. 
This problem is simply a linear regression with $\vf' = R\vf = (I+\eps \hat{L})\vf $. 
L-conv solves it using SGD and finds $\eps \hat{L}$. 
This problem can also be solved exactly using the solution to linear regression. 
Let $X=(\vf_1,\dots \vf_N)$ and $Y=(\vf'_1,\dots \vf'_N)$ be the matrix of all inputs and outputs, respectively. 
The rotation equation is $Y^T = RX^T$.
Thus, the rotation matrix is given by $R = (Y^TX)(X^TX)^{-1}$. 
Figure \ref{fig-ap:L-so2-combined} (c, d) shows the results of this experiment. 
The $L$ found using L-conv with SGD is much cleaner than the numerical linear regression solution $L_{LR} = (R-I)/\theta$. 
The loss becomes extremely small both on training and test data. 

\out{
\begin{figure}
    \centering
    \includegraphics[width=1\linewidth]{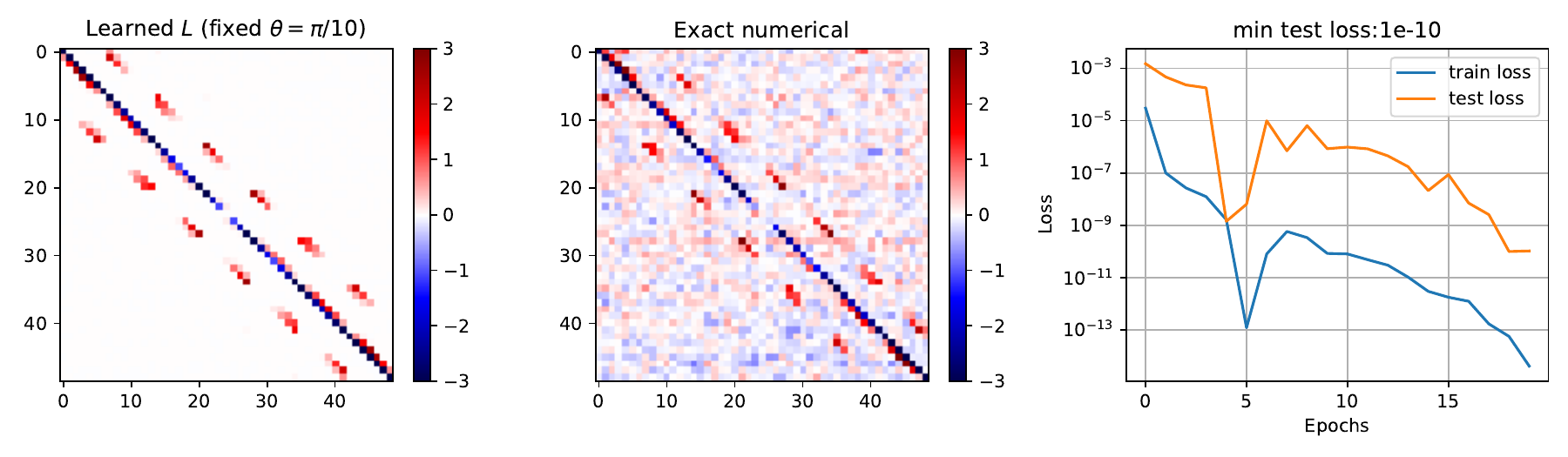}
    \caption{Single layer L-conv learns the infinitesimal generator of 2D rotations.
    The learned $L$ (left) is much cleaner than a numerical one calculated using the exact solution to the linear regression (middle). 
    Both training and test losses are extremely small. 
    }
    \label{fig:L-L0-so2}
\end{figure}
}%

\subsection{Learning rotation angle
\label{ap:exp-L-multi}
}
In this experiment we have a pair of input images $ (\vf_n, R(\theta_n)\vf_n)$, with $R(\theta) \in SO(2)$ (approximating 2D rotations).
The two inputs differ by a finite rotation with angle $\theta_n \in [0,\pi/8)$. 
Th task is to learn the rotation angle $\theta_n$. 
For this task we use a recursive L-conv. 
We set $W^0=I$.
The L-conv weight $\ba{\eps}$ is $m\times m$. 
To be able to encode multiple angles, we set $m=10$ and feed $10$ copies of the $\vf_n$ as input $\vh_0 = [\vf_n]\times 10$.
We pass this through the same L-conv layer $t=3$ times as $\vh_i = Q[\vh_{i-1}]$.  
In the final layer, we first take the dot product of the final output $\vh_t$ with the rotated input $\vy_n = R(\theta)\vf_n $ to obtain $\vg = \tanh \pa{{\vy_n}^T \vh_t}  $. 
We then pass the output $\vg \in \R^m$ through a fully-connected (FC) layer with $5$ nodes and tanh activation, and finally through a linear FC layer with one output to obtain the angle. 
The batch size was 16, Adam optimizer, learning rate $10^{-3}$, rest were default.

Despite being a much harder task than fitting a fixed angle rotation, the learned $\hat{L}$ of this experiment has the highest (0.70) cosine correlation with the ground truth $L_\theta$ (Fig. \ref{fig-ap:L-so2-combined}, b). 
Even though the architecture is rather complicated and L-conv is followed by two MLP layers, the $\hat{L}$ in L-conv learns the infinitesimal generator of rotations very well. 
We also conducted experiments with larger random images $(20\times 20)$ and larger angles of rotation $\theta_n \in [0,\pi/4)$ (Fig. \ref{fig:L-so2-20}).
While the accuracy of learning the angles is still pretty good (Fig. \ref{fig:L-so2-20}, e,test loss $2.7\times 10^{-4}$) 
the larger angles result in less correlation between the learned $\hat{L}$ and the ground truth $L_{\theta}$ (Fig. \ref{fig:L-so2-20}, b, correlation $0.12$ with 8 times recurrence). 
The learned $\hat{L}$ is closer to a finite angle rotation.
This may be because the with small number of recurrences the network found small but finite rotations approximate larger rotations better than using a true infinitesimal generator.  

\begin{figure}
    \centering
    \includegraphics[width=\linewidth]{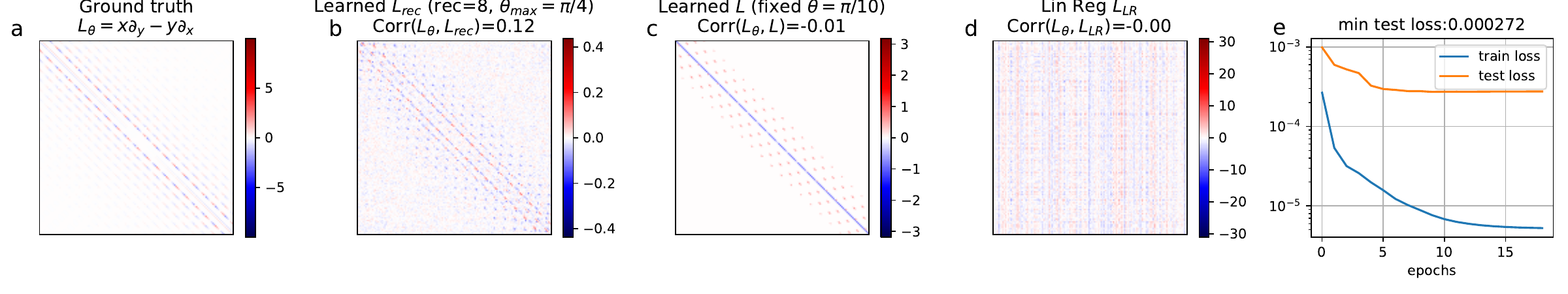}
    \caption{Learning $\hat{L}$ via larger rotation angles for larger images. 
    This time the correlation with ground truth $L_\theta$ is much less, but accuracy is still very good.}
    \label{fig:L-so2-20}
\end{figure}

\out{
\begin{figure}
    \centering
    \includegraphics[width=.8\linewidth]{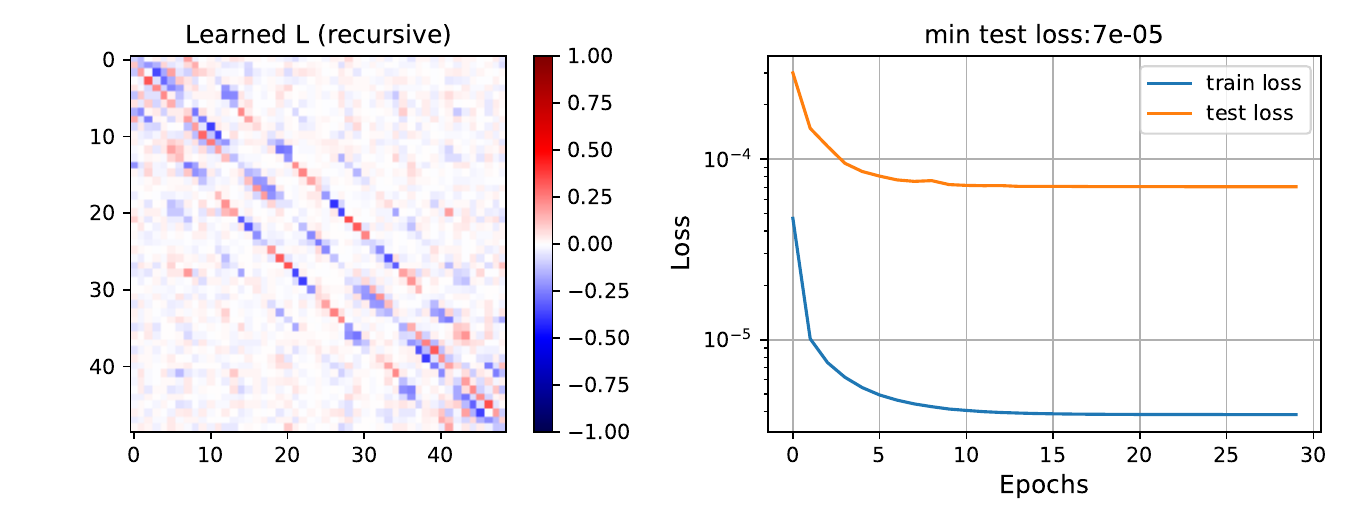}
    \caption{
    Recursively used L-conv for learning multiple rotation angles. 
    The learned $L$ (left) also captured some weights beyond first neighbors. 
    }
    \label{fig:L-so2-multi}
\end{figure}
}%%%%

\out{
\nd{Vector field plots of $\hat{L}_i$ in experiments.
Apply orthogonality with $B_{ij} = 2k\Tr{l_il_j}-2\Tr{l_i}\Tr{l_j}=$ }
}%%%%%

% \section{Experiments}

% \section{Discover Potential Symmetries \label{sec:learning}}

\out{
\subsection{Learning Lie  Algebra}
\ry{rotation MNIST and check the generators}
\nd{known symmetry experiments; two generators, non-abelian}

\nd{
\cite{rao1999learning} used data augmentation to learn single Lie algebra generators for 1D shift and 2D rotation (5x5 and 9x9 random images). 
They used fixed shifts (0.5 pixel) and angles ($0.2\approx \pi/15$ radians CW or CCW) for learning the generator $L$ ($G$ in their paper). 
They then use the learned $L$ to estimate the shift amount or angle.    
If we could learn the generator and the angle at the same time, we are doing something new. 
}

When dealing with \textit{unknown} continuous symmetry groups, it can be  impossible to design a G-conv.
The Lie algebra, however, has a much simpler, linear structure and universal for all Lie groups. 
Because of this, L-conv affords us with a powerful tool to probe systems with unknown symmetries. 
% One of the reasons we introduce L-conv is to have an equivariant architecture which is generic enough to be used to learn symmetries. 

L-conv is a generic weight-sharing ansatz and the number of $L_i$ is usually small. 
This means that even if we do not know $L_i$, it may be possible to \textit{learn} the $L_i$ from the data.

% In fact, as we show in our experiments, \edits{learning low-rank $L_i$ via SGD simultaneously with other weights} yield impressive performance on data with hidden symmetries (Fig. \ref{fig:L-conv-4-datasets}), without needing any input about the symmetry. 

% \subsection{Learning Lie Algebra Basis}
\edits{We learn the $L_i$ using SGD, simultaneously with $W^i$ and all other weights. 
Our current implementation is similar to a GCN $F(h) = \sigma(L_i\cdot h \cdot W^i +b)$ where both the weights $W^i$ and the propagation rule $L_i$ are learnable.
}
When the spatial dimensions $d$ of the input $x\in \R^{d\times c}$ is large, e.g. a flattened image, the $L_i$ with $d^2$ parameters can become expensive to learn. 
However, generators of groups are generally very sparse matrices. 
Therefore, we represent $L_i$ 
\edits{
using low-rank decomposition
% The $L_i$ is encoded as sparse matrices
$L_i = U_iV_i$. An encoder $V$ of shape $n_L\times d_h\times d$ encodes $n_L$ matrices $V_i$, and a decoder $U$ of shape $n_L\times d\times d_h$. 
Here $d$ is the input dimensions and $d_h$ the latent dimension with $d_h\ll d$ for sparsity.
}
% with an autoencoder structure with one or multiple hidden dimensions. 
% The hidden bottleneck allows for low-rank encoding of the $L_i$.

\out{
In order for the $L_i$ to form a basis for a Lie algebra, they should be closed under the commutation relations \eqref{eq:Lie-commutator}, as well as orthogonal under the Killing Form \cite{hall2015lie}. %[Chapter~6]
These conditions can be added to the model as regularizers (see Appendix \ref{ap:killing}), but 
% A naive implementation of these conditions 
regularization also introduces an additional time complexity of  $O(n_L^2d_h^2 d)$, which can be quite expensive compared to the $O(n_ld_hd)$ of learning $L_i$ via SGD. 
Therefore, in the experiments reported here we did not use any regularizers for $L_i$.

% \nd{clean up and move to appendix}
}%%%

\out{
\ND{Discuss Chelsea Finn's meta-learning}
\nd{
Discuss \cite{zhou2020meta}, 
Their weight sharing matrix $U$ looks different from ours. 
It is a Kronecker decomposition of the full filter convolved over the input. 
In our case, each $L$ generates one step of the convolution and gets one of the weights parameters. 
Each $L$ then moves in one direction with a given step size. 
Again, they mention finite groups, like permutation.
Their $U$ is a stack of all permutations and they will share the same weights. 
\citet{ravanbakhsh2020universal} Universal MLP: proves a universal approximation theorem for single hidden layer equivariant MLP for Abelian and finite groups.
}
}%%%

Next, we show that in certain cases, such as linear regression, the symmetry group %a subset of continuous symmetries 
can be derived analytically. 
While the results may not apply to more non-linear cases, they give us a good idea of the nature of the symmetries we should expect L-conv to learn. 
% Our experimental results with L-conv follow after that. 

\out{
\subsection{Equivariance and Invariance when Labels are Categorical}
We will first show that if $G$ is continuous and connected, then in problems with categorical labels such as classification, the only valid equivariances either keep the labels invariant, or only include a discrete subgroup of $G$.  

\begin{lemma}\thlabel{lem:const-rep}
The only representation of a connected Lie group $G$ on $\Z_n$ are constant. 
\end{lemma}
%%%%
\begin{proof}
A representation is a smooth homomorphism, meaning that it is continuous and infinitely differentiable.  
Since $G$ is a connected Lie and hence topological group, it has a connected continuous manifold.
A function $T:G\to \Z_n$ is continuous functions if its pre-image in $G$ is open.
This implies that if $T(u)=z$, there exist an open ball $u\in B_u$ on which $T$ is constant ($T(u') = z, \forall u'\in B_u$.
Consider two $u,v \in G$ for which $T(u)\ne T(v)$, belonging to open balls $u\in B_u$ and $v\in B_v$ over which $T$ is constant. 
Since $T(u)\ne T(v)$ we must have $B_u\cap B_v = \emptyset$. 
Therefore, any element in the boundary of $B_u$ cannot be in any other $B_v$, meaning the domain of $T$ is not all elements in $G$ and hence it cannot be a smooth homomorphism.
Thus, the only smooth $T:G\to \Z_n$ are constant functions.
\end{proof}

\begin{corollary}
In supervised learning with continuous inputs $\vx \in \R^d$ and categorical outputs $y\in \Z_2^m$, if the dataset is equivariant under a connected Lie group $G$ the only possible equivariance is label invariance, or equivariance under discrete subgroups of $G$.  
\end{corollary}
\begin{proof}
From Lemma \thref{thm:const-rep}, if $G$ is continuous and connected the only representations $T:G\to \Z_2$ must be constant, meaning the labels are kept invariant. 
For discrete subgroups of $G$ this restriction does not exist. 
\end{proof}

}%%%%

\subsection{Known Symmetry on Physical Systems}

\ry{ pendulum / potential well phase space, discover the ??}

}%%%%%%%

\section{Experiments on Images 
\label{ap:exp-image}
}

To understand precisely how L-conv performs in comparison with CNN and other baselines, we conduct a set of carefully designed experiments. 
Defining pooling for L-conv merits more research. 
Without pooling, we cannot use L-conv in state-of-the-art models for problems such as image classification. 
Therefore, we 
%comparing L-conv with CNN and fully connected (FC) layers 
use the simplest possible models in our experiments: one or two L-conv, or CNN, or FC layers, followed by a classification layer. 
We do not use any other operations such as dropout or batch normalization in any of the experiments. 
% This allows us to directly compare the effect of replacing a CNN with L-conv, for instance. 
% We conducted a series of experiments on images to assess the usefulness of L-conv. 

% $$
%  num_filters=32, 
%                 kernel_size=9, 
%                 L_hid= [16], #[16], 
%                 activation = 'relu',
%                 L_trainable = True),
% $$
\begin{figure*}[ht]
    \centering
    \includegraphics[width=.49\textwidth]{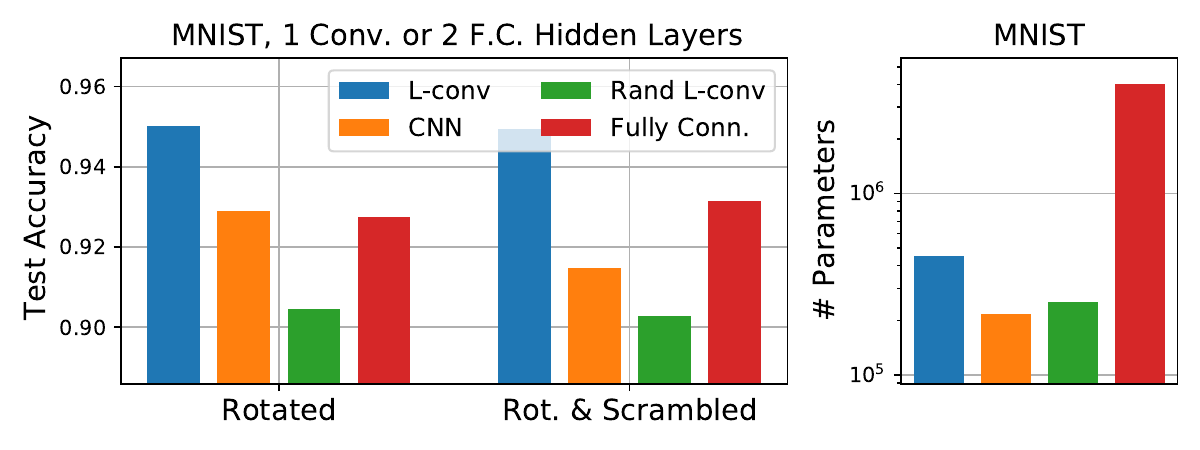}
    \includegraphics[width=.49\textwidth]{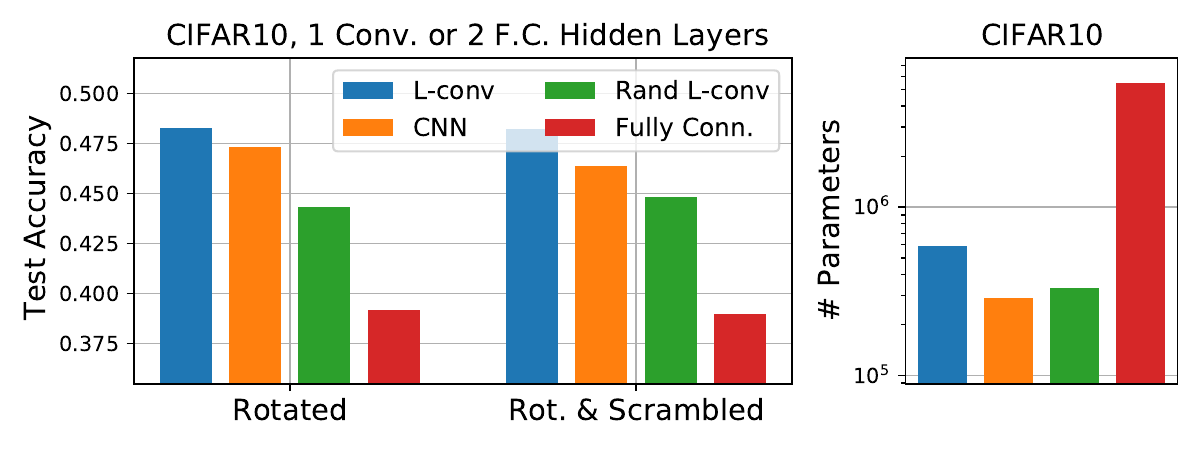}
    \includegraphics[width=.49\textwidth]{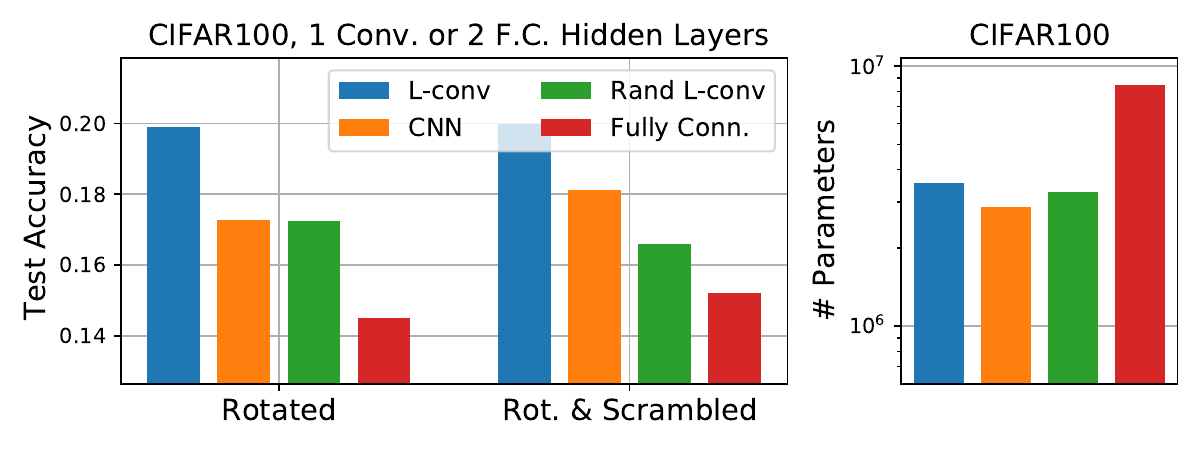}
    \includegraphics[width=.49\textwidth]{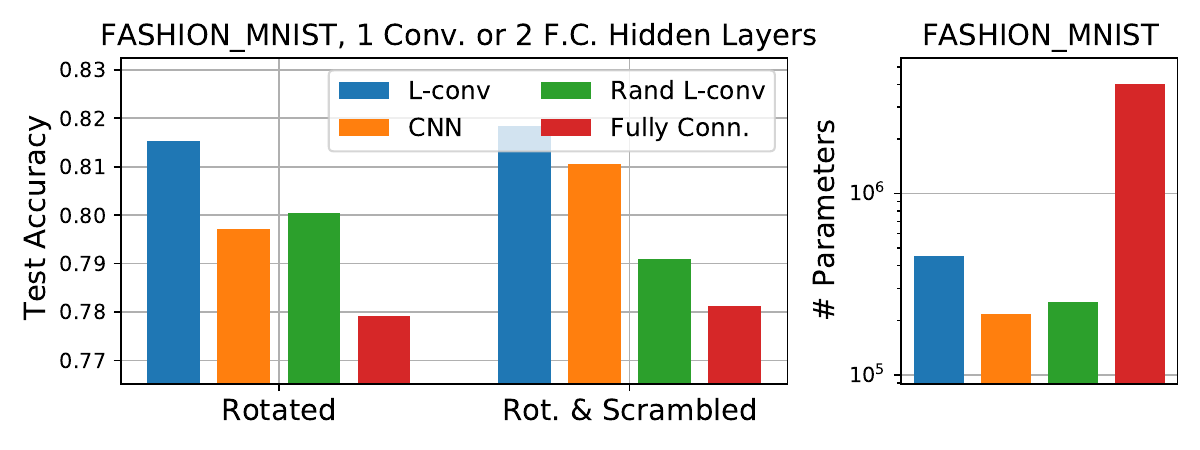}
    \caption{Results on four datasets with two variant: ``Rotated'' and ``Rotated and scrambled''. 
    In all cases L-conv performs best. 
    On MNIST, FC and CNN come close, but using 5x more parameters.  
    }
    \label{fig:L-conv-4-datasets}
\end{figure*}

\begin{figure*}[ht]
    \centering
    \includegraphics[trim=0 1cm 0 0, width=\textwidth]{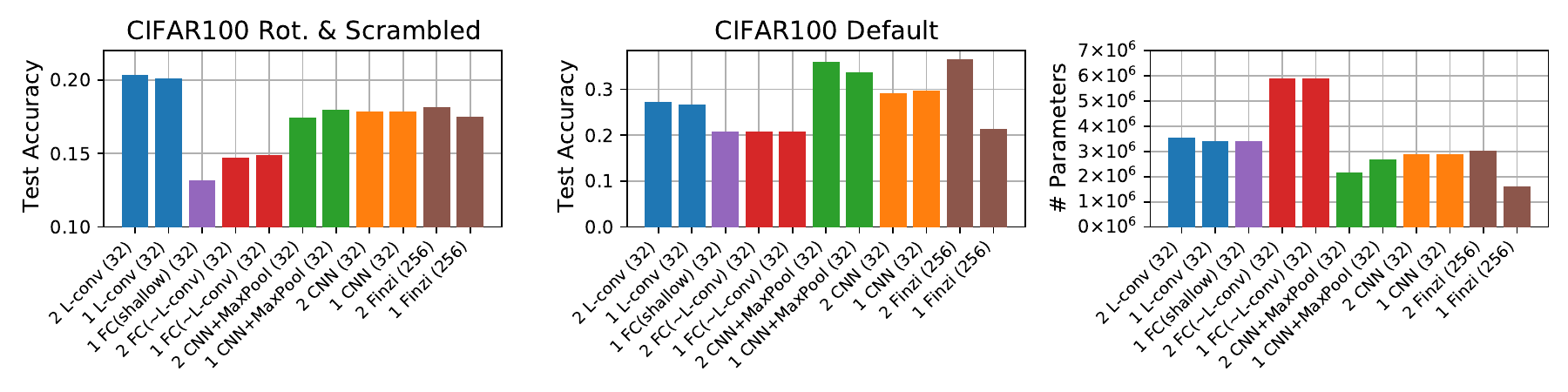}
    \caption{
    Comparison of one and two layer performance of L-conv (blue), CNN without pooling (orange), CNN with Maxpooling after each layer (green), fully connected (FC) with structure similar to L-conv (red), Lie-Conv (Finzi) \cite{finzi2020generalizing} (brown), and shallow FC, which has a single hidden layer with width such that the total number of parameters matches L-conv (purple). 
    The labels indicate number of layers, layer architecture and number of filters
    (e.g. ``2 L-conv (32)'' means two layers of L-conv with 32 filters followed by one classification layer). 
    Left and middle plots show test accuracies on CIFAR100 with rotated and scrambled images, and on the original CIFAR100 dataset, respectively.
    The plot on the right shows the number of parameters in each model, which is the same for the two datasets.
    % L-conv outperforms all other tested models on rotated and scrambled CIFAR100. 
    % On Default CIFAR100, CNN with Maxpooling significantly outperforms regular CNN and L-conv, indicating the importance of proper pooling.
    }
    \label{fig:2layer}
\end{figure*}

\textbf{Test Datasets}
We use four datasets: MNIST, CIFAR10, CIFAR100, and FashionMNIST. 
To test the efficiency of L-conv in dealing with hidden or unfamiliar symmetries, we conducted our tests on two modified versions of each dataset: 1) \textbf{Rotated:} each image rotated by a random angle (no augmentation); 2) \textbf{Rotated and Scrambled:} random rotations are followed by a fixed random permutation (same for all images) of pixels. 
We used a 80-20 training test split on  60,000 MNIST and FashionMNIST, and on 50,000 CIFAR10 and CIFAR100 images. 
Scrambling destroys the correlations existing between values of neighboring pixels, removing the locality of features in images. 
As a result, CNN need to encode more patterns, as each image patch has a different correlation pattern. 

\textbf{Test Model Architectures}
We conduct controlled experiments, with one (Fig. \ref{fig:L-conv-4-datasets}) or two (Fig. \ref{fig:2layer}) hidden layers being either L-conv or a baseline, followed by a classification layer. 
For CNN, L-conv and L-conv with random $L_i$, we used $n_f=m_l=32$ for number of output filters (i.e. output dimension of $W^i$). 
For CNN we used $3\times 3$ kernels and equivalently used $n_L= 9$ for the number of $L_i$  in L-conv and random L-conv. 
We also used ``LieConv'' \cite{finzi2020generalizing} as a baseline (Fig. \ref{fig:2layer}, brown). 
We used the default $k=256$ in LieConv, which yields comparable number of parameters to our other models.
For the symmetry group in LieConv we used $SE(3)$. 
We also used the default ResNet architecture provided by \citet{finzi2020generalizing} for both the one and two layer experiments. 
We turned off batch normalization, consistent with other experiments. 
% For low-rank encoding of $L_i$ we used $k=16$, although lower values like $k=6$ had similar performance. 
% In L-conv, we use low-rank encodings for $L_i$. 
% The L-conv model layers, in Figure \ref{fig:model-architectures}
% have three parts. 
% As described in sec. \ref{sec:learning},
We encode $L_i$ as sparse matrices $L_i = U_iV_i$ with hidden dimension $d_h=16$ in Fig. \ref{fig:L-conv-4-datasets} and $d_h=8$ in Fig. \ref{fig:2layer}, showing that very sparse $L_i$ can perform well.
% $L_i = U_iV_i$ using an encoder $V$ of shape $n_L\times d_h\times d$ encoding $n_L$ matrices $V_i$, with $d$ being the input dimensions and $d_h$ the latent dimension, and a $d\times d_h$ decoder $U_i$. 
The weights $W^i$ are each $m_{l}\times m_{l+1}$ dimensional. 
The output of the L-conv layer is $d\times m_{l+1}$. 
As mentioned above, we use two FC baselines. 
The FC in Fig. \ref{fig:L-conv-4-datasets} and FC($\sim$L-conv) in Fig. \ref{fig:2layer} mimic L-conv, but lacks weight-sharing.
The FC weights are $W = ZV$ with $V$ being $(n_L d_h)\times d$
and $Z$ being $(m_{l+1}\times d)\times d_h$. 
For ``FC (shallow)'' in Fig. \ref{fig:2layer}, we have one wide hidden layer with $u=n_{L-conv}/(m d c) $, where $n_{L-conv}$ is the total number of parameters in the L-conv model, $m$ and $c$ the input and output channels, and $d$ is the input dimension. 
We experimented with encoding $L_i$ as multi-layer perceptrons, but found that
% In practice, we found that 
a single hidden layer with linear activation % and very low dimensions (hidden width 6 on rotated and scrambled MNIST) 
works best. % in our tests.
% \textbf{Two Layers and Pooling}
We also conduct tests with two layers of L-conv, CNN and FC (Fig. \ref{fig:2layer}), with each L-conv, CNN and FC layer as descried above, except that we reduced the hidden dimension in $L_i$ to $d_h=8$. 
% For L-conv and CNN, the first layer has $m_1=32$ filters, and the second layer has $m_2=64$ filters. For CNN we used $3\times 3$ kernels for both layers, and equivalently for L-conv we used $n_L = 9$ for the number of $L_i$.
% For FC we used $1024$ units for each layer.

\textbf{Baselines}
We compare L-conv against four baselines: 
CNN, random $L_i$, fully connected (FC) and LieConv. 
Using CNN or $SE(3)$ LieConv on scrambled images amounts to using poor inductive bias in designing the architecture.
Similarly, random, untrained $L_i$ is like using bad inductive biases. 
Testing on random $L_i$ serves to verify that L-conv's performance is not due to the structure of the architecture, and that the $L_i$ in L-conv really learn patterns in the data. 
Finally, to verify that the higher parameter count in L-conv is not responsible for the high performance, we construct 
two kinds of FC models.
The first type (``Fully Conn.'' in Fig. \ref{fig:L-conv-4-datasets} and ``FC ($\sim$ L-conv)'' in Fig. \ref{fig:2layer}) is a 
multilayer FC network with the same input ($d\times m_0$), hidden ($k\times n_L $ for low-rank $L_i$) and output ($d\times m_1$)  dimensions as L-conv, but lacking the weight-sharing, leading to much larger number parameters than L-conv.
The second type (``FC (shallow)'' in Fig. \ref{fig:2layer}) consists of a single hidden layer with a width such that the total number of model parameters match L-conv.

\textbf{Results}
% The results of our test are shown in Fig. \ref{fig:L-conv-4-datasets} 
% \edits{ and \ref{fig:2layer}.}
Fig. \ref{fig:L-conv-4-datasets} shows the results for single layer experiments. 
% As we see, 
On all four datasets both in the rotated and the rotated and scrambled case L-conv performed considerably better than CNN and the baselines. 
% We are also showing the total number of trainable parameters in L-conv and other model next to the accuracy plot using the same colors. 
Compared to CNN, 
L-conv naturally requires extra parameters to encode $L_i$, but low-rank encoding with rank $d_h\ll d$ only requires $O(d_h d)$ parameters, which can be negligible compared to FC layers. 
We observe that FC layers consistently perform worse than L-conv, despite having much more parameters than L-conv.
We also find that not training the $L_i$ (``Rand L-conv'') leads to significant performance drop. 
We ran tests on the unmodified images as well (Supp. Fig \ref{fig:L-conv-4-datasets-full}), where CNN performed best, but L-conv trails closely behind CNN. 

% The results of the comparison between single and double layers on MNIST are shown in Fig.~\ref{fig:MNIST_single_double_layer_lconv_cnn_fc}. 
Additional experiments testing the effect of number of layers, number of parameters and pooling are shown in Fig. \ref{fig:2layer}.
On CIFAR100, we find that both FC configurations, FC($\sim$L-conv) and FC(shallow) consistently perform worse than L-conv, evidence that L-conv's performance is \textit{not} due to its extra parameters.
% This is evidence that the improved performance of L-conv on a domain with hidden and complex symmetries such as rotated scrambled CIFAR100.
L-conv outperforms all other tested models on rotated and scrambled CIFAR100, including LieConv. 
Without pooling, we observe that both L-conv and CNN do not benefit from adding a second layer. 
% This can be explained by  \thref{prop:multi-L-conv}, which states that multi-layer L-conv is still encoding the same symmetry group $G$, only covering a larger portion of $G$. 
% Our hypothesis is that lack of pooling is the reason behind this, as we discuss next.
On the default CIFAR100 dataset, one and two layer CNN with max-pooling perform significantly better than L-Conv. 
Two Layer $SE(3)$ LieConv (labelled ``2 Finzi (256)'') performs best on default CIFAR100, but not on the scrambled and rotated version. 
This is expected, as the symmetries of the latter are masked by the scrambling. 
% Also, note that the two layer LieConv is a ResNet with 27 trainable modules, whereas the two layer L-conv has only seven (an encoder and decoder for $L_i$ and shared weights  
This is where the benefit of our model becomes evident, namely cases where the data may have hidden or unfamiliar symmetries. 
We also verified that the higher performance of L-conv compared to CNN is not due to higher number number of parameters (Appendix \ref{ap:experiments-old})
\out{
\textbf{Pooling on L-conv} 
On Default CIFAR100, CNN with Maxpooling significantly outperforms regular CNN and L-conv, indicating the importance of proper pooling.
Interestingly, on rotated and scrambled CIFAR100, we find that max-pooling does not yield any improvement.
We believe the role of pooling is much more fundamental than simple dimensionality reduction.
On images, pooling with strides blurs our low-level features, allowing the next layer to encode symmetries at a larger scale. 
% Pooling operations, such maxpooling, generally yield a significant performance boost in CNN.  
\cite{cohen2016group} showed a relation between pooling and coset of subgroups and that strides are subsampling the group to a subgroup $H\subset G$, resulting in outputs which are equivariant only under $H$ and not the full $G$.
These subgroups appearing at different scales in the data may be quite different. 
However, a naive implementation of pooling on L-conv may involve three $L_i$ and be quite expensive (see Appendix \ref{ap:pool}). 
Devising an efficient and mathematically sound pooling algorithm for L-conv is a future step we are working on. 

% Lastly, unlike an actual Lie algebra basis, we did not include regularizers enforcing orthogonality among $L_i$. 
% As we see, for the rotated and rotated and scrambled MNIST, single-layer L-conv performs better than both CNN and FC, while double-layer CNN performs better than L-conv on the rotated MNIST but not the rotated and scrambled MNIST dataset.
% We can see that adding an extra layer improves the performance of L-conv and CNN, but decreases the accuracy of FC.

}%%%%

\out{
\begin{table}[]
\begin{tabular}{|l|l|}
\hline
L-conv 1 layer  & $n_L = 9$, $f=32$                     \\ \hline
L-conv 2 layers & $n_L = 9$, $f=32$; $n_L=9$, $f=64 $            \\ \hline
CNN 1 layer     & $k=(3,3)$, $f=32$              \\ \hline
CNN 2 layers    & $k=(3,3)$, $f=32$; $k=(3,3)$, $f=64$ \\ \hline
FC 1 layer      & $n = 1024$                         \\ \hline
FC 2 layers     & $n = 1024$; $n = 1024$                   \\ \hline
\end{tabular}
\caption{Parameters used in Fig.~\ref{fig:MNIST_single_double_layer_lconv_cnn_fc}.}
\end{table}
}

\out{
\textbf{Hardware and Implementation}
We implemented L-conv in Keras and Tensorflow 2.2 and ran our tests on a system with a 6 core Intel Core i7 CPU, 32GB RAM, and NVIDIA Quadro P6000 (24GB RAM) GPU. 
The L-conv layer did not require significantly more resources than CNN and ran only slightly slower. 
}%%%

\out{

\paragraph{Pooling on L-conv} 
Pooling operations, such maxpooling, generally yield a significant performance boost in CNN.  
\cite{cohen2016group} showed a relation between pooling and coset of subgroups. 
They also state that the strides in CNN pooling are like subsampling the group to a subgroup $H\subset G$, resulting in outputs which are equivariant only under $H$ and not the full $G$. 

Lastly, unlike an actual Lie algebra basis, we did not include regularizers enforcing orthogonality among $L_i$. 
}%%%
% \subsection{Pooling}
% \input{secs/pooling}

% \subsection{Pretrained Generators}
% Test linear classifier performance with/without pretraining

\out{
\begin{figure}
    \centering
    \includegraphics[width=.5\textwidth]{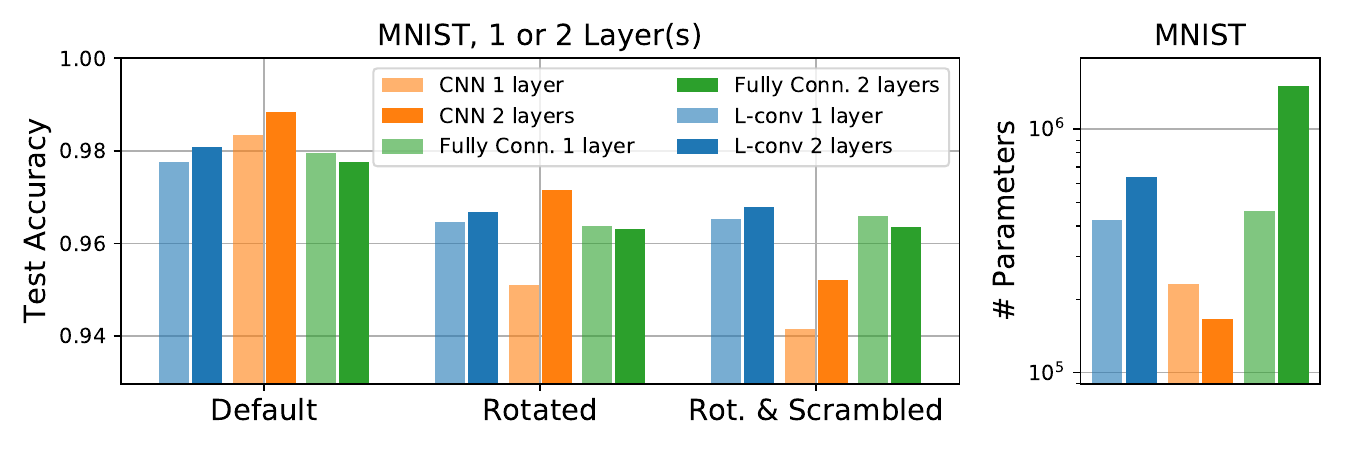}
    {\footnotesize 
    \begin{tabular}{l|l}
    {\bf model; layers} & {\bf parameters} \\
    \hline
    L-conv; 1  & $n_L = 9$, $n_f=32$                     \\ %\hline
    L-conv; 2 & $n_L = 9$, $n_f=32$; $n_L=9$, $n_f=64 $            \\ %\hline
    CNN; 1     & $k=(3,3)$, $n_f=32$              \\ %\hline
    CNN; 2    & $k=(3,3)$, $n_f=32$; $k=(3,3)$, $n_f=64$ \\ %\hline
    FC; 1      & $n = 1024$                         \\ %\hline
    FC; 2     & $n = 1024$; $n = 1024$                   \\ %\hline
    \end{tabular}
    }%%%
    \caption{Test results on MNIST with three variants: ``Default'', ``Rotated'' and ``Rotated and scrambled''.
    The architectures included in this figure are single layer and two layers L-conv, CNN, and fully-connected dense layers. 
    }
    \label{fig:MNIST_single_double_layer_lconv_cnn_fc}
\end{figure}
}

\subsection{Details of experiments}

\paragraph{Hardware and Implementation}
We implemented L-conv in Keras and Tensorflow 2.2 and ran our tests on a system with a 6 core Intel Core i7 CPU, 32GB RAM, and NVIDIA Quadro P6000 (24GB RAM) GPU. 
The L-conv layer did not require significantly more resources than CNN and ran only slightly slower.

\out{
\subsubsection{Comparison with related models \label{ap:compare} }

\paragraph{Comparison with Meta-learning Symmetries by Reparameterization (MSR)}
Recently \citet{zhou2020meta} also introduced an architecture which can learn equivariances from data. 
We would like to highlight the differences between their approach and ours, specifically Proposition 1 in \citet{zhou2020meta}.  
Assuming a discrete group $G=\{g_1,\dots, g_n\}$, they decompose the weights $W\in \R^{s\times s}$ of a fully-connected layer, acting on $\vx \in \R^s$ as $\mathrm{vec}(W) = U^Gv$ where $U^G\in \R^{s\times s}$ are the ``symmetry matrices'' and $v\in \R^s$ are the ``filter weights''. 
Then they use meta-learning to learn $U^G$ and during the main training keep $U^G$ fixed and only learn $v$.
We may compare MSR to our approach by setting $d=s$. 
First, note that although the dimensionality of $U\in \R^{nd\times d}$ seems similar to our $L \in \R^{n\times d\times d}$, the $L_i$ are $n$ matrices of shape $d\times d$, whereas $U$ has shape $(nd) \times d$ with many more parameters than $L$. 
Also, the weights of L-conv $W\in \R^{n\times m_l \times m_{l-1}}$, with $m_l$ being the number of channels, are generally much fewer than MSR filters $v\in \R^d$. 
Finally, the way in which $Uv$ acts on data is different from L-conv, as the dimensions reveal. 
The prohibitively high dimensionality of $U$ requires MSR to adopt a sparse-coding scheme, mainly Kronecker decomposition.
Though not necessary, we too choose to use a sparse format for $L_i$, finding that very low-rank $L_i$ often perform best. 
A Kronecker decomposition may bias the structure of $U^G$ as it introduces a block structure into it.

\paragraph{Contrast with Augerino}
In a concurrent work, \citep{benton2020learning} propose Augerino, a method to learn invariances with neural networks. 
Augerino uses data augmentation to transform the input data, which means it is restricting the group to be a subgroup of the augmentation transformations. 
% Augerino learns which subset of the augmentations improved the prediction. 
% This is done by writing 
The data augmentation is written as $g_\eps = \exp\pa{\sum_i \eps_i \theta_i L_i} $ (equation (9) in \cite{benton2020learning}), with randomly sampled $\eps_i\in [-1,1]$. 
$\theta_i$ are trainable weights which determine which $L_i$ helped with the learning task. 
However, in Augerino,  $L_i$ are fixed \textit{a priori} to be the six generators of affine transformations in 2D (translations, rotations, scaling and shearing). 
In contrast,  our approach is more general.  
We learn the generators $L_i$ directly without restricting them to be a known set of generators.  
Additionally, we do not use the exponential map, hence, implementing L-conv is very straightforward.  
Lastly, Augerino uses sampling to effectively cover the space of group transformations. 
Since the sum over Lie algebra generators is tractable, we do not need to use sampling.
% The main differences with our approach 
% , which mean they are assuming the group to be a subgroup of their augmentation transformations. 
% Specifically, they assume the group to be affine transformations and learn which sub-group of it is used in the dataset. 
% In eq. 9, the $G_i$ (equivalent to our $L_i$) are the generators of affine transformations and the state below eq. 9 the $G_1,...G_6$ are predefined translations, rotations, scaling and shearing.  
% In contrast, we are learning $L_i$ and do not restrict them to a known set of generators. 
}%%%

\subsection{Additional Experiments \label{ap:experiments-old}}

\textbf{Matching number of parameters in CNN}
To verify that the difference in the number of parameters between CNN and L-conv was not responsible for the improved performance, we ran experiment where we allowed the kernel-size of L-conv and CNN to differ and tried to match the number of parameters between the two. 
Fig. \ref{fig:L-conv-CNN-match} shows that on rotated and scrambled MNIST L-conv still performs better than CNN even after the latter has been allowed to have the same or more number of parameters than L-conv.

\begin{figure}
    \centering
    \includegraphics[width = \columnwidth]{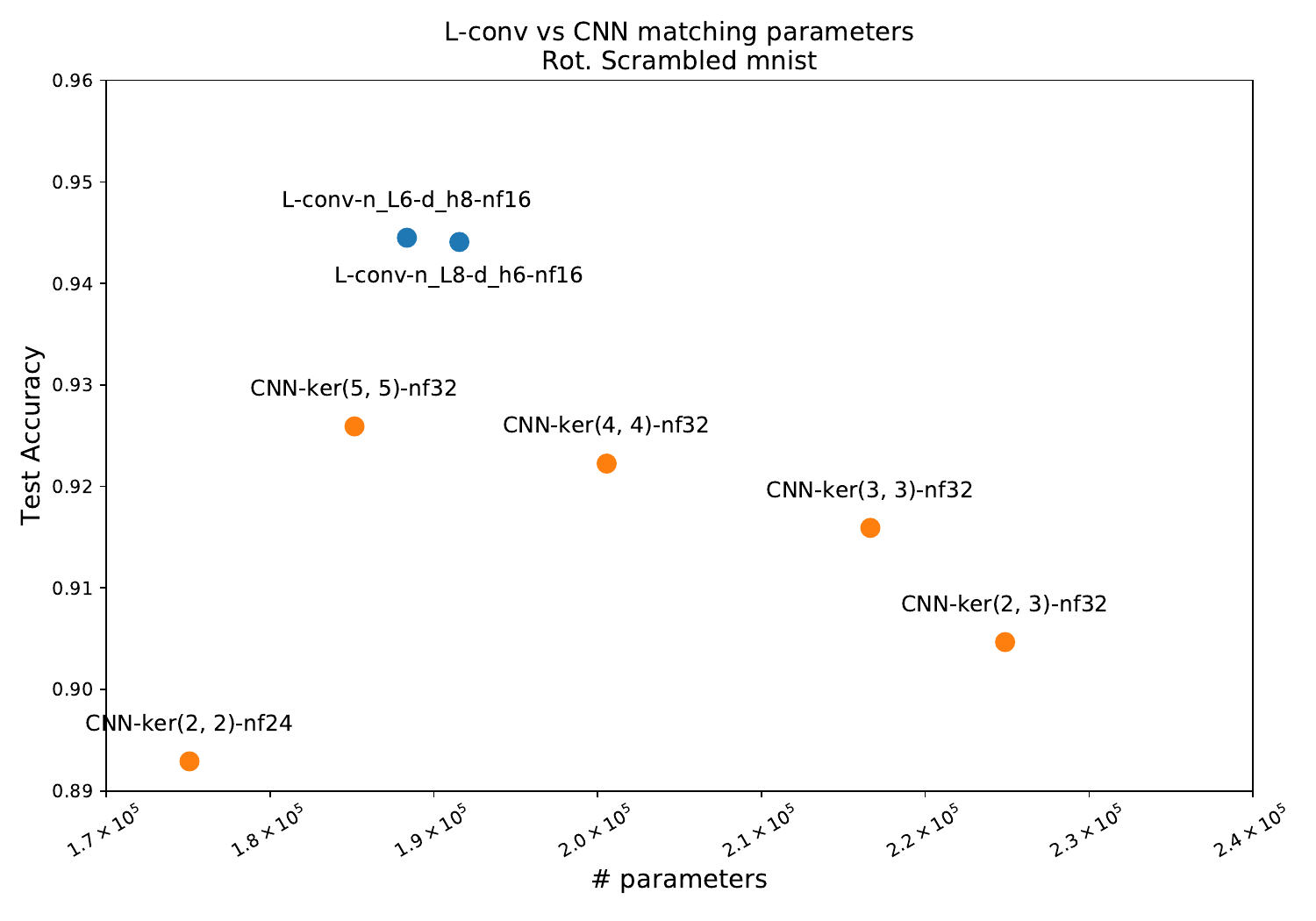}
    \caption{Matching number of parameters in CNN and L-conv, we observe that L-conv still performs better on Rotated and Scrambled MNIST.}
    \label{fig:L-conv-CNN-match}
\end{figure}

\begin{figure}%[htbp]
    \centering
    \includegraphics[width=.49\textwidth]{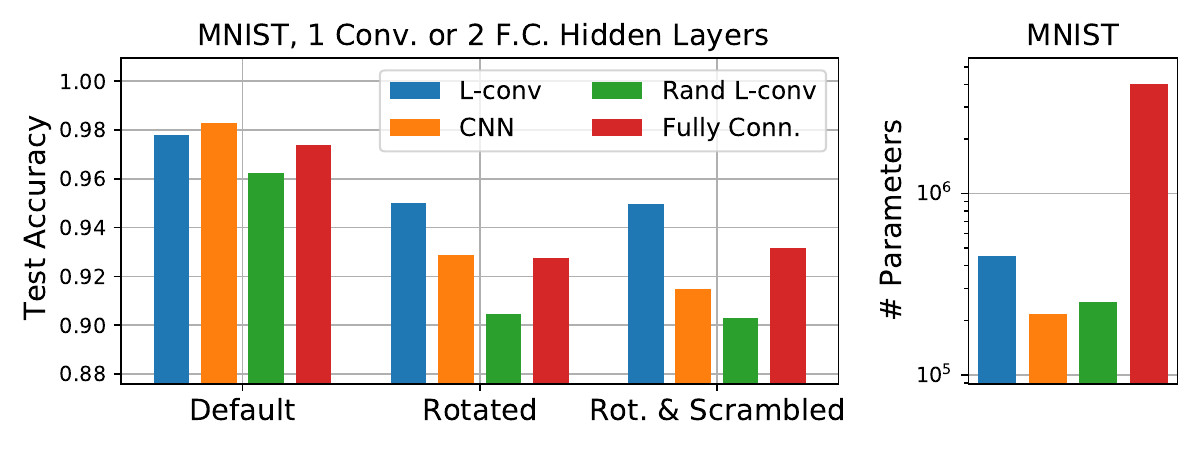}
    \includegraphics[width=.49\textwidth]{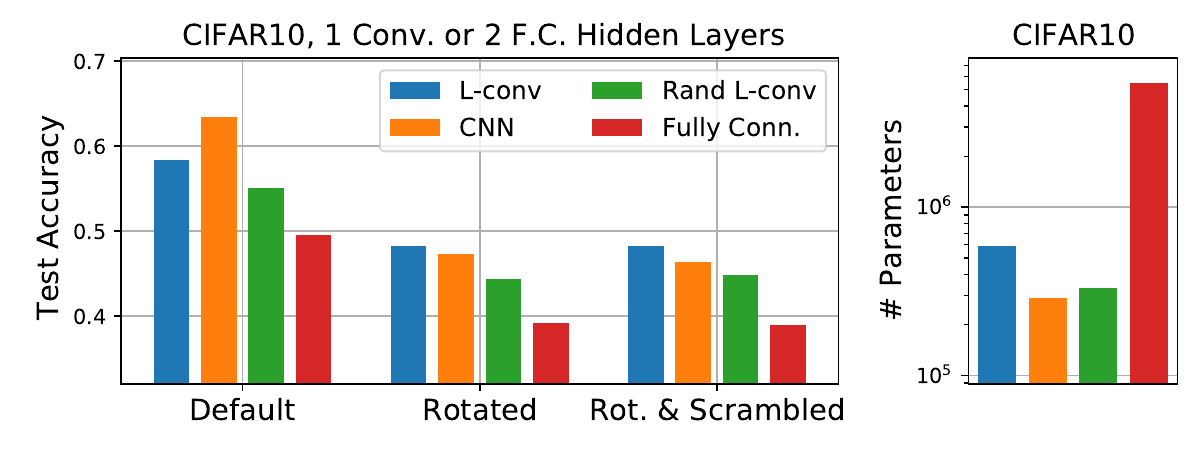}
    \includegraphics[width=.49\textwidth]{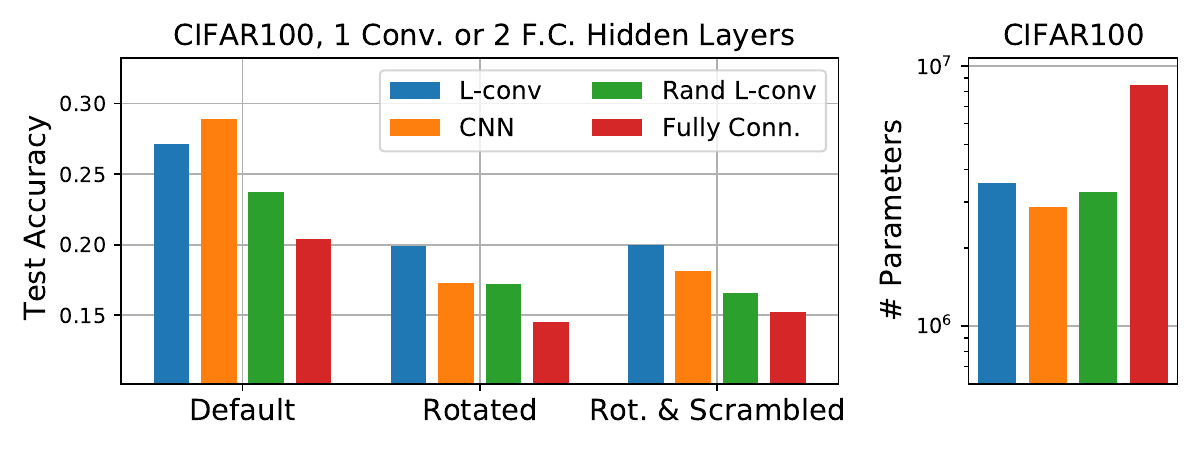}
    \includegraphics[width=.49\textwidth]{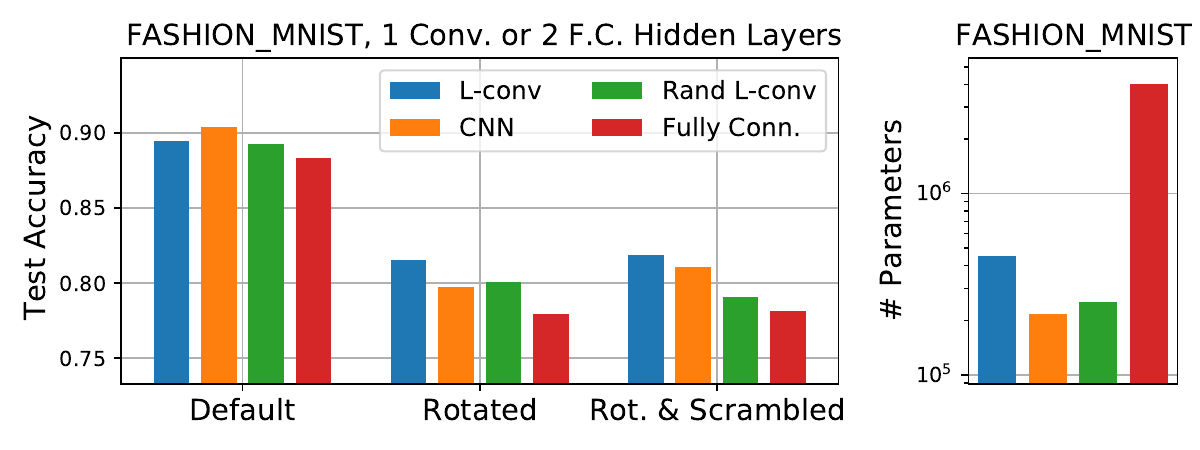}
    \caption{Test results on four datasets with three variant: ``Default'' (unmodified dataset), ``Rotated'' and ``Rotated and scrambled''. 
    On the Default dataset, CNN performs best, but L-conv is always the second best. 
    For Rotated and Rot. \& Scrambled,
    in all cases L-conv performed best. 
    In MNIST, FC and CNN layers come close, but using 5x more parameters.  
    }
    \label{fig:L-conv-4-datasets-full}
\end{figure}

\begin{figure}%[htbp]
    \centering
    \includegraphics[width=.8\linewidth]{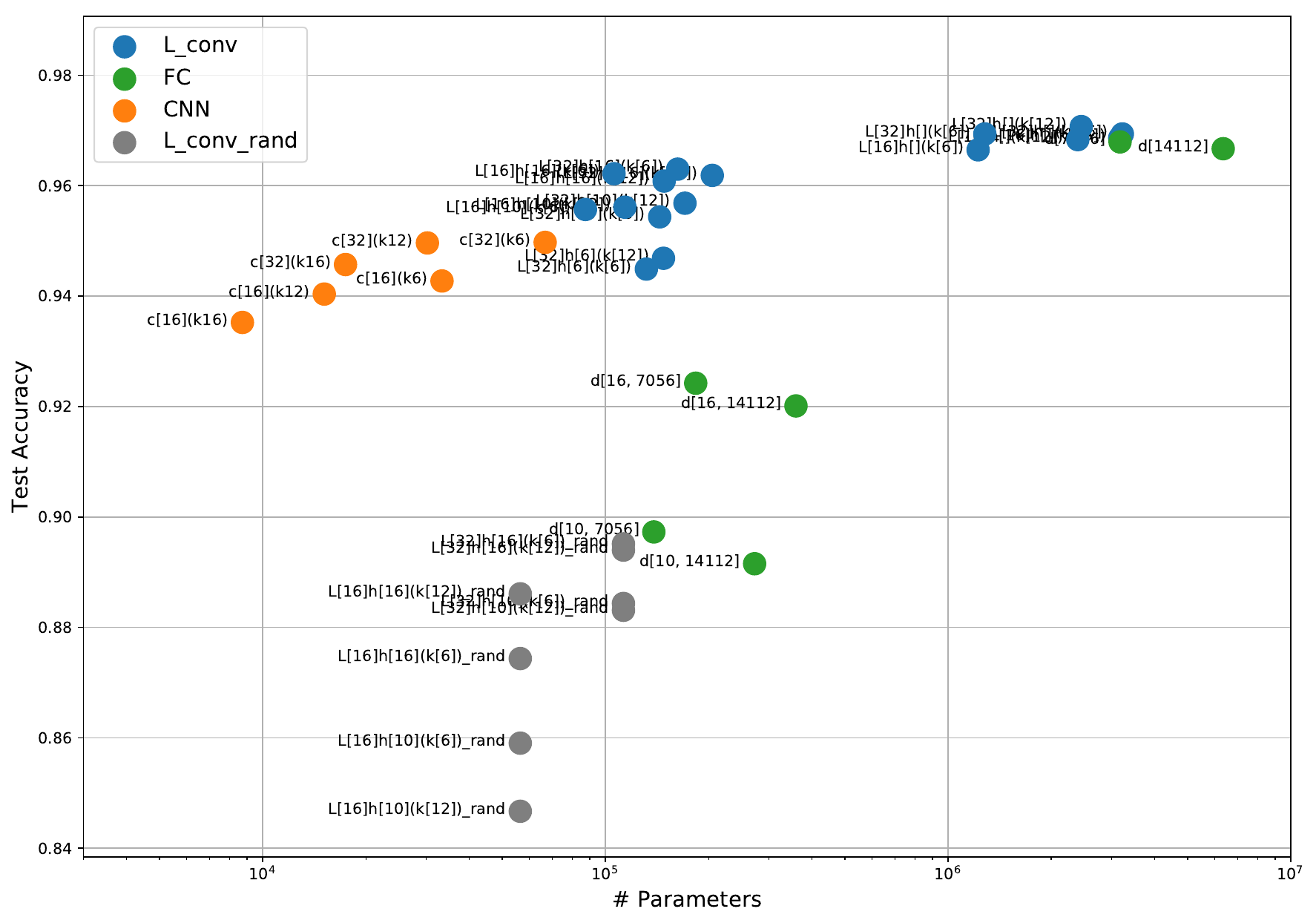}
    \caption{Training low-rank L-conv layer during training. 
    % Here we compare the performance of a single layer of L-conv on a classification task on scrambled rotated MNIST, where pixels have been permuted randomly and images have been rotated between $-90$ to $+90$ degrees. 
    % The models consisted of a final classification layer preceded by either one L-conv (blue), or one CNN (orange), or multiple fully-connected (FC, green) layers with similar number of neurons as the L-conv, but without weight sharing.
    % We see that most L-conv configurations had the highest performance without a too many trainable parameters. 
    % Note that, parameters in FC layers are much higher than comparable L-conv, but yield worse results.  
    % The dots are labeled to show the configurations, with $L[32]h[6](k[6])$ meaning $k=6$ as number of $L_i$, 32 output filters, and $h=6$ hidden dimensions for low-rank encoding of $L_i$. 
    % The y-axis shows the test accuracy and the x-axis the number of trainable parameters. The grey lines show the performance of L-conv with fixed random $L_i$, but trainable shared wights, showing that indeed the learned $L_i$ improve the performance quite significantly. 
    }
    \label{fig:L-conv-scramb-MNIST}
\end{figure}

In Figure \ref{fig:L-conv-scramb-MNIST} we compare the performance of a single layer of L-conv on a classification task on scrambled rotated MNIST, where pixels have been permuted randomly and images have been rotated between $-90$ to $+90$ degrees. 
The models consisted of a final classification layer preceded by either one L-conv (blue), or one CNN (orange), or multiple fully-connected (FC, green) layers with similar number of neurons as the L-conv, but without weight sharing.
We see that most L-conv configurations had the highest performance without a too many trainable parameters. 
Note that, parameters in FC layers are much higher than comparable L-conv, but yield worse results.  
The dots are labeled to show the configurations, with $L[32]h[6](k[6])$ meaning $k=6$ as number of $L_i$, 32 output filters, and $h=6$ hidden dimensions for low-rank encoding of $L_i$. 
The y-axis shows the test accuracy and the x-axis the number of trainable parameters. The grey lines show the performance of L-conv with fixed random $L_i$, but trainable shared wights, showing that indeed the learned $L_i$ improve the performance quite significantly.

\out{
% \section{Structure of Learned $L_i$}

\begin{figure}%[h]
    \centering
    \includegraphics[width=\columnwidth]{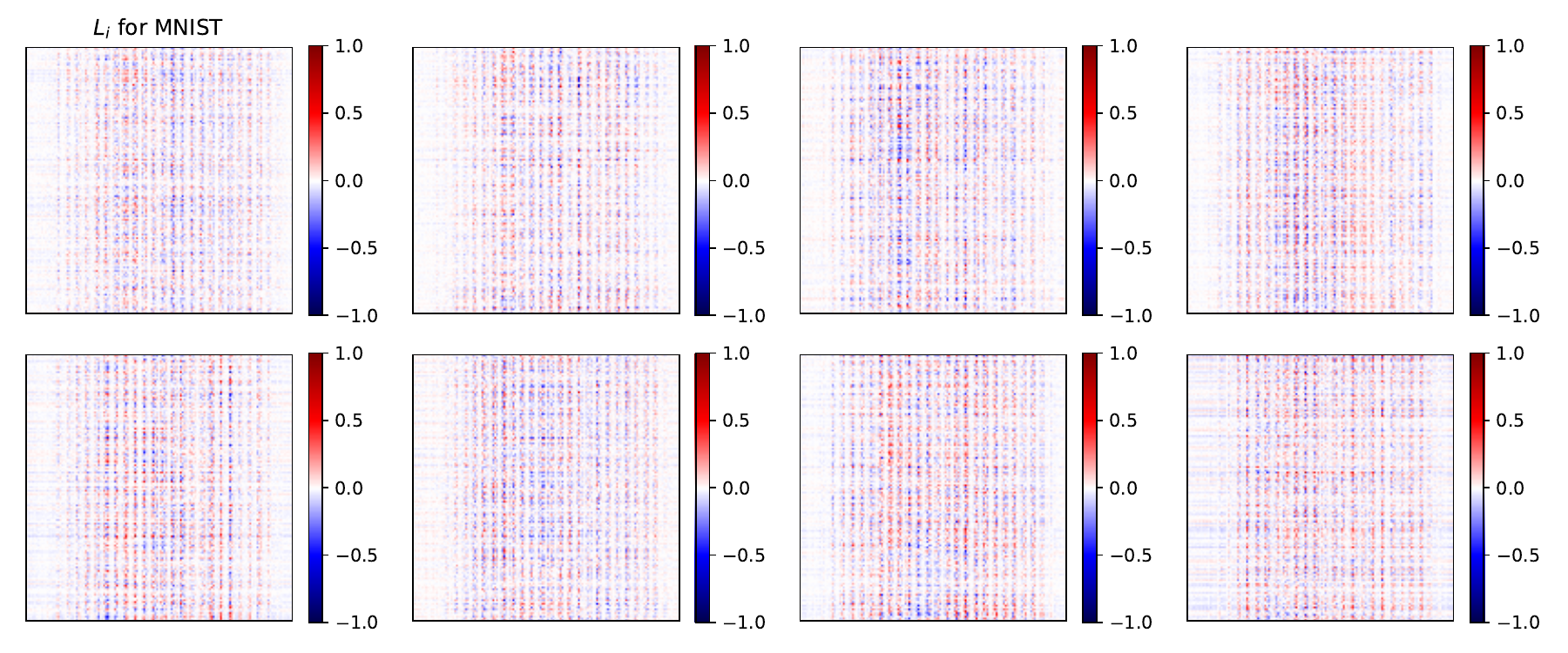}
    \includegraphics[width=.7\columnwidth]{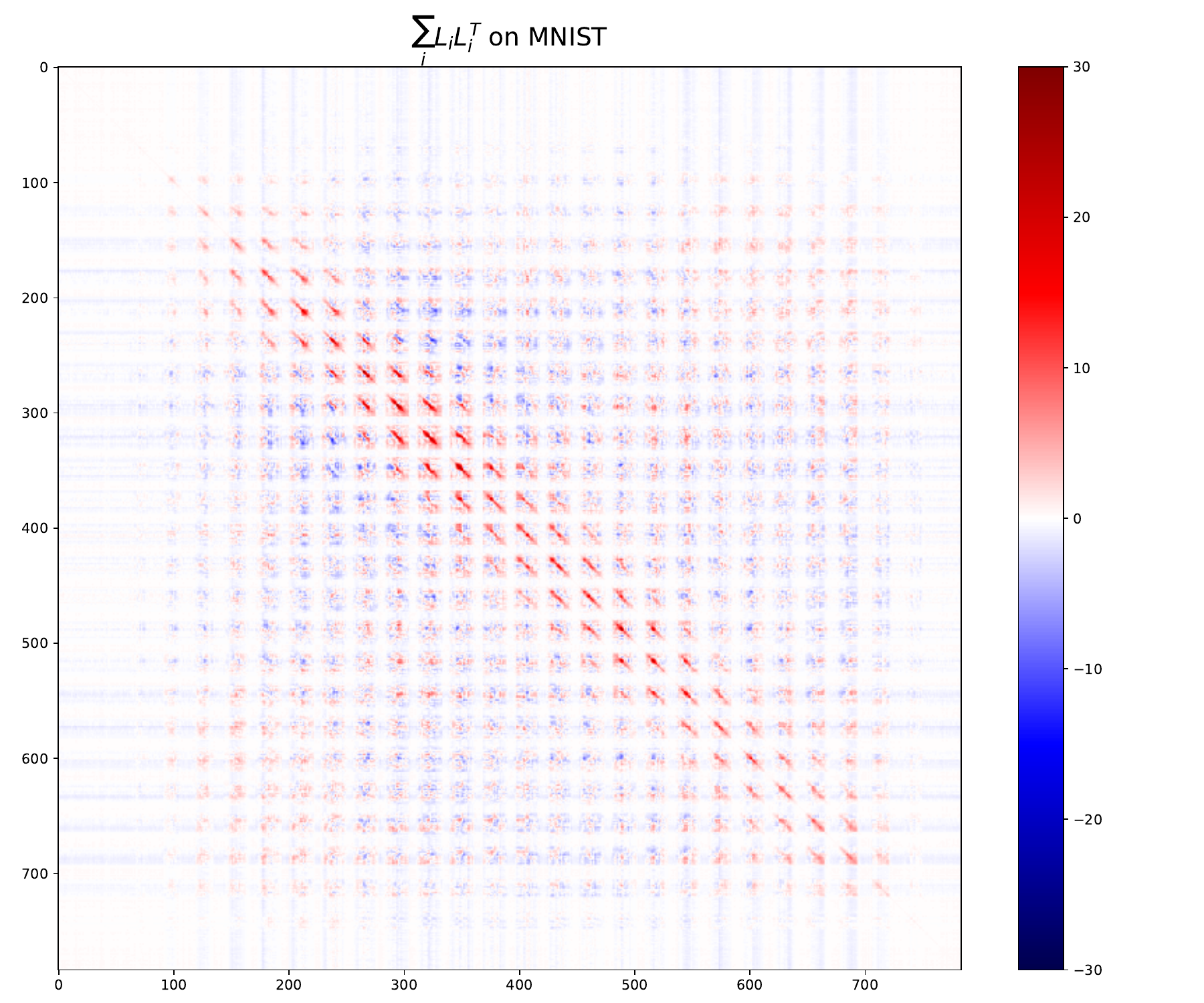}
    \caption{Visualization of the $L_i$ found on MNIST and their covariance $LL^T=\sum_i L_iL_i^T$. }
    \label{fig:L_i-LL}
\end{figure}

\begin{figure}%[h]
    \centering
    \includegraphics[width=\columnwidth]{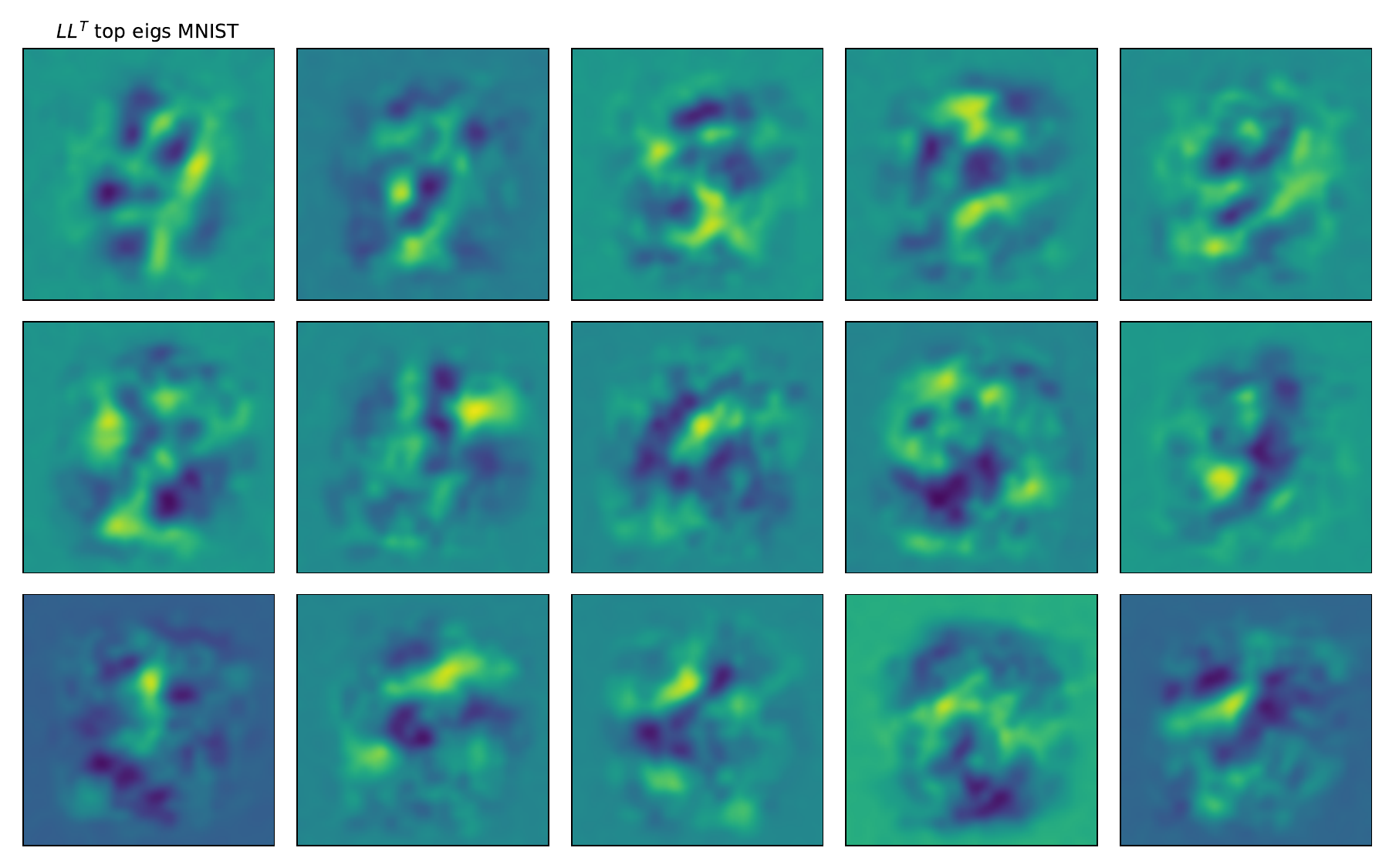}
    \caption{Visualization of the top eigenvectors of $\sum_i L_iL_i^T$.
    They show some resemblance to the eigenvectors of the covariance matrix $H=XX^T$. }
    \label{fig:LL-eigs}
\end{figure}

\newpage

}%%%%

% \newpage
% \input{secs2/universal}

% \input{secs2/scrap}

\end{document}